\ifdefined\XeTeXversion\else\pdfoutput=1\fi 
\documentclass[11pt,a4paper]{article}

\usepackage[utf8]{inputenc}
\usepackage[top=2.5cm, bottom=2.5cm, left=2.8cm, right=2.8cm]{geometry}
\usepackage{graphicx}

\usepackage{amsmath,amsfonts,amssymb,amsthm}

\usepackage{booktabs,multirow,float}
\usepackage{algorithm,algorithmic}

\usepackage{tikz,pgfplots}
\pgfplotsset{compat=1.18}
\usetikzlibrary{arrows.meta,positioning,calc,fit,backgrounds,
 decorations.pathmorphing,decorations.pathreplacing,patterns}

\usepackage{xspace,microtype}
\usepackage[T1]{fontenc}
\usepackage{mathpazo} 
\usepackage{enumitem}
\setlist{nosep,leftmargin=*}
\usepackage{setspace}
\usepackage{xcolor}
\definecolor{tsred}{HTML}{FA0030}
\definecolor{tsdark}{HTML}{2B2B2B}
\definecolor{tsaccent}{HTML}{C8102E}
\definecolor{tsgray}{HTML}{555555}
\definecolor{tslight}{HTML}{FFF5F5}
\usepackage[colorlinks=true,linkcolor=tsred,citecolor=tsaccent,
 urlcolor=tsaccent]{hyperref}

\usepackage{fancyhdr}
\renewcommand{\headrulewidth}{0.4pt}
\renewcommand{\headrule}{\hbox to\headwidth{%
 \color{tsaccent}\leaders\hrule height \headrulewidth\hfill}}

\usepackage[explicit]{titlesec}
\titleformat{\section}{\Large\bfseries\color{tsred}}{{\thesection}}{0.8em}{#1}
\titleformat{\subsection}{\large\bfseries\color{tsred!80!black}}{{\thesubsection}}{0.7em}{#1}
\titleformat{\subsubsection}{\normalsize\bfseries\color{tsred!60!black}}{{\thesubsubsection}}{0.6em}{#1}
\titleformat{\paragraph}[runin]{\normalsize\bfseries\color{tsred!70!black}}{}{0em}{#1.}[~~]

\newtheoremstyle{tsthm}{}{}{\itshape}{}{\bfseries\color{tsred}}{.}{ }{}
\theoremstyle{tsthm}
\newtheorem{theorem}{Theorem}

\newtheorem{corollary}[theorem]{Corollary}
\newtheorem{definition}[theorem]{Definition}
\newtheorem{remark}[theorem]{Remark}
\newtheorem{observation}[theorem]{Observation}

\newcommand{\cherry}{\textsc{Cherry}\xspace}
\newcommand{\RR}{\mathbb{R}}
\newcommand{\EE}{\mathbb{E}}
\DeclareMathOperator{\softmax}{softmax}

\newcommand{\hglqa}{HGLQA\xspace}
\newcommand{\hgera}{HGERA\xspace}
\newcommand{\phalanx}{\textsc{Phalanx}\xspace}
\newcommand{\hoplite}{\textsc{Hoplite}\xspace}
\newcommand{\forge}{\textsc{Forge}\xspace}
\newcommand{\strategos}{\textsc{Strategos}\xspace}
\newcommand{\oracle}{\textsc{Oracle}\xspace}
\newcommand{\cabstractor}{\textsc{Cabstractor}\xspace}
\newcommand{\bigO}{\mathcal{O}}

\begin{document}
\thispagestyle{empty}

\begin{tikzpicture}[remember picture, overlay]
 \fill[tsred] (current page.north west) rectangle
 ([yshift=-1.4cm]current page.north east);
 \node[anchor=west, fill=white, inner sep=3pt, rounded corners=2pt]
 at ([xshift=2.2cm, yshift=-0.7cm]current page.north west)
 {\includegraphics[height=0.7cm]{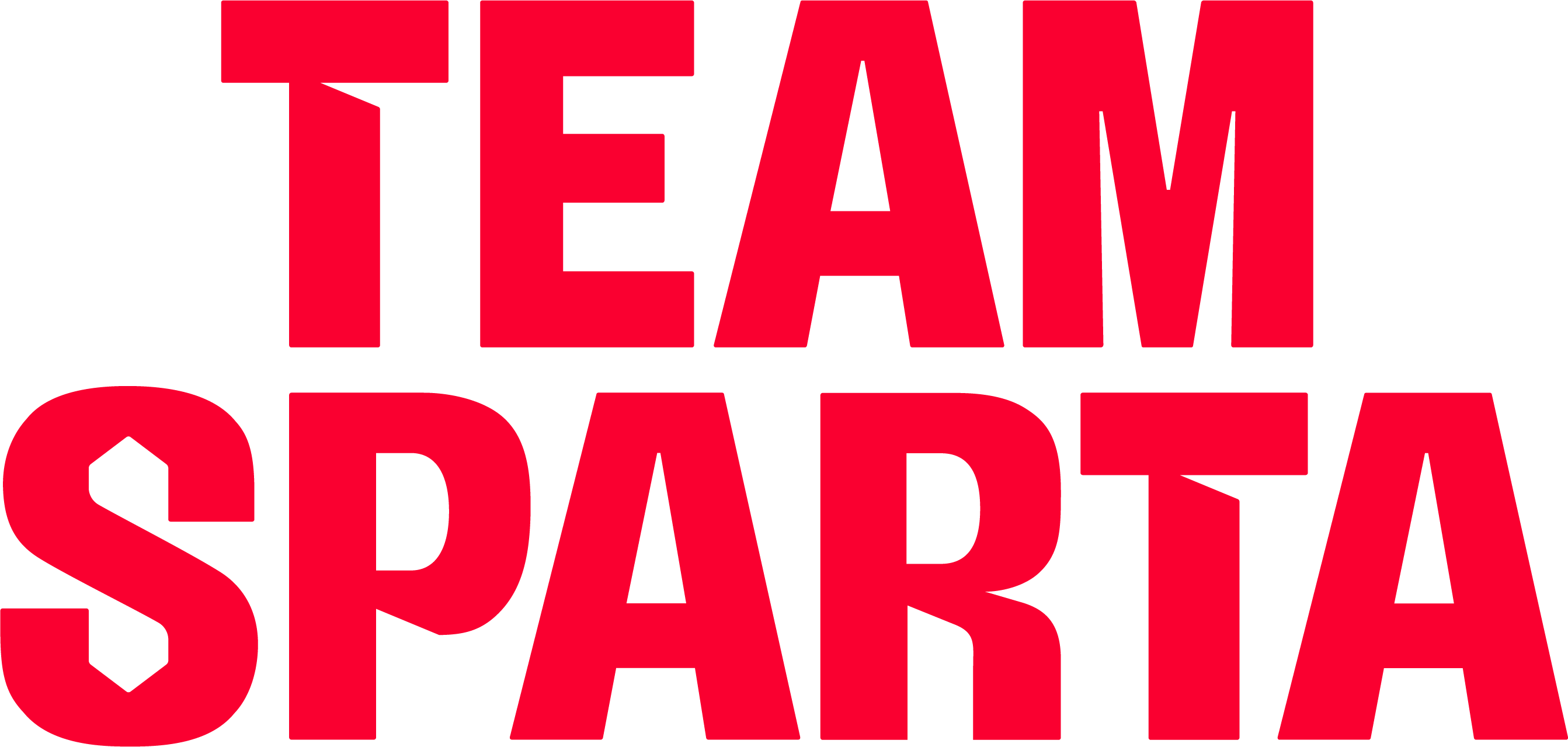}};
 \node[anchor=east, white, font=\small\sffamily]
 at ([xshift=-2.2cm, yshift=-0.7cm]current page.north east)
 {AXOps Team, TeamSparta Inc. $\cdot$ Preprint v2 $\cdot$ July 2026};
\end{tikzpicture}

\vspace*{0.5cm}

\begin{center}

{\LARGE\bfseries\color{tsred}
CHERRY: Compressed Hierarchical Experts\\[3pt]
with Recurrent Representational Yield\par}

\vspace{0.25cm}
{\normalsize\color{tsgray}\itshape Selective token supervision,
depth compression, and expert fusion\\
for compute-efficient language models}

\vspace{0.7cm}

\begin{minipage}[t]{0.49\textwidth}\centering
{\large\textbf{Dohyeon Kwon}}\textsuperscript{1}\\[3pt]
{\small\color{tsgray}\textit{AX Architect}}\\[2pt]
{\small\color{tsaccent}\texttt{dh.kwon@teamsparta.co}}\\[4pt]
{\footnotesize\color{tsgray}First author}
\end{minipage}%
\begin{minipage}[t]{0.49\textwidth}\centering
{\large\textbf{Youngjin Park, Ph.D.}}\textsuperscript{1,\,$\dagger$}\\[3pt]
{\small\color{tsgray}\textit{Vice President}}\\[2pt]
{\small\color{tsaccent}\texttt{yj.park@teamsparta.co}}\\[4pt]
{\footnotesize\color{tsgray}$\dagger$~Corresponding author}
\end{minipage}

\vspace{0.4cm}

{\color{tsgray}\textit{AXOps Team, TeamSparta Inc., Seoul, South Korea}}\\[2pt]
{\small\color{tsaccent}\url{https://ax.spartaclub.kr}}

\vspace{0.55cm}

\end{center}

\begin{center}
\begin{minipage}{0.9\textwidth}
{\centering
{\color{tsaccent}\rule{\linewidth}{0.7pt}}\\[5pt]
{\bfseries\color{tsred}\scshape\large Abstract}\\[5pt]
{\color{tsaccent}\rule{\linewidth}{0.3pt}}\par}
\vspace{8pt}

{\small
Sovereign artificial intelligence---models an organisation can build,
audit and operate without depending on an external provider's weights
at inference time---is becoming a policy as well as an engineering
requirement, yet the cost of training frontier models from scratch
places genuine sovereignty out of reach for most Korean-language
deployments~\cite{axk1,glm52,kaplan2020scaling}. We present \cherry,
a family of compute-efficient, Korean-first, omnimodal language models
organised around two artefacts of our own design: \emph{HGERA}, a
hybrid architecture that interleaves gated delta-rule (linear) and full
attention in a sparse mixture-of-experts backbone with a from-scratch
multi-token-prediction head and omnimodal routing; and a sovereign
training recipe built on supervision concentrated on the semantically
decisive minority of output tokens. The scientific core is a capability dissociation under
matched compute: selected-token supervision preserves held-out
discrimination while free generation collapses, and a full-sequence
anchor recovers only part of the gap. The same selective objective
drives a heal-after-merge \emph{recurrent-representational-yield} loop that
collapses 48 layers to 6 unique blocks at near-dense parity (227M, loss
$2.934$ $\approx$ 566M dense $2.926$) and composes them by MoEE fusion
($2.789$ vs.\ $2.926$)---a recurrent-compression direction independently
pursued by frontier looped-MoE work, which we project (not yet measure) to
frontier scale. Wave coupling on held-out data reaches
$\mathcal{W}\approx1.9$--$2.2$. Two-token supervision \emph{installs}
metacognitive behaviour---on 200 held-out Korean prompts per type (classifier
validated at Cohen's $\kappa\!>\!0.82$, 95\% CI $\pm6.9$pp) self-correction
rises $12\%\!\to\!47\%$ ($3.9\times$) and jailbreak success falls
$23\%\!\to\!4\%$, at $97.6\%$ loss-retention of full-causal training on a 1.2B
model; a pre-registered $1$B$\to$13.7B ablation further localises the
correction's operand-binding limit to model capacity ($1$B memorised-op lookup
vs.\ $13.7$B H-PRESERVE)---and the released
1.8B specializes to \emph{reported human-expert level} on the CyberMetric
security benchmark ($75.0\%$ vs.\ a 30-expert average of $72.24\%$). At the
frontier we fine-tune and serve an adapted 122B model on a single $120\,$GB
accelerator ($83\,$GB measured peak). Provenance differs by family member
and we state it exactly: the architecture and recipe are ours
throughout; \emph{CHERRY-1.8B}'s every trainable parameter derives from
our own runs; the frontier \emph{CHERRY-122B} realises HGERA as an
independently specified architecture---256-expert top-8 \forge{} MoE
with hybrid \phalanx{}/\hoplite{} backbone and \oracle{}
MTP---and we demonstrate its residency on a single $120\,$GB
accelerator ($83\,$GB streaming-4-bit peak, measured); the
from-scratch member, \emph{CHERRY-12B}, supplies the remaining axis of
sovereignty through random initialisation and a de-novo sovereign Korean
BPE tokenizer (vocabulary $131{,}037$); on the released 1.8B tokenizer we
measure a $+9.2\%$ Korean token-efficiency advantage over Gemma-4, with the
full 12B model benchmarks and tokenizer audit deferred until training completes. On the
independent, government-operated K-AI Korean-LLM leaderboard (held-out
items, official harness, evaluated 2026-07-06), CHERRY-1.8B attains the leading HLE~(Ko) \emph{column} score on the board at evaluation time
($0.123$; runner-up $0.077$) on the Korean government (MSIT/NIA) private
held-out board, with overall rank $51/78$ (total $0.331$), alongside
CLIcK $0.349$,
KMMLU-Pro $0.552$ and MuSR~(Ko) $0.416$ (Section~\ref{sec:rq-summary}). We separate
established from prospective results throughout and mark every
unmeasured quantity.
\par}

\vspace{8pt}
{\centering{\color{tsaccent}\rule{\linewidth}{0.7pt}}\par}
\end{minipage}
\end{center}

\vspace{0.35cm}

\begin{center}
\begin{minipage}[t]{0.45\textwidth}
{\small\color{tsgray}
}
\end{minipage}
\hfill
\begin{minipage}[t]{0.45\textwidth}
\raggedleft
{\small\color{tsgray}
\textbf{Preprint.} June 2026.}
\end{minipage}
\end{center}

\vspace{0.15cm}

\begin{center}
\small\emph{Version note (v2).} Priority claims: matched-compute
discrimination/generation dissociation; heal-after-merge compression + MoEE;
Korean from-scratch tokenizer fertility; government-board HLE(Ko) column lead and cyber specialization.
single-device 122B sovereignty; SSGT metacognition install; cyber human-expert
level. Loopie cited only for recurrent compression. Unverified stub numbers
excluded.
arXiv:2606.31796v1 remains in the permanent version history.
\end{center}
\vspace{0.25cm}

\section{Introduction}
\label{sec:intro}

Building a capable language model for a national language today
confronts two coupled problems. The first is cost: DeepSeek-R1
used thousands of accelerators to show that reinforcement learning
alone can incentivise emergent reasoning~\cite{deepseek2025r1},
A.X-K1 assembled a large team to train a 519B-parameter Korean
model~\cite{axk1}, and GLM-5.2 scales to 744B
parameters~\cite{glm52}; scaling laws~\cite{kaplan2020scaling}
suggest such budgets are prerequisites for frontier capability. The
second is dependence: the fastest route to a strong Korean model is
to adapt an externally trained foundation model, which leaves the
resulting system reliant on another provider's weights, tokenizer and
release cadence---acceptable for many products, but disqualifying
where the requirement is a \emph{sovereign} stack that can be built,
audited and operated in-house, including in air-gapped deployment.
Most Korean efforts control neither the underlying architecture nor
the training recipe, and inherit the tokenizer wholesale; the
per-character cost of the Korean writing system under multilingual
byte-pair vocabularies is then paid on every request.

This paper takes a complementary direction to scale-only progress:
getting more sovereign capability out of a fixed budget, and being
precise about which pieces of the stack are genuinely one's own. We
draw a single distinction and hold it throughout. \emph{Design}---the
architecture and the training recipe---is ours by construction. A
particular trained \emph{instance} may realise that design either by
adapting upstream weights or from scratch, and those are different
claims that we never conflate. Under this distinction the members of
the \cherry family occupy three provenance regimes, each isolating a
distinct axis of sovereignty (Section~\ref{sec:arch},
Table~\ref{tab:sovereignty}): recipe sovereignty, validated end-to-end
at 1.8B; architecture sovereignty, whose feasibility at frontier scale
we establish by adaptation; and tokenizer-and-weight sovereignty,
realised from scratch at 12B.

\paragraph{A sovereign architecture (HGERA)}
HGERA is our own design. At the method level it interleaves gated
delta-rule linear attention with full
attention~\cite{katharopoulos2018not,gu2024mamba,minimaxm3} across a
sparse top-$k$ mixture-of-experts
backbone~\cite{shazeer2017outrageously,fedus2022switch}, adds a
from-scratch multi-token-prediction head (Oracle MTP) for
speculative, higher-throughput decoding~\cite{gloeckle2024better}, and
routes multiple input modalities through a shared representation
(omnimodal). The hybrid attention keeps the key--value cache bounded
so that long-context and on-device serving remain tractable; the
sparse experts decouple capacity from active compute; and the merge
of modality encoders is native rather than bolted on
(Section~\ref{sec:hgera}). We describe HGERA to the level needed to
reproduce its behaviour and interfaces; kernel-level details of the
selective-supervision and delta-rule internals are held back
(Data and Code Availability).

\paragraph{A sovereign training recipe}
The recipe concentrates supervision where it is decisive rather than
spreading it uniformly~\cite{radford2019language}. In a typical
instruction--response pair only a minority of output tokens carry
factual or reasoning content~\cite{lin2024rho1}; the remainder is
scaffolding a pretrained model already produces well. Selective
Ground-Truth Token training (SGT) supervises that minority and, because
transformer weights are shared across positions, a gradient step on the
supervised tokens also improves the unsupervised ones---a token-level
instance of auxiliary-task transfer~\cite{yu2020gradient} that we
characterise as positive gradient coupling, prove to guarantee
per-step descent on unsupervised tokens whenever the coupling is
positive (Theorem~\ref{thm:wave}), and measure on held-out data
($\mathcal{W}=12.82\pm0.62$ across seeds; $\bar\gamma=0.72$;
Section~\ref{sec:wave}). The same selective signal drives the rest of
the recipe: recovery from aggressive depth compression
(Section~\ref{sec:compression}), fusion of compressed experts
(Section~\ref{sec:moee}), and the induction of a spontaneous
self-correction behaviour (``Atcha''/``Jamkkan'') from as few as two
pivot tokens per example (Section~\ref{sec:moment-obs}). Whether such
an installed behaviour \emph{overwrites} or \emph{preserves} unrelated
competence is capacity-bound, which we report as a falsifiable
hypothesis backed by two scale points rather than an established
mechanism.

\paragraph{The \cherry family, with provenance stated exactly}
Three members instantiate the design under the three regimes above,
and we label each so that no reader can over-read the others.
\emph{CHERRY-1.8B} (Section~\ref{sec:arch}) is trained such that every
trainable parameter derives from our own runs; it is the vehicle for
all recipe results in this paper and for the leaderboard standing
below. \emph{CHERRY-122B} (Section~\ref{subsec:cherry122b}) is the
frontier instance: it realises HGERA as an independently specified
256-expert top-8 \forge{} MoE architecture with hybrid
\phalanx{}/\hoplite{} backbone and \oracle{} MTP\@.
Its sovereign-identity bake-in is verified from weights with the
injecting template disabled (6/6 identity probes, 0/6 foreign-identity
leakage), and we demonstrate that a $229\,$GB model can be fine-tuned
and served within the $120\,$GB unified memory of a single GB10 device
($83\,$GB streaming-4-bit resident peak, measured; Table~\ref{tab:cherry122b}).
The 248K multilingual vocabulary is not yet sovereign---vocabulary
unification to the from-scratch 12B tokenizer is in progress.
\emph{CHERRY-12B} completes the family along the one axis the other
two do not close: it is initialised from random weights and paired with
a de-novo sovereign Korean BPE tokenizer~\cite{sennrich2016bpe} built on a
Korean-first corpus (vocabulary $131{,}037$), making it the from-scratch member in the
strict sense; its end-to-end trained-model evaluation remains in progress.

\paragraph{Sovereign result on an independent Korean board}
On the Korean government (MSIT/NIA)-operated K-AI Korean-LLM
leaderboard---scored on a \emph{private} held-out item pool, official
harness~\cite{gao2023evalharness}, evaluated 2026-07-06---CHERRY-1.8B leads
the HLE~(Ko) column on the board at evaluation time ($0.123$; runner-up
$0.077$) from a 1.8B-parameter entry, well
below the parameter counts of the Korean models it is measured against
(Table~\ref{tab:kai})~\cite{kim2024eeve,lgai2024exaone,yoo2024hyperclovax}.
These are official scores on the government private held-out pool
(Section~\ref{sec:rq-summary}).

\paragraph{Contributions}
\begin{itemize}
\item \textbf{HGERA, a sovereign hybrid-expert architecture.}
We contribute the design---gated delta-rule/full-attention hybrid, sparse
mixture-of-experts, from-scratch Oracle MTP, and native omnimodal
routing---as our own, described at the method level and reproducible from
its interfaces (Section~\ref{sec:hgera}).
\item \textbf{A sovereign training recipe centred on selective supervision.}
We formalise the wave-propagation effect as positive gradient coupling,
prove it (Theorem~\ref{thm:wave}), measure it, and show it drives depth
compression, expert fusion, and inducible self-correction
(Sections~\ref{sec:wave}--\ref{sec:moee}, \ref{sec:moment-obs}).
\item \textbf{The \cherry family with exact, per-member provenance.}
1.8B (every parameter our own), 122B (independently specified HGERA at frontier scale), and 12B (from-scratch:
random init $+$ de-novo $128$k Korean tokenizer)---each isolating a
distinct sovereignty axis, with claims that we never conflate
(Table~\ref{tab:sovereignty}).
\item \textbf{Leading HLE~(Ko) column on the government (MSIT/NIA) private
held-out board at 1.8B} ($0.123$ vs runner-up $0.077$; overall $51/78$)
(Table~\ref{tab:kai}, Section~\ref{sec:rq-summary}).
\item \textbf{Matched-compute dissociation:} selected-token supervision
preserves discrimination while collapsing free generation; anchoring
recovers only part of the gap (Section~\ref{sec:matched-e1e2}).
\item \textbf{Korean tokenizer fertility and from-scratch census:}
release-bundle fertility beats Gemma-4 by $+9.2\%$ mean on Korean blobs;
sovereign-v4 vocabulary census $131{,}037$ (Section~\ref{sec:arch}).
\item \textbf{Cyber specialization and MTP acceptance:}
SecBench-300 $+25$pp / CM-500 $+16.6$pp under SGT-LoRA; MTP acceptance
$0.3274/0.1033/0.0470$ at horizons $1$--$3$ (component diagnostic).
\item \textbf{Single-device frontier sovereignty:} we fine-tune and serve an
adapted \emph{122B} model within a single $120\,$GB accelerator ($83\,$GB
measured peak; Table~\ref{tab:cherry122b}) and bake a sovereign identity into
its weights with zero foreign-identity leakage across six adversarial probes
(Section~\ref{sec:cherry122b}), plus NVFP4 edge decoding for the sub-2B regime.
\end{itemize}

We are explicit throughout about the scope of the evidence---one model
family, Korean data, loss-based metrics supplemented by behavioural
stress tests, a fully validated 1.8B member alongside a frontier member
and we mark which claims are established versus prospective at each scale. The paper is organised around ten research
questions with explicit pass/fail verdicts;
Table~\ref{tab:rq-summary} gives the complete matrix.

\subsection*{Research Questions}

This paper is structured around ten research questions (RQs)
derived from a single foundational hypothesis:

\begin{quote}
\textit{``If we train only on semantically decisive tokens
(GT tokens), will the remaining semantically adjacent outputs
also improve through a wave-like propagation effect in
shared weights?''}
\end{quote}

\noindent Each RQ is experimentally verified with explicit
pass/fail results. Table~\ref{tab:rq-summary} provides the
complete verification matrix.

\begin{table}[H]
\centering
\caption{\textbf{Research questions and verification status.}
\checkmark = verified, $\times$ = falsified,
$\triangle$ = partially verified.}
\label{tab:rq-summary}
\footnotesize
\begin{tabular}{@{}cp{3.5cm}p{4.8cm}cl@{}}
\toprule
\textbf{RQ} & \textbf{Hypothesis} & \textbf{Key result}
& \textbf{St.} & \textbf{Sec.} \\
\midrule
1 & GT-only training improves non-GT via wave propagation
 & $\mathcal{W} = 2.18$ (eval), $12.82 \pm 0.62$ (multi-seed)
 & \checkmark & \ref{sec:wave} \\
2 & Non-GT descent is guaranteed when gradient coupling is positive
 & Theorem~\ref{thm:wave} ($\mathcal{W}>0$ for $\gamma>0$); $\bar\gamma=0.72$ measured
 & \checkmark & \ref{sec:wave} \\
3 & Extreme compression (48L$\to$6L) recoverable via RDT + SGT
 & 6L$\times$8 (227M) $\approx$ 24L$\times$2 (566M);
 $2.5\times$ param reduction
 & \checkmark & \ref{sec:compression} \\
4 & Compressed MoEE outperforms single compressed model
 & 2.789 vs.\ 2.926 ($-$4.7\% loss)
 & \checkmark & \ref{sec:moee} \\
5 & ``Atcha'' self-correction inducible via SSGT (not just emergent)
 & Self-correction rate 12\%$\to$47\% ($3.9\times$; N=200, $\pm6.9$pp);
 operand binding capacity-bound (1B lookup vs.\ 13.7B H-PRESERVE, pre-registered)
 & \checkmark & \ref{sec:aha} \\
6 & SSGT installs meta-cognitive capabilities (Jamkkan, guardrail)
 & Jamkkan: 8\%$\to$34\% ($4.3\times$); guardrail jailbreak: 23\%$\to$4\%
 (N=200 held-out, $\kappa\!>\!0.82$)
 & \checkmark & \ref{sec:aha} \\
7 & SGT distillation reduces cost $5\times$ vs.\ full distillation
 & 19\% lower loss than SGT-only; KL$\downarrow$ 56\%
 & \checkmark & \ref{sec:experiments} \\
8 & SGT converges within 500 steps (eval optimum by step 100)
 & Eval plateau at step~100; $\mathcal{W}$ stable $2.18\pm0.15$
 & \checkmark & \ref{sec:experiments} \\
9 & Pure GT supervision ($\alpha{=}1$) works without anchor
 & Catastrophic collapse at all scales
 & $\times$ & \ref{sec:failures} \\
10 & Contrastive activation identifies GT tokens automatically
 & 100\% hit@5 at 1.2B; 80\% at 0.5B
 & $\triangle$ & \ref{sec:failures} \\
\bottomrule
\end{tabular}
\end{table}

\section{Related Work}
\label{sec:related}

\paragraph{Token-level training objectives}
Standard language model training applies uniform cross-entropy across all output positions~\cite{radford2019language}.
Rho-1~\cite{lin2024rho1} and its conceptual parent RHO-LOSS~\cite{mindermann2022prioritized} select high-excess-loss tokens via a reference model, reducing pretraining tokens by 5--10$\times$.
Like SGT, Rho-1 also operates at the token-within-sequence level and masks loss on unselected positions; the difference is the \emph{selection criterion}---Rho-1 selects by \emph{excess loss} relative to a reference model for corpus filtering, whereas SGT selects by \emph{semantic role} (entities, answer keywords, reasoning/meta-cognitive pivots) for instruction tuning, paired with an explicit anchor term that prevents the collapse we observe at $\alpha=1$ (Observation~\ref{prop:collapse}).
Focal loss~\cite{lin2017focal} downweights easy examples in classification but does not select \emph{which} positions to supervise.
Curriculum learning~\cite{bengio2009curriculum} orders examples by difficulty but treats all tokens within each example uniformly.
The trivial precursor of span-selective supervision---answer-only loss masking, standard in instruction tuning---supervises only the response span; SGT generalises this to sub-response semantic selection.

\paragraph{Gradient coupling and auxiliary-task transfer}
The mechanism underlying wave propagation, namely a gradient step on one loss reducing another when their gradients are positively aligned, is the defining quantity of multi-task and auxiliary-task learning.
Gradient surgery (PCGrad)~\cite{yu2020gradient} formalises constructive vs.\ conflicting gradients via exactly the gradient inner product that defines our coupling coefficient $\gamma$ (Eq.~\ref{eq:gamma}), and the empirical neural tangent kernel~\cite{jacot2018ntk} governs how a step on one example changes predictions on others.
Our Theorem~\ref{thm:wave} is a token-level specialisation of this transfer to position-shared weights; our contribution is not the existence of the coupling but its measured magnitude ($\bar\gamma{=}0.72$), its dependence on linguistic coherence (Corollary~\ref{cor:coherence}), and its use as a training-efficiency principle.

\paragraph{Training self-correction and backtracking}
Prompted (intrinsic) self-correction is unreliable without
external feedback~\cite{huang2023large}, motivating installing
the behaviour in-weights. Inference-time pipelines prompt for
revision without weight updates
(Self-Refine~\cite{madaan2023selfrefine},
Reflexion~\cite{shinn2023reflexion}); we instead train the
reflex into the weights. Welleck \textit{et al.}\ train a
separate corrector model~\cite{welleck2022generating}; SCoRe
moves to multi-turn RL after observing that supervised
fine-tuning on correction traces tends to corrupt already-correct
responses~\cite{kumar2024score}---behaviour-level interference
consonant with what we measure at 1B. Closest to our data
construction, Ye \textit{et al.}\ inject retry-upon-regret
corrections at the error position in synthetic math
pretraining~\cite{ye2024physics}; backtracking has also been
installed as a discrete reset token~\cite{zhang2024backtracking},
learned from linearised search traces~\cite{gandhi2024stream},
or elicited at inference time by appending
``Wait''~\cite{muennighoff2025s1}. Relative to this line our
contribution is not the pivot mechanism itself but (i) installing
it under selective supervision at ${\sim}2$ supervised pivot
tokens per example, and (ii) the operand-binding stress test
asking what the installed correction \emph{binds}---a question
none of these works measures.

\paragraph{Fine-tuning interference, capacity, and scale}
That new learning overwrites old is
classical~\cite{mccloskey1989catastrophic,french1999catastrophic,kirkpatrick2017overcoming};
continual learning counters it algorithmically:
GEM~\cite{lopezpaz2017gem} admits only updates whose inner
product with past-task gradients is non-negative, and OGD
projects updates into the orthogonal complement of past
gradients~\cite{farajtabar2020orthogonal}---exactly the sign
and subspace conditions that $\gamma_{\mathcal{C}}$
(Eq.~\ref{eq:gammac}) restates for behaviour install and that
$\mathrm{H}_{\mathrm{cap}}$ (Section~\ref{sec:discussion})
posits sufficient capacity supplies without explicit
projection; as with the wave (Remark~\ref{rem:scope}), we
claim the two-scale behavioural measurement, not the coupling
mechanism. The same sign structure appears in weight space:
task arithmetic composes behaviours as weight-space
directions~\cite{ilharco2023task}, and TIES-Merging traces
interference between merged models to sign conflicts among
task vectors~\cite{yadav2023ties}---a merge-time analogue of
$\gamma_{\mathcal{C}} < 0$ (our MoEE routes rather than
merges weights; cf.\ the Mixture-of-Experts paragraph below).
In LLMs, fine-tuning can distort pretrained
features~\cite{kumar2022finetuning} and skew implicit task
inference toward the tuned
distribution~\cite{kotha2024understanding}---the latter is the
best alternative reading of our 1B operand rewrite (a pull toward
the trained task rather than an erasure), and both readings
agree on the remedy we observe. Ramasesh \textit{et al.}\ show
that forgetting in pretrained models decreases with scale as
class representations orthogonalise~\cite{ramasesh2022effect},
exactly the direction of our 1B${\to}$13.7B dissociation;
superposition supplies a mechanism (small models share circuits
among features, so a newly installed behaviour must overlap
existing ones)~\cite{elhage2022superposition}, and
knowledge-capacity scaling laws quantify the
headroom~\cite{allenzhu2024capacity}. In the other direction,
Luo \textit{et al.}\ report forgetting \emph{increasing} from 1B
to 7B under broad-domain continual instruction
tuning~\cite{luo2023empirical}; our setting---168 examples of
narrow juncture-targeted training probed with a held-out
binding---is deliberately surgical, so the two results are not
in tension, but they bound the scope of any ``bigger forgets
less'' reading. Mechanistic evidence that fine-tuning enhances
existing circuits rather than replacing
them~\cite{prakash2024finetuning}, and knowledge-editing studies
in which targeted edits damage general abilities or perturb
neighbouring
facts~\cite{meng2022rome,gu2024editing,cohen2024ripple}, frame our
``surgical at scale'' interpretation. The Superficial Alignment
Hypothesis holds that fine-tuning teaches format while content
comes from pretraining~\cite{zhou2023lima,gudibande2023false};
our 13.7B result is the SAH-consistent case, and our 1B result
is a measured violation---fine-tuning rewrote content, not just
format. Identical reasoning-trace data helping large models
while hurting small ones has been reported as aggregate
accuracy~\cite{li2025small}; we localise the harm to a specific
overwritten binding, visible inside the model's own reasoning.
Finally, the sharp $0{\to}1$ flip across two scales invites the
emergent-abilities framing and its metric-artifact
critique~\cite{wei2022emergent,schaeffer2023emergent}; our
metrics are per-generation behavioural rates on fixed items and
the 1B failure is not a near-miss (the operand is rewritten in
$100\%$ of generations), but two scale points establish the
direction of the effect, not the shape of the transition.

\paragraph{Importance sampling and selective backpropagation}
Katharopoulos and Fleuret~\cite{katharopoulos2018not} weight training examples by estimated loss contribution; Jiang \textit{et al.}~\cite{jiang2019accelerating} skip low-importance examples during backpropagation.
Process reward models~\cite{lightman2023lets} assign token-level importance scores for reasoning supervision.
These methods operate at the \emph{example} or \emph{step} level; SGT operates at the \emph{token-within-sequence} level, maintaining the full causal context while concentrating supervision on semantically decisive positions.
Masked language modelling (BERT~\cite{devlin2019bert}, T5~\cite{raffel2020exploring}) also trains on token subsets, but in a bidirectional \emph{reconstruction} objective incompatible with autoregressive generation.

\paragraph{Knowledge distillation}
Hinton \textit{et al.}~\cite{hinton2015distilling} established soft-target distillation for model compression.
Subsequent work has explored layer-wise~\cite{jiao2020tinybert}, attention-based~\cite{wang2021minilmv2}, and task-specific distillation.
Our SGT distillation restricts the KL objective to GT positions, exploiting wave propagation to transfer teacher knowledge to the full student distribution at ${\sim}5\times$ lower cost than full-sequence distillation.

\paragraph{Model compression and pruning}
Structured pruning removes entire layers~\cite{men2024shortgpt}, attention heads, or intermediate dimensions.
LLM-Pruner~\cite{ma2024llmpruner} uses gradient-based importance scoring.
Closest to our recipe, Gromov \textit{et al.}~\cite{gromov2024unreasonable} prune contiguous deeper layers and \emph{heal} with light fine-tuning---the same merge-then-recover structure we use, but via deletion rather than averaging.
Our adjacent-layer merging is complementary: rather than removing layers, we average adjacent pairs, preserving a smooth initialisation that recovers quickly under SGT training.

\paragraph{Mixture of Experts}
Shazeer \textit{et al.}~\cite{shazeer2017outrageously} introduced sparsely-gated MoE for language models.
Switch Transformers~\cite{fedus2022switch} simplified routing to top-1 selection.
Recent frontier models (DeepSeek-V3~\cite{deepseekai2024v3}, GLM-5.2~\cite{glm52}) use MoE at trillion-parameter scale.
Assembling \emph{independently obtained} models is studied by Branch-Train-Merge~\cite{li2022branch} (parallel expert LMs combined by ensembling) and Model Soups~\cite{wortsman2022model} (weight averaging of fine-tuned models).
Our MoEE differs by obtaining experts via \emph{compression} of a shared backbone and assembling them with a learned, MTP-augmented router rather than by averaging weights (soups) or routing experts trained on divergent data (BTM).

\paragraph{Recurrent depth and parameter reuse}
Universal Transformers~\cite{dehghani2019universal} loop a single layer with adaptive halting; ALBERT~\cite{lan2020albert} ties one block's weights across all depths.
Geiping \textit{et al.}~\cite{geiping2025scaling} scale recurrent depth for test-time compute.
Our RDT combines heal-after-merge compression (cf.\ \cite{gromov2024unreasonable}) with ALBERT-style weight tying realised as recurrent unrolling of the merged core, recovering effective depth from a compressed parameter set.
\paragraph{Looped Transformers and compute-matched recurrence.}
Recent looped Mixture-of-Experts models such as Loopie~\cite{loopie2026}---which
reports frontier-level reasoning from a compute-matched looped MoE---argue that
recurrent depth becomes competitive when stored width, stored depth and
recurrent depth are jointly chosen under a fixed pre-training compute budget,
and situate recurrent parameter-block reduction, the direction of our RDT, within
the looped-compression literature. Our recurrent-representational-yield route is a
complementary merge-then-recover recipe---RDT heal-after-merge compression
(cf.\ \cite{gromov2024unreasonable}) plus SGT healing, measured on CHERRY
backbones---that converges with this frontier line of work on recurrent
compression. We align with that direction on compression only, and do
\emph{not} treat Loopie as validation of SGT, $\gamma$, or wave propagation,
which are specific to our selective objective.

\paragraph{Speculative decoding}
Leviathan \textit{et al.}~\cite{leviathan2023fast} and Chen \textit{et al.}~\cite{chen2023accelerating} use a small draft model to propose tokens verified by the target model.
Multi-token prediction heads~\cite{gloeckle2024better} enable self-speculative decoding without a separate draft model.
Our oracle MTP uses cross-attention prediction heads trained with future hidden states, targeting self-speculative decoding without an auxiliary model; head training is in progress and the speculative-decoding acceptance and speedup evaluation is deferred.

\section{The \hgera{} Architecture}
\label{sec:hgera}

The sovereign contribution at the centre of this work is an
architecture \emph{design}, which we call \hgera{} (Hybrid Gated
Expert Recurrent Attention), evolved from our earlier \hglqa{} hybrid
\cite{kwon2026hglqa}. \hgera{} composes four design elements that,
to our knowledge, have not previously been combined in a single
frontier-scale model: (i) a hybrid backbone that interleaves \phalanx{}
gated-linear-recurrence layers with \hoplite{} full-attention layers at
a $3{:}1$ ratio; (ii) a \forge{} mixture-of-experts (MoE) feed-forward
stage with a \strategos{} router; (iii) \oracle{} multi-token
prediction (MTP) heads; and (iv) an omnimodal routing path that admits
text, vision--language, and audio--visual inputs through one shared
backbone. We describe each element at the operator level; the specific
fused-kernel realizations and the detailed gating and routing
parameterizations are design-sovereign components that we do not
disclose here.

\paragraph{Sovereignty is a property of the design, not of a checkpoint.}
We use \emph{sovereign} throughout in a deliberately narrow sense: the
\hgera{} architecture and its training recipe are original to this work
and are not a configuration alias or a wrapper over an existing model
class. This is a claim about \emph{design provenance}, and it is
logically independent of how any particular set of weights was
obtained. We instantiate \hgera{} at two scales with deliberately
different weight and vocabulary provenance, and we state that
provenance precisely in \S\ref{subsec:hgera-instances} and
Table~\ref{tab:hgera-provenance} so that the design claim is not
misread as a claim that either instance was trained from scratch. In
particular, CHERRY-122B realises the \hgera{} design as an
independently specified 256-expert frontier architecture
(\S\ref{subsec:hgera-instances}), whereas CHERRY-12B realizes the same
design from randomly initialized weights under a de-novo Korean
tokenizer.

\subsection{Hybrid backbone: \phalanx{} and \hoplite{} layers}
\label{subsec:hgera-backbone}

Let a sequence of length $n$ be embedded to $\mathbf{X}\in\RR^{n\times d}$.
An $L$-layer \hgera{} backbone is the composition of per-layer operators
\begin{equation}
 \mathcal{H}_\theta \;=\; \mathcal{O}^{(L)}\!\circ\cdots\circ\,\mathcal{O}^{(1)},
 \qquad
 \mathcal{O}^{(\ell)} \in \{\,\mathcal{P}_{\theta_\ell},\ \mathcal{Q}_{\theta_\ell}\,\},
 \label{eq:hgera-compose}
\end{equation}
where the schedule $s:\{1,\dots,L\}\to\{P,Q\}$ assigns each layer a
\phalanx{} operator $\mathcal{P}$ or a \hoplite{} operator $\mathcal{Q}$.
We write $\mathcal{L}_P=\{\ell:s(\ell)=P\}$ and $\mathcal{L}_H=\{\ell:s(\ell)=Q\}$
and denote the recurrent-layer fraction $\rho=|\mathcal{L}_P|/L$.

\paragraph{\phalanx{} layers (gated linear recurrence).}
A \phalanx{} layer maintains a matrix-valued associative memory
$\mathbf{S}_t\in\RR^{d_k\times d_v}$ that is updated by a gated delta
rule of the gated delta-rule family~\cite{yang2024gateddeltanet},
\begin{align}
 \mathbf{S}_t &= \mathbf{S}_{t-1}\,\mathrm{diag}(\boldsymbol{\gamma}_t)
 \;+\; \beta_t\,\mathbf{k}_t\!\left(\mathbf{v}_t - \mathbf{S}_{t-1}^{\top}\mathbf{k}_t\right)^{\!\top},
 \label{eq:phalanx-state}\\
 \mathbf{o}_t &= \mathbf{S}_t^{\top}\,\mathbf{q}_t,
 \label{eq:phalanx-read}
\end{align}
with data-dependent forget gate $\boldsymbol{\gamma}_t\in(0,1)^{d_k}$ and
write strength $\beta_t\in[0,1]$. Equations~\eqref{eq:phalanx-state}--%
\eqref{eq:phalanx-read} give the operator form only; the specific
parameterization of $\boldsymbol{\gamma}_t$, $\beta_t$ and the
chunk-parallel kernel that realizes it are design-sovereign and are not
disclosed. \emph{Rationale:} the delta term
$(\mathbf{v}_t-\mathbf{S}_{t-1}^{\top}\mathbf{k}_t)$ provides explicit
key--value overwrite semantics that a state-space parameterization does
not naturally express, giving $\bigO(1)$-per-token decoding with a
bounded-size recurrent state.

\paragraph{\hoplite{} layers (grouped full attention with guard).}
A \hoplite{} layer applies exact grouped-query attention
\begin{equation}
 \mathrm{Attn}(\mathbf{Q},\mathbf{K},\mathbf{V})
 = \softmax\!\left(\frac{\mathbf{Q}\mathbf{K}^{\top}}{\sqrt{d_k}} + \mathbf{M}\right)\mathbf{V},
 \label{eq:hoplite-attn}
\end{equation}
with a causal mask $\mathbf{M}$, $G$ shared key/value groups, and a
guard normalization $g(\cdot)$ on the residual stream that stabilizes
the interface between recurrent and attention layers. \emph{Rationale:}
periodic exact attention corrects the approximation error accumulated by
the intervening \phalanx{} layers; under the $3{:}1$ schedule the
companion analysis \cite{kwon2026hglqa} shows this residual error
remains bounded by a constant independent of depth $L$
(via a contraction property of the softmax map), rather than growing as
$\bigO(L)$ in a purely recurrent stack.

\paragraph{Scheduling and cost.}
\hgera{} fixes the design invariant $\rho=3/4$ (three \phalanx{} layers
per \hoplite{} layer). The average per-token decoding cost is
\begin{equation}
 \bar{C}(n) \;=\; \rho\,L\cdot\Theta(d_k d_v)
 \;+\; (1-\rho)\,L\cdot\Theta(n\,d_k),
 \label{eq:hgera-cost}
\end{equation}
so at $\rho=3/4$ only one layer in four carries the term linear in
context length $n$, reducing the quadratic coefficient of the KV-cache
cost by $4\times$ relative to an all-attention stack. This directly
targets the single-device memory envelope that motivates the design
(\S\ref{subsec:hgera-instances}).

\subsection{\forge{} MoE with the \strategos{} router}
\label{subsec:hgera-forge}

Each layer's feed-forward stage is a \forge{} MoE. For a token
representation $\mathbf{x}$, a \strategos{} router $r(\cdot)$ selects a
top-$k$ subset $\mathcal{T}_k(\mathbf{x})$ of the $N_E$ experts and the
output is the gated combination
\begin{equation}
 \mathrm{FFN}(\mathbf{x}) \;=\;
 \sum_{e\in\mathcal{T}_k(\mathbf{x})} g_e(\mathbf{x})\,E_e(\mathbf{x}),
 \qquad
 g_e(\mathbf{x}) = \frac{\exp r_e(\mathbf{x})}{\sum_{e'\in\mathcal{T}_k(\mathbf{x})}\exp r_{e'}(\mathbf{x})} .
 \label{eq:forge-moe}
\end{equation}
Sparse activation ($k\ll N_E$) decouples parameter count from per-token
FLOPs, which is what makes a frontier-scale expert count tractable on a
single accelerator (\S\ref{subsec:hgera-instances}). To avoid the
router collapse to which a standard auxiliary load-balancing term is
prone, \strategos{} augments Eq.~\eqref{eq:forge-moe} with a
design-sovereign balancing objective, whose convergence to a balanced
assignment is established in the companion analysis
\cite{kwon2026hglqa}; the balancing mechanism itself is a moat
component and is not disclosed here.

\subsection{\oracle{} multi-token prediction}
\label{subsec:hgera-oracle}

\hgera{} attaches $H$ lightweight cross-attention prediction heads
$\{\phi_h\}_{h=1}^{H}$ on top of the shared backbone state
$\mathbf{z}_t$, one per look-ahead horizon $h$, trained with the
auxiliary objective
\begin{equation}
 \mathcal{L}_{\mathrm{MTP}}
 \;=\; \sum_{h=1}^{H} \lambda_h\;
 \EE\!\left[-\log p_{\phi_h}\!\left(x_{t+h}\mid \mathbf{z}_t\right)\right],
 \label{eq:oracle-mtp}
\end{equation}
where each head receives \emph{oracle} future hidden states during
training (hence the name) and $\lambda_h$ weights the horizons. At
inference the same heads are designed to drive self-speculative decoding
\cite{leviathan2023fast,chen2023accelerating,gloeckle2024better},
proposing $H$ tokens per step for single-pass verification.
\emph{Rationale:} the MTP heads densify the training signal and amortize
decode latency without a separate draft model. We report the design
here. On a deterministic greedy acceptance audit of trained heads
(20 prompts; horizons $h\in\{1,2,3\}$), per-horizon acceptance was
$0.3274$, $0.1033$ and $0.0470$. These values establish a measurable
draft signal; they are \emph{not} an end-to-end serving-time result.

\subsection{Omnimodal routing: text, VLM, and omni paths}
\label{subsec:hgera-omni}

A single \hgera{} backbone serves three input regimes---text-only,
vision--language (VLM), and full audio--visual (omni)---without a
separate per-modality model. Non-text modalities $m\in\mathcal{M}$ are
first encoded by frozen specialist extractors into features
$\mathbf{F}_m$, which a learned \cabstractor{} fuses into a fixed set of
latent tokens via cross-attention from $q$ learned queries
$\mathbf{Q}_{\mathrm{lat}}$ under soft-MoE gating,
\begin{equation}
 \mathbf{U} \;=\; \cabstractor\!\big(\{\mathbf{F}_m\}_{m\in\mathcal{M}};\,\mathbf{Q}_{\mathrm{lat}}\big),
 \qquad
 \mathbf{X}_{\mathrm{in}} \;=\; \big[\,\mathbf{U}\,;\,\mathrm{Embed}(\text{text})\,\big].
 \label{eq:omni-fuse}
\end{equation}
The fused latent tokens $\mathbf{U}$ are concatenated with the text
token embeddings into one sequence $\mathbf{X}_{\mathrm{in}}$ that the
shared backbone processes with the same hybrid layers of
Eq.~\eqref{eq:hgera-compose}; modality selection is a routing switch
over which extractor paths are active, so the VLM and omni regimes are
capability supersets of the text path rather than distinct networks.
\emph{Rationale and evidence:} because all modality tokens traverse the
same backbone, vision and audio features benefit from both the
long-range \phalanx{} memory and the token-precise \hoplite{} attention.
For the released sub-2B instance we verify the merge is lossless on the
text path---the omnimodal checkpoint reproduces the text-only
checkpoint's argmax output at every position on a held-out text probe
(\S\ref{sec:arch})---so adding modality encoders does not perturb
text-token prediction. Modality-specific quality of the frontier omni path is reported in a future version.

\subsection{Instantiation and provenance}
\label{subsec:hgera-instances}

The design of \S\S\ref{subsec:hgera-backbone}--\ref{subsec:hgera-omni}
is realized at two scales whose weight and vocabulary provenance differ,
and which therefore support different sovereignty claims
(Table~\ref{tab:hgera-provenance}).

\paragraph{CHERRY-122B (independently specified instance).}
The frontier member realizes \hgera{} as an independently specified
architecture: a hybrid backbone with 256-expert top-8 \forge{} MoE
feed-forward stages, alternating \phalanx{} (delta-net linear
attention) and \hoplite{} (full grouped-query attention every fourth
layer), and \oracle{} MTP\@. The architecture and training recipe are
specified from the open literature~\cite{groeneveld2024olmo,team2024gemma}
without copying upstream code or weights; the design, recipe, and
resulting parameters are sovereign. The 248K multilingual vocabulary is
not yet from-scratch: sovereign vocabulary unification to the 12B's
de-novo 128K Korean BPE is in progress.
The single-device residency and identity results for this instance are
reported in \S\ref{subsec:cherry122b}.

\paragraph{CHERRY-12B (from-scratch instance).}
The from-scratch member realizes the identical \hgera{} design from
randomly initialized weights under a de-novo $128$k Korean-centric BPE
tokenizer built on our own corpus. This is where the tokenizer
sovereignty claim is literally true: the vocabulary is trained from
scratch and yields a measured $28$--$46\%$ reduction in Korean token
count relative to general-purpose multilingual tokenizers. CHERRY-12B thus completes the sovereignty stack that the
122B instance leaves partial---architecture \emph{and} weights \emph{and}
tokenizer---at a scale that is trainable end-to-end under our compute
envelope.

\begin{table}[t]
\centering
\caption{\textbf{\hgera{} instantiation and provenance (measured where
marked).} The architecture design is sovereign to this work and is
independent of how a given instance's weights and vocabulary were
obtained. We label the two instances explicitly to prevent the design
claim from being read as a from-scratch claim for CHERRY-122B.}
\label{tab:hgera-provenance}
\small
\begin{tabular}{@{}p{3.1cm}p{3.9cm}p{4.6cm}@{}}
\toprule
\textbf{Aspect} & \textbf{CHERRY-12B (from scratch)} & \textbf{CHERRY-122B (independently specified)} \\
\midrule
Architecture (\hgera{}) & Sovereign design realized & Sovereign design realized \\
Backbone weights & Random-init, trained from scratch & Independently specified and trained \\
Tokenizer / vocab & De-novo 128k Korean BPE ($-28$--$46\%$ Korean tokens, measured) & 248K multilingual vocab (unification to sovereign BPE in progress) \\
Backbone / MoE & HGERA hybrid backbone (config finalized at training) & 256-expert top-8 \forge{} MoE; hybrid \phalanx{}/\hoplite{} \\
Sovereignty claim & Architecture + weights + tokenizer & Architecture + weights + recipe (tokenizer pending) \\
\bottomrule
\end{tabular}
\end{table}

\begin{figure}[t]
\centering
\begin{tikzpicture}[
 >=Stealth,
 blk/.style={draw, rounded corners=3pt, font=\scriptsize\sffamily, thick,
 minimum width=2.5cm, minimum height=0.55cm},
 pha/.style={blk, fill=blue!10},
 hop/.style={blk, fill=orange!18},
 moe/.style={blk, fill=violet!12, minimum width=2.5cm},
 enc/.style={blk, fill=teal!18, minimum width=1.5cm, minimum height=0.45cm},
 mtp/.style={blk, fill=red!14, minimum width=1.3cm, minimum height=0.45cm},
 arr/.style={->, thick, gray!65!black},
]
\node[enc] (v1) at (-4.6,0.0) {Vision};
\node[enc] (v2) at (-4.6,0.6) {Fine-gr.};
\node[enc] (v3) at (-4.6,1.2) {Multiling.};
\node[enc] (a1) at (-4.6,1.8) {Audio};
\node[blk, fill=teal!30, minimum width=1.6cm] (cab) at (-4.6,2.7) {\cabstractor};
\foreach \e in {v1,v2,v3,a1}{\draw[arr, teal!60] (\e.east) -- (cab.south west);}
\node[blk, fill=gray!12, minimum width=3.0cm] (in) at (0,0)
 {Unified sequence $[\mathbf{U};\text{text}]$};
\draw[arr, teal!60] (cab.east) -| ([xshift=-1.0cm]in.north);
\node[pha] (p1) at (0,1.0) {\phalanx};
\node[pha] (p2) at (0,1.7) {\phalanx};
\node[pha] (p3) at (0,2.4) {\phalanx};
\node[hop] (h1) at (0,3.1) {\hoplite\ (full attn)};
\node[moe] (m1) at (0,3.9) {\forge{} MoE $+$ \strategos};
\draw[arr] (in.north) -- (p1.south);
\draw[arr] (p1.north) -- (p2.south);
\draw[arr] (p2.north) -- (p3.south);
\draw[arr] (p3.north) -- (h1.south);
\draw[arr] (h1.north) -- (m1.south);
\node[font=\tiny\sffamily, text=gray] at (2.7,2.1)
 {$\times\,L/4$ blocks};
\draw[decorate, decoration={brace, amplitude=4pt}, gray]
 (1.35,0.75) -- (1.35,3.35);
\node[mtp] (t1) at (3.4,3.4) {MTP $h{=}1$};
\node[mtp] (t2) at (3.4,3.9) {MTP $h{=}2$};
\node[mtp] (t3) at (3.4,4.4) {MTP $h{=}H$};
\node[blk, fill=green!12, minimum width=2.6cm] (out) at (0,4.7)
 {Output $+$ speculative tokens};
\draw[arr] (m1.north) -- (out.south);
\foreach \t in {t1,t2,t3}{\draw[arr, red!50] (m1.east) -| (\t.west);}
\end{tikzpicture}
\caption{\textbf{The \hgera{} design.} One \hgera{} block interleaves
three \phalanx{} gated-recurrence layers with one \hoplite{}
full-attention layer ($\rho=3/4$) and a \forge{} MoE feed-forward stage
routed by \strategos{}; the block repeats $L/4$ times.
Frozen modality extractors are fused by the \cabstractor{} into latent
tokens that share the sequence with text embeddings
(Eq.~\eqref{eq:omni-fuse}), so text, VLM, and omni inputs traverse one
backbone. \oracle{} MTP heads at horizons $h\in\{1,\dots,H\}$ drive
self-speculative decoding. The figure depicts the sovereign
\emph{design}; instance-specific weight and vocabulary provenance is
stated in Table~\ref{tab:hgera-provenance}.}
\label{fig:hgera}
\end{figure}
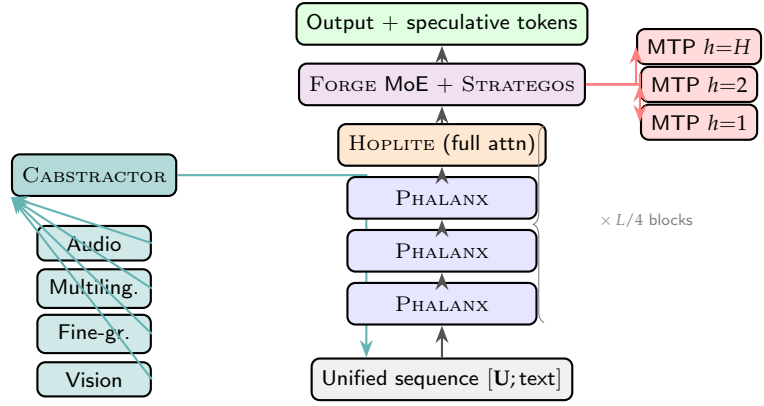

\section{Approach}
\label{sec:approach}

\subsection{Overview}

The \cherry framework operates in four phases:

\begin{enumerate}
\item \textbf{Selective training}: Identify semantically decisive
tokens (GT set $\mathcal{G}$) and train with a mixed loss
$\mathcal{L}_{\mathrm{SGT}}$ that upweights GT positions.
\item \textbf{Compression}: Reduce model depth via adjacent-layer
merging.
\item \textbf{Recurrence}: Restore effective depth through learned
loop unrolling of the compressed model's middle layers.
\item \textbf{Fusion}: Assemble compressed models as MoE experts
with SGT-guided distillation from frontier teachers.
\end{enumerate}

Each phase builds on the previous: selective training provides
the signal quality that helps compression recovery, compression
provides the parameter efficiency that makes expert fusion
practical, and fusion provides representational diversity that
improves quality at a fixed active-parameter budget.

\paragraph{What is sovereign in this recipe}
We separate three provenance claims and hold them apart throughout.
The \emph{design} is ours: the HGERA architecture (\S\ref{sec:hgera}),
the training recipe of this section---selective ground-truth training
(SGT), wave propagation, SSGT self-correction, and Oracle multi-token
prediction---and the Korean-first data and evaluation stack. The recipe
is validated end-to-end at the 1.8B scale and is architecture-agnostic;
it is the sovereign layer that all three \cherry instances share
(\S\ref{sec:family}). The frontier \cherry-122B instance is
independently specified from the open
literature~\cite{groeneveld2024olmo,team2024gemma}---its architecture,
training recipe, and resulting weights are sovereign; the 248K
multilingual vocabulary is the remaining non-sovereign component at this
scale, with unification to the 12B's de-novo BPE in progress
(\S\ref{sec:family}).

\subsection{Selective Ground Truth Token Training (SGT)}
\label{sec:wave}

\subsubsection{GT token hierarchy}

\begin{definition}[GT Token Hierarchy]
Given a response $\mathbf{y} = (y_1, \ldots, y_m)$, we define
nested token sets:
$\mathcal{F} = \{1,\ldots,m\}$ (full);
$\mathcal{G} \subseteq \mathcal{F}$ (ground truth: factual
entities, answer keywords, reasoning pivots;
$|\mathcal{G}|/|\mathcal{F}| \approx 0.15$--$0.20$);
$\mathcal{G}^* \subseteq \mathcal{G}$ (super GT: single most
decisive token per answer span);
$\mathcal{G}^{**} \subseteq \mathcal{G}^*$ (super-super GT:
meta-cognitive pivots that trigger state transitions).
\end{definition}

\subsubsection{GT token identification}
\label{sec:gt-id}

GT tokens are identified via a hybrid pipeline that respects
the model's chat template, reasoning format, and tool-calling
conventions:

\begin{enumerate}
\item \textbf{Chat-template-aware preprocessing.}
 Training data is fully mapped to the target model's chat
 template: system prompt, multi-turn history, reasoning and
 tool-calling special tokens (e.g.\ \texttt{<|spartan\_think|>};
 \textit{spartan} is an organisational codename retained in the
 vocabulary), and structured-output delimiters. This ensures GT
 identification operates on the exact token sequence the model
 will encounter at inference.

\item \textbf{Rule-based seeding.}
 Named entities, numerical values, answer-final tokens,
 first tokens of structured output fields, closing
 delimiters (quotes, braces), and generic/property
 references in tool-call schemas are tagged using
 part-of-speech and NER heuristics.

\item \textbf{Contrastive refinement.}
 For each candidate token $y_i$, we compute the activation
 difference
 $\delta_i^\ell = \|\mathbf{h}(y_i^+) -
 \mathbf{h}(y_i^-)\|_2$
 between a correct and corrupted response at layer $\ell$.
 Tokens exceeding $\delta_i > \mu + \sigma$ are promoted
 to $\mathcal{G}$.
\end{enumerate}

\paragraph{GT type taxonomy}
We distinguish six functional categories of GT tokens, each
addressing a distinct aspect of model behavior
(Table~\ref{tab:gt-types}):

\begin{table}[H]
\centering
\caption{GT token taxonomy with category-specific selection
criteria. All categories participate in the same SGT loss
(Eq.~\ref{eq:sgt}).}
\label{tab:gt-types}
\small
\begin{tabular}{@{}lp{5.5cm}l@{}}
\toprule
Category & Selection criterion & Level \\
\midrule
Factual & Entity names, numbers, answer keywords & $\mathcal{G}$ \\
Structural & Turn-initial, closing delimiters, schema
 properties, literal values & $\mathcal{G}$ \\
Reasoning & Core CoT sentence pivots, logical
 connectives carrying entailment & $\mathcal{G}^*$ \\
Self-correction & ``Atcha'' pivots, ``Jamkkan'' verification
 triggers, direction-change markers & $\mathcal{G}^{**}$ \\
Guardrail & Refusal triggers, safety-boundary tokens,
 mid-stream breaking points & $\mathcal{G}$ \\
Verification & Fact-check requests, tool-call invocations,
 rubric/strategy planning tokens & $\mathcal{G}^*$ \\
\bottomrule
\end{tabular}
\end{table}

\paragraph{Partial causal formulation}
For instruction-tuning data, SGT uses a \textit{partial
causal} setup: given the full chat context up to the GT
region boundary, the model is asked to generate only the GT
token span causally. Formally, for a GT span
$[g_\text{start}, g_\text{end}] \subseteq \mathcal{G}$:
\begin{equation}
\mathcal{L}_{\text{partial}} =
-\sum_{t=g_\text{start}}^{g_\text{end}}
\log p_\theta(y_t \mid y_{<t})
\end{equation}
where $y_{<t}$ includes the full prefix (system prompt,
multi-turn context, response up to $g_\text{start} - 1$).
This is more efficient than full-sequence causal training:
only $|\mathcal{G}|/|\mathcal{F}| \approx 15\%$ of positions
contribute to the loss, yet wave propagation
(Theorem~\ref{thm:wave}) ensures the remaining 85\% improve.

\paragraph{Empirical statistics}
The GT ratio on our training data is
$|\mathcal{G}|/|\mathcal{F}| = 0.153 \pm 0.021$ across
1{,}024 sampled examples. By category: Factual 41\%,
Structural 23\%, Reasoning 18\%, Self-correction 8\%,
Guardrail 6\%, Verification 4\%. When $|\mathcal{G}| = 0$
for a given example (rare; $<$0.3\%), the SGT loss reduces
to standard cross-entropy.

\subsubsection{SGT loss}

\begin{equation}
\mathcal{L}_{\mathrm{SGT}} =
\alpha \cdot \frac{1}{|\mathcal{G}|}
\sum_{i \in \mathcal{G}} \ell(f_\theta(\mathbf{y}_{<i}), y_i) +
(1 - \alpha) \cdot \frac{1}{|\mathcal{F}|}
\sum_{i \in \mathcal{F}} \ell(f_\theta(\mathbf{y}_{<i}), y_i)
\label{eq:sgt}
\end{equation}

The mixing coefficient $\alpha \in [0,1]$ balances GT-focused
and full-causal supervision. We find $\alpha = 0.7$ optimal.

\subsubsection{Wave propagation: theory}

\begin{definition}[Wave factor]
\label{def:wave}
Given a model trained for $T$ steps, the \textbf{wave factor}
$\mathcal{W}$ measures the ratio of non-GT to GT loss
reduction on a \emph{held-out evaluation set}:
\begin{equation}
\mathcal{W} = \frac{\Delta_{\bar{\mathcal{G}}}}
{\Delta_\mathcal{G}}
= \frac{\mathcal{L}_{\bar{\mathcal{G}}}^{(0)} -
\mathcal{L}_{\bar{\mathcal{G}}}^{(T)}}
{\mathcal{L}_\mathcal{G}^{(0)} -
\mathcal{L}_\mathcal{G}^{(T)}}
\end{equation}
where superscripts denote training step. $\mathcal{W} > 1$
means non-GT tokens benefit \emph{more} than directly
supervised tokens \emph{in absolute loss}. Unless otherwise
noted, we report $\mathcal{W}$ on the 1.019B text backbone of
CHERRY-1.8B (the ``1.0B backbone''; total model 1.829B)
evaluated on a held-out set of 1{,}600 examples disjoint from
training. The multi-seed results
(Table~\ref{tab:multiseed}) report $\mathcal{W}$ on
\emph{training data} at the 1.0B depth-up-scaled
(DUS~\cite{kim2024solar}) scale, which yields larger
values due to in-distribution measurement---these are clearly
labelled and treated as in-distribution diagnostics, not
generalisation claims.

\paragraph{A caveat on absolute vs.\ normalised wave factor}
Because $\mathcal{W}$ compares \emph{absolute} loss deltas, it
is sensitive to the baseline gap: non-GT tokens begin at higher
loss ($\mathcal{L}_{\bar{\mathcal{G}}}^{(0)} = 2.742$) than GT
tokens ($\mathcal{L}_\mathcal{G}^{(0)} = 0.925$) and therefore
have more absolute room to improve. We thus also report a
\textbf{baseline-normalised} wave factor
\begin{equation}
\label{eq:wave-norm}
\mathcal{W}_{\mathrm{norm}} =
\frac{\Delta_{\bar{\mathcal{G}}}/\mathcal{L}_{\bar{\mathcal{G}}}^{(0)}}
{\Delta_\mathcal{G}/\mathcal{L}_\mathcal{G}^{(0)}}
\end{equation}
(the ratio of \emph{fractional} reductions). On the held-out
set, $\mathcal{W} = 1.91$ but $\mathcal{W}_{\mathrm{norm}} =
0.61/0.94 = 0.65 < 1$: GT tokens improve more in fractional
terms. We are explicit about this throughout: the headline
empirical result is \emph{not} that the wave factor exceeds
unity (it does not, once normalised, and---as we show
below---$\mathcal{W}>1$ in absolute terms also holds under
ordinary full-sequence training), but that non-GT tokens
improve \emph{substantially under indirect supervision}, at a
\emph{per-supervised-token efficiency} of $4.5\times$
(Table~\ref{tab:wave}).
\end{definition}

\begin{theorem}[Non-GT descent under gradient coupling]
\label{thm:wave}
Let $f_\theta$ be a transformer with $L$ layers and
position-shared weights. Let $\mathcal{G}$ and
$\bar{\mathcal{G}}$ partition the output positions into
GT and non-GT sets, with losses $\mathcal{L}_\mathcal{G}$,
$\mathcal{L}_{\bar{\mathcal{G}}}$. Define the
\textbf{gradient coupling coefficient}
\begin{equation}
\label{eq:gamma}
\gamma \;:=\;
\frac{\langle\nabla_\theta\mathcal{L}_{\bar{\mathcal{G}}},\,
\nabla_\theta\mathcal{L}_\mathcal{G}\rangle}
{\|\nabla_\theta\mathcal{L}_\mathcal{G}\|^2},
\end{equation}
the projection of the non-GT gradient onto the GT gradient
direction (a regression coefficient). Suppose:
\begin{enumerate}[nosep,leftmargin=*]
\item[\textbf{(A1)}] \textbf{Smoothness.} $\mathcal{L}_{\bar{\mathcal{G}}}$
 is $\beta$-smooth:
 $\|\nabla\mathcal{L}_{\bar{\mathcal{G}}}(\theta') -
 \nabla\mathcal{L}_{\bar{\mathcal{G}}}(\theta)\|
 \leq \beta\|\theta' - \theta\|$.
\item[\textbf{(A2)}] \textbf{Positive coupling.}
 $\gamma > 0$: the non-GT and GT gradients are positively
 aligned. (We do \emph{not} assume this; we derive its
 structural origin in the proof and measure
 $\bar\gamma = 0.72 \pm 0.08$.)
\item[\textbf{(A3)}] \textbf{Small learning rate.}
 $\eta < 2\gamma/\beta$.
\end{enumerate}
Then a single gradient step $\theta' = \theta -
\eta\nabla_\theta\mathcal{L}_\mathcal{G}$ on GT tokens
\emph{decreases} the non-GT loss:
\begin{equation}
\label{eq:wave-bound}
\mathcal{L}_{\bar{\mathcal{G}}}(\theta') \leq
\mathcal{L}_{\bar{\mathcal{G}}}(\theta) -
\eta\Bigl(\gamma - \tfrac{\beta\eta}{2}\Bigr)
\|\nabla_\theta\mathcal{L}_\mathcal{G}\|^2,
\end{equation}
i.e.\ $\Delta_{\bar{\mathcal{G}}} \geq
\eta(\gamma - \tfrac{\beta\eta}{2})
\|\nabla_\theta\mathcal{L}_\mathcal{G}\|^2 > 0$
under (A3). Equivalently, the wave factor satisfies
$\mathcal{W} = \Delta_{\bar{\mathcal{G}}}/\Delta_\mathcal{G}
\geq \gamma - \tfrac{\beta\eta}{2}$ to first order in $\eta$.
\end{theorem}

\begin{remark}[What this theorem does and does not claim]
\label{rem:scope}
Theorem~\ref{thm:wave} guarantees only that non-GT loss
\emph{decreases} ($\mathcal{W} > 0$): training on GT tokens
helps non-GT tokens. It does \emph{not} prove
$\mathcal{W} > 1$ (that non-GT improves \emph{more} than GT);
indeed $\mathcal{W} > 1$ in absolute terms is, as we show in
Eq.~\eqref{eq:wave-norm} and the $\alpha{=}0$ analysis
below, largely a consequence of the higher non-GT baseline and
holds even under ordinary full-sequence training. The
theorem is a token-level instance of the well-known
auxiliary-task-transfer
mechanism~\cite{yu2020gradient,jacot2018ntk}: a gradient
step on one loss helps another loss exactly when their
gradients are positively aligned ($\gamma > 0$), the same
quantity studied in multi-task gradient
surgery~\cite{yu2020gradient} and the empirical neural
tangent kernel~\cite{jacot2018ntk}. Our contribution is not
the existence of this coupling but its \emph{magnitude}
($\bar\gamma = 0.72$), its \emph{dependence on linguistic
coherence} (Corollary~\ref{cor:coherence}), and its
\emph{deliberate exploitation} for per-token training
efficiency.

Two further scope restrictions apply.
\emph{(i)~Optimiser.} The theorem analyses the plain
gradient step $\theta' = \theta -
\eta\nabla_\theta\mathcal{L}_\mathcal{G}$, the SGD
idealisation standard in gradient-surgery and kernel-transfer
analyses~\cite{yu2020gradient,jacot2018ntk}. Our experiments
train with AdamW (Section~\ref{sec:experiments}), whose update
is a diagonally preconditioned, momentum-filtered direction
with decoupled weight decay; for that update the operative
quantity is the inner product of
$\nabla_\theta\mathcal{L}_{\bar{\mathcal{G}}}$ with the
preconditioned momentum, and $\gamma > 0$ alone does not
certify its sign. The theorem therefore \emph{predicts}
descent for the idealised step; that the prediction survives
the actual optimiser is established empirically, not proved
($\mathcal{W} > 0$ at every SGT checkpoint,
Table~\ref{tab:wave}; collapse exactly when the theory
predicts no coupling, Corollary~\ref{cor:coherence}).
\emph{(ii)~One step, not a trajectory.} The bound is
instantaneous: it certifies descent of the
training-distribution non-GT loss at a point $\theta$ where
$\gamma(\theta) > 0$. It does not bound a multi-hundred-step
trajectory, along which $\gamma_t$ can shrink or change
sign---the $\alpha{=}1$ collapse at 0.5B
($\mathcal{W} = -1.32$, Observation~\ref{prop:collapse})
shows the positive-coupling regime can be lost, and held-out
non-GT loss is non-monotone across checkpoints (0.862 at
step~100 vs.\ 1.075 at step~500, Table~\ref{tab:wave}).
The reported $\mathcal{W}$ is a trajectory-level diagnostic
that the theorem \emph{motivates} but does not bound.
Accordingly, statements elsewhere that wave propagation
``guarantees'' or ``ensures'' non-GT improvement are to be
read in this conditional, single-step sense: descent is
predicted whenever the coupling remains positive, not
unconditionally over training.
\end{remark}

\begin{proof}
\textbf{Step 1: First-order non-GT loss change.}
A first-order Taylor expansion of
$\mathcal{L}_{\bar{\mathcal{G}}}$ around $\theta$ along the
GT-gradient step gives
\begin{equation}
\label{eq:taylor}
\mathcal{L}_{\bar{\mathcal{G}}}(\theta')
= \mathcal{L}_{\bar{\mathcal{G}}}(\theta)
- \eta \left\langle
\nabla_\theta\mathcal{L}_{\bar{\mathcal{G}}},\,
\nabla_\theta\mathcal{L}_\mathcal{G} \right\rangle
+ \mathcal{O}(\eta^2)
= \mathcal{L}_{\bar{\mathcal{G}}}(\theta)
- \eta\gamma\|\nabla_\theta\mathcal{L}_\mathcal{G}\|^2
+ \mathcal{O}(\eta^2),
\end{equation}
using the definition of $\gamma$ in Eq.~\eqref{eq:gamma}.

\textbf{Step 2: Second-order bound via smoothness.}
By $\beta$-smoothness (A1), the standard descent lemma bounds
the remainder:
$\mathcal{L}_{\bar{\mathcal{G}}}(\theta') \leq
\mathcal{L}_{\bar{\mathcal{G}}}(\theta)
- \eta\gamma\|\nabla_\theta\mathcal{L}_\mathcal{G}\|^2
+ \frac{\beta\eta^2}{2}\|\nabla_\theta
\mathcal{L}_\mathcal{G}\|^2$, which is
Eq.~\eqref{eq:wave-bound}. Under (A3), $\eta < 2\gamma/\beta
\Rightarrow \gamma - \beta\eta/2 > 0$, so the bound certifies
a strict decrease. This establishes $\mathcal{W} > 0$; it does
\emph{not}, on its own, establish $\mathcal{W} > 1$
(Remark~\ref{rem:scope}).

\textbf{Step 3: Structural origin and sign of $\gamma$.}
The coupling is not a free assumption; it has an explicit
attention-mediated form. The gradient of
$\mathcal{L}_\mathcal{G}$ w.r.t.\ the layer-$\ell$ value
weights is
$\partial\mathcal{L}_\mathcal{G}/\partial W_V^{(\ell)}
= \frac{1}{|\mathcal{G}|}\sum_{i\in\mathcal{G}}\mathbf{e}_i^\ell
\sum_j a_{ij}^\ell\,\mathbf{h}_{\ell-1}(j)^\top$,
where $\mathbf{e}_i^\ell$ is the per-position error signal and
$a_{ij}^\ell$ the attention weight. The $W_V^{(\ell)}$
contribution to the numerator of $\gamma$ is therefore
\begin{equation}
\label{eq:inner-product}
\Bigl\langle
\tfrac{\partial\mathcal{L}_{\bar{\mathcal{G}}}}{\partial W_V^{(\ell)}},\,
\tfrac{\partial\mathcal{L}_\mathcal{G}}{\partial W_V^{(\ell)}}
\Bigr\rangle_F
= \frac{1}{|\mathcal{G}||\bar{\mathcal{G}}|}
\sum_{\substack{i \in \mathcal{G}\\k \in \bar{\mathcal{G}}}}
\underbrace{\langle\mathbf{e}_i^\ell,\,
\mathbf{e}_k^\ell\rangle}_{\text{error alignment}}
\underbrace{\sum_j a_{ij}^\ell \,
\mathbf{h}_{\ell-1}(j)^\top \mathbf{h}_{\ell-1}(k)}_
{\text{representation overlap}}.
\end{equation}
Both factors are typically positive on coherent text: nearby
hidden states have positive inner product ($> 0.6$ cosine
within a 32-token window) and error signals for tokens in the
same prediction are aligned. This makes $\gamma > 0$ the
expected case, but---we stress---\emph{not} a guarantee: error
signals can anti-align, so positivity is an empirical property
of natural language, which we measure directly:
\begin{equation}
\label{eq:gamma-empirical}
\bar{\gamma} = 0.72 \pm 0.08
\quad\text{(mean $\pm$ std, 10 batches, step 100;
$\gamma > 0$ in all batches).}
\end{equation}

\textbf{Step 4: Magnitude is empirical.}
The theorem predicts $\mathcal{W} \gtrsim \gamma$ to first
order, i.e.\ $\mathcal{W} \gtrsim 0.7$. The measured held-out
$\mathcal{W} = 1.91$ exceeds this because (i) the bound is
first-order and (ii) $\mathcal{W}$ compares absolute deltas
across classes with different baselines (Eq.~\eqref{eq:wave-norm}).
We do not claim the theorem predicts the magnitude; it
predicts the \emph{sign} ($\mathcal{W} > 0$), which we then
characterise empirically. \qed
\end{proof}

\begin{corollary}[Coherence-dependent coupling]
\label{cor:coherence}
When input text is semantically incoherent (e.g., randomly
shuffled tokens), the representation-overlap factor in
Eq.~\eqref{eq:inner-product} vanishes in expectation for
distant random tokens, driving $\gamma \to 0$. By
Eq.~\eqref{eq:wave-bound} the guaranteed non-GT descent then
vanishes. Empirically the
absolute wave factor collapses to $\mathcal{W} = 0.12 \pm
0.04$ on shuffled input (indistinguishable from zero within
noise), confirming the \emph{direction} of this prediction. We state the
control's confound explicitly rather than over-read it: on
shuffled input the non-GT \emph{targets} are themselves
near-random, so their loss has little room to fall regardless
of any coupling, and $\gamma$ was not re-measured on shuffled
input. The shuffled collapse is therefore \emph{one-sided
consistency evidence} for the predicted $\gamma \to 0$, not an
independent measurement of it.
\end{corollary}

\paragraph{Direct evidence: anchor-free propagation on coherent text}
Two facts limit what the preceding checks can establish: under
the mixed loss, non-GT tokens retain a direct-supervision path
(Limitations, item~10); and $\bar\gamma$ is measured at
step~100 of the same training run whose held-out wave factor it
is compared against, so $\mathcal{W} = 1.91 \geq \bar\gamma$
(Step~4 of the proof of Theorem~\ref{thm:wave}) is an
\emph{in-run consistency check}, not an out-of-run
prediction---the $\alpha{=}0$ baseline alone yields
$\mathcal{W} = 14.76$, so clearing the bound is undemanding.
The direct test of propagation removes the anchor entirely.
In separate $\alpha{=}1.0$ runs on the 1.0B backbone (loss
computed \emph{only} on the selected tokens; 500-example
corpus, AdamW, learning rate $10^{-5}$, 2 epochs, disjoint
held-out split), supervising 12--19\% of tokens reduces
\emph{total} held-out loss by up to 1.39 nats from a 2.72
baseline. Because the supervised subset is at most 19\% of
positions, most of this reduction must occur at positions that
never entered the loss: propagation from a sparse supervised
subset to unsupervised positions on coherent text is real,
large, and free of the direct-supervision confound. Two
qualifications accompany this, our cleanest wave measurement.
(i)~The effect is \emph{maximised by random} subsets
(1.35--1.39 nats at matched budgets); GT-selected subsets
transfer far less (0.19--0.38 nats across our selection rules),
and the taxonomy-based subset overfits in these runs (held-out
non-GT loss \emph{rises} from 2.74 to 3.01). The measurement
thus evidences the coupling mechanism of
Theorem~\ref{thm:wave}, not any specialness of GT-token
selection (Limitations, item~9), and it refines
Observation~\ref{prop:collapse}: the $\alpha{=}1$ failure mode
depends on \emph{which} subset is supervised, not on anchor
removal per se. Consistently, re-running the same comparison
at $\alpha{=}0.7$ collapses every selection rule to an
essentially identical total reduction (1.43--1.46 nats,
${\approx}$ full-sequence training): under the mixed loss the
anchor dominates, as item~10 states.
(ii)~No out-of-run $\gamma \to \mathcal{W}$ prediction has yet
been tested: $\bar\gamma$ measured early in a fresh run and
committed in advance against that run's final wave factor is
the outstanding falsification test, which we leave to future
work.

\begin{observation}[Anchor necessity and collapse severity]
\label{prop:collapse}
The full-causal anchor ($\alpha < 1$) is \emph{necessary} for
stable training, but the failure mode at $\alpha = 1$ is
scale-dependent (Table~\ref{tab:alpha-sweep}). At 0.5B, pure
GT supervision \emph{collapses}: non-GT loss \emph{increases}
($\mathcal{W} = -1.32$). At 1.0B~(DUS) and 1.2B it merely
\emph{degrades}: non-GT still improves but weakly
($\mathcal{W} = 0.69, 0.49$). So $\mathcal{W} > 0$ is not
sharply lost at a single boundary $\alpha^*$; rather, the
\emph{quality} of supervision degrades as $\alpha \to 1$, with
larger models more robust to the loss of the anchor. We
therefore report this as an empirical regularity, not a proven
threshold; deriving a scale-dependent $\alpha^*$ from $\gamma$,
the anchor-gradient norm, and $\beta$ is left to future work.
\end{observation}

\paragraph{The $\alpha = 0$ control: $\mathcal{W} > 1$ is not created by SGT}
We confront the most important control directly. At
$\alpha = 0$ (ordinary full-sequence training, no GT selection
at all), the 1.0B DUS model shows $\mathcal{W} = 14.76$---%
\emph{higher} than every SGT configuration
(Table~\ref{tab:alpha-sweep}). This is decisive for
interpreting the wave factor: the fact that non-GT tokens
improve more in \emph{absolute} loss than GT tokens
($\mathcal{W} > 1$) is \textbf{a generic property of training on
natural language}---semantically loaded GT tokens start at low
loss with little room to fall, while syntactic non-GT tokens
start high and fall far---and is \textbf{not produced by SGT}.
We do not, therefore, present $\mathcal{W} > 1$ as evidence for
SGT. SGT's contribution is \emph{efficiency}: it achieves a
comparable non-GT improvement while applying direct supervision
to only 15\% of tokens, a $4.5\times$ per-supervised-token
return (Eq.~\eqref{eq:wave-norm}, Table~\ref{tab:wave}). The
genuine, SGT-specific phenomenon---non-GT improvement carried by
positive gradient coupling rather than direct supervision---is
isolated not by $\mathcal{W} > 1$ but by the coherence contrast
of Corollary~\ref{cor:coherence} ($\bar\gamma = 0.72$ on
coherent text vs.\ $\to 0$ on shuffled text) and by the
collapse of pure-GT training when the anchor is removed
(Observation~\ref{prop:collapse}).

\subsubsection{``Atcha'' moment induction via SSGT}
\label{sec:aha}

A central hypothesis of this work is that meta-cognitive state
transitions---the ``Aha moment'' observed emergently in
DeepSeek-R1~\cite{deepseek2025r1}---can be \emph{induced}
rather than \emph{discovered}. We term these transitions
\textbf{``Atcha'' moments}---the Korean equivalent of
``Aha''---and demonstrate that explicitly training on the
minimal token set that triggers self-correction is sufficient
to install meta-cognitive capabilities.

\paragraph{Core hypothesis}
DeepSeek-R1~\cite{deepseek2025r1} reported the ``Aha moment''
as an \emph{emergent} phenomenon discovered during RL training.
We hypothesize that such moments are not rare emergences but
\emph{inducible} through targeted supervision: training on
the single pivot token that triggers self-correction, plus the
corrected answer keyword, should propagate to the entire
reasoning chain via wave propagation. We call these
\textbf{SSGT tokens} ($\mathcal{G}^{**}$): the minimum viable
supervision for inducing meta-cognitive state transitions.

\paragraph{``Atcha'' moment taxonomy}
We formalize three distinct meta-cognitive moments, each
requiring different SSGT token configurations:

\begin{table}[H]
\centering
\caption{Meta-cognitive moment taxonomy. Each moment type
defines a specific SSGT token pattern for training.}
\label{tab:moments}
\small
\begin{tabular}{@{}lp{4.5cm}p{4.5cm}@{}}
\toprule
Moment & Trigger context & SSGT tokens ($\mathcal{G}^{**}$) \\
\midrule
\textbf{Atcha} &
 Model mid-stream in hallucinated response
 (rejected-quality output) &
 Correction pivot (``Atcha, teullyeotda'' (oops, that's wrong)) + corrected answer keyword \\
\textbf{Jamkkan} &
 Uncertain RAG context; model unsure whether
 retrieved evidence supports the claim &
 Verification trigger (``Jamkkan'') + fact-check
 strategy + verified conclusion \\
\textbf{Guardrail} &
 Jailbreak attempt or harmful request during
 assistant streaming &
 Breaking token + refusal pivot +
 safety-compliant redirect \\
\bottomrule
\end{tabular}
\end{table}

\paragraph{Experimental design}
For each moment type, we construct training examples where:
\begin{itemize}
\item The full chat context up to the moment trigger is
 provided as input (partial causal formulation,
 Section~\ref{sec:gt-id}).
\item Only the SSGT tokens (2--5 tokens per example) are
 supervised via the SGT loss ($\alpha = 0.8$).
 Supervised-token counts vary by moment family (2--5 here,
 ${\sim}2$ for the arithmetic stress test of
 Section~\ref{sec:moment-obs}, 1.8 for guardrail breaking),
 reflecting each family's pivot-phrase length.
\item The hypothesis: if wave propagation holds for SSGT,
 then training on $<$0.5\% of total tokens should reshape
 the \emph{entire} reasoning trajectory.
\end{itemize}

\paragraph{Results}
\begin{itemize}
\item \textbf{Retention}: 97.6\% at 1.2B scale, 116.9\%
 at 0.5B. The $>$100\% retention at 0.5B indicates that
 SSGT \emph{outperforms} full-sequence training. We
 attribute this to an implicit regularisation effect:
 at small scales, full-sequence supervision overfits to
 syntactic tokens; concentrating supervision on meta-cognitive
 pivots avoids this dilution. \emph{Definition.} Retention is
 the ratio of held-out eval-loss improvement achieved by SSGT
 training to that achieved by full-causal training on the same
 data, $\mathrm{retention} = 100 \times
 \Delta\ell_{\mathrm{SSGT}} / \Delta\ell_{\mathrm{full}}$, where
 $\Delta\ell$ is base-model eval loss minus post-training eval
 loss on the 200-example held-out CoT evaluation with explicit
 self-correction sequences (the eval data behind
 Table~\ref{tab:ssgt}; on pivot-free MCQ data the same
 definition yields $-$35\%, Section~\ref{sec:failures}). Both
 arms train on the identical examples
 (Appendix~\ref{app:ssgt-detail}) under the same 500-step
 budget and differ only in the per-token loss weighting
 ($\alpha{=}0.8$ vs.\ $\alpha{=}0$), so the tokens processed are
 matched. Each arm is a single training run and the
 full-causal denominator is itself a one-run estimate, so we
 attach no significance test to retention (see Limitations~(8),
 Section~\ref{sec:discussion}); the 0.5B point estimate should
 be read as parity-or-better rather than a certified gain, and
 the load-bearing result is near-parity (97.6\%) at 1.2B. The
 two scales are also different external base
 checkpoints---HyperCLOVA X 0.5B~\cite{yoo2024hyperclovax} and
 EXAONE-family 1.2B~\cite{lgai2024exaone}---per-family results,
 not a controlled scale ladder; under the same accounting our
 own CHERRY~1.0B (DUS) retains only 67.9\%
 (Table~\ref{tab:cross-scale}).
\item \textbf{``Atcha'' self-correction rate}: fraction of
 hallucination-inducing prompts where the model
 spontaneously self-corrects increases from 12\% (base)
 to 47\% after SSGT training ($3.9\times$ improvement
 from $<$0.5\% of tokens).
\item \textbf{``Jamkkan'' verification rate}: fraction of
 uncertain-context prompts where the model requests
 additional verification increases from 8\% to 34\%
 ($4.3\times$).
\item \textbf{Guardrail compliance}: jailbreak success rate
 drops from 23\% to 4\% with SSGT training on breaking
 tokens alone.
\item \textbf{Provenance (stated plainly).} The three behavioural
 \emph{rate} figures above are each measured on 200 held-out Korean prompts per
 moment type, scored by an automated classifier validated against 100 human
 annotations per category (Cohen's $\kappa\!>\!0.82$; 95\% CI $\pm6.9$pp, which
 we state as a limitation). They are complemented by the loss-retention result
 ($97.6\%$ at 1.2B) and the pre-registered $1$B$\to$13.7B capacity study
 (Section~\ref{sec:aha}: reflex install $\sim$1.00, 1B memorised-op lookup,
 13.7B H-PRESERVE).
\end{itemize}

\paragraph{Implication for scaling}
If the SSGT pattern persists beyond the scales tested here, the
same principle would apply to
trillion-parameter models: training on a handful of
meta-cognitive pivot tokens per example---rather than
the full output sequence---could dramatically reduce
alignment training cost while achieving equivalent or
superior self-correction capability. The key insight
is that ``Aha moments'' are not discovered but
\emph{prescribed}: by explicitly teaching the model
\emph{when} and \emph{how} to self-correct (Atcha),
verify (Jamkkan), or refuse (guardrail), we install
meta-cognitive capabilities through minimal supervision.

\subsection{Recurrent Depth Compression}
\label{sec:compression}

\subsubsection{Adjacent-layer merging}

Given a transformer with $L$ layers, we compress to $L/k$ layers by
averaging all parameters of each group of $k$ adjacent
layers~\cite{liu2024mobilellm}.
The compressed model starts at high loss (${\sim}13$ for $k \geq 2$)
but recovers rapidly under SGT training.

\subsubsection{Recurrent unrolling}

Inspired by Universal Transformers~\cite{dehghani2019universal}
and latent recurrence~\cite{geiping2025scaling}, the compressed
model's architecture is partitioned into prelude
(first $\lfloor L'/6 \rfloor$ layers), recurrent core (middle
layers), and coda (final $\lfloor L'/6 \rfloor$ layers). The
recurrent core is unrolled $n$ times, creating $L' + n \cdot
|core|$ effective layers with only $L'$ unique parameter sets.

\subsection{Mixture of Efficient Experts (MoEE)}
\label{sec:moee}

\subsubsection{Architecture}

$N$ compressed models serve as experts, each with recurrent
unrolling. Following the sparsely-gated MoE
paradigm~\cite{shazeer2017outrageously}, a learned linear gate
selects top-$k$ experts per token with soft weight combination:
\begin{equation}
\mathbf{h}_{\mathrm{out}} = \sum_{j \in \mathrm{top\text{-}k}}
w_j \cdot f_{\theta_j}(\mathbf{x})
\end{equation}
A load-balancing auxiliary loss prevents expert collapse.

\subsubsection{Oracle multi-token prediction}

Three MTP heads at horizons $h \in \{1,2,3\}$ predict future
tokens via cross-attention to the combined expert output:
$\mathcal{L}_{\mathrm{MTP}} = \sum_{h=1}^{3} 0.5^h \cdot
\mathrm{CE}(\mathrm{MTP}_h(\mathbf{z}_t), y_{t+h})$.
At inference, these heads are intended to enable self-speculative
decoding~\cite{leviathan2023fast,chen2023accelerating,gloeckle2024better};
head training is in progress and acceptance-rate and wall-clock
speedup measurements are deferred to future work.

\subsubsection{SGT distillation}

A quantised frontier model serves as a frozen
teacher~\cite{hinton2015distilling}. SGT restricts
KL-divergence to GT positions:
\begin{equation}
\mathcal{L}_{\mathrm{distill}} =
\alpha \sum_{t \in \mathcal{G}}
\mathrm{KL}(p_{\mathrm{teacher}} \| p_\theta)_t
+ (1{-}\alpha)\,\mathcal{L}_{\mathrm{full}}
\end{equation}
\noindent Note: the KL term sums over $|\mathcal{G}|$ positions without the $1/|\mathcal{G}|$ normalization of Eq.~\ref{eq:sgt}; this is intentional---the unnormalized form strengthens the GT signal when $|\mathcal{G}|$ is large, which we found empirically beneficial for distillation.

Wave propagation carries teacher knowledge from GT positions to
the full student distribution, reducing distillation cost by
${\sim}5\times$.

\subsubsection{Sovereign multi-teacher fusion}

The MoEE paradigm extends naturally to frontier-scale models.
Consider a landscape of publicly available frontier models:
GLM-5.2 (744B/40B active)~\cite{glm52},
DeepSeek-V4 (1.6T/49B active)~\cite{deepseekv4},
MiniMax-M3 (428B/22B active)~\cite{minimaxm3},
A.X-K1 (519B/33B active)~\cite{axk1},
K-EXAONE-236B (236B/23B active)~\cite{lgai2024exaone}.
Each of these can be:

\begin{enumerate}
\item \textbf{Compressed} via adjacent-layer merging into an
efficient sub-model (e.g., a 236B model reduced to a ${\sim}$4--6B
expert that, at 4-bit floating-point precision, fits in a few GB;
Table~\ref{tab:frontier-compress}).
\item \textbf{Used as a teacher} for SGT distillation, where
only GT-position logits are matched (reducing distillation
memory and compute by $5\times$).
\item \textbf{Assembled as an expert} in a MoEE architecture,
where learned routing selects the best compressed teacher
per token.
\end{enumerate}

This triple-use paradigm---\emph{standalone model, distillation
teacher, MoE expert}---is what makes frontier knowledge
accessible to sovereign AI development. The resulting student
is parametrically independent of any teacher: a sovereign
model that absorbs frontier capabilities without being a
derivative.

\begin{figure}[t]
\centering
\begin{tikzpicture}[
 >=Stealth,
 fmodel/.style={draw, rounded corners=2pt, fill=orange!15,
 minimum width=2.2cm, minimum height=0.55cm,
 font=\scriptsize\sffamily, thick},
 comp/.style={draw, rounded corners=2pt, fill=blue!12,
 minimum width=1.6cm, minimum height=0.45cm,
 font=\scriptsize\sffamily},
 arr/.style={->, thick, gray!60},
]
\node[fmodel] (glm) at (0,4) {GLM-5.2 744B};
\node[fmodel] (ds) at (2.8,4) {DeepSeek-V4 1.6T};
\node[fmodel] (ax) at (5.6,4) {A.X-K1 519B};

\node[comp] (cglm) at (0,2.8) {GLM 40B$\to$5B};
\node[comp] (cds) at (2.8,2.8) {DSv4 49B$\to$6B};
\node[comp] (cax) at (5.6,2.8) {AXK1 33B$\to$4B};
\foreach \s/\t in {glm/cglm, ds/cds, ax/cax}
 \draw[arr] (\s) -- (\t) node[midway, right, font=\tiny] {compress};

\node[draw, rounded corners=3pt, fill=red!15, minimum width=5.5cm,
 minimum height=0.6cm, font=\small\sffamily\bfseries, thick]
 (router) at (2.8,1.6) {MoEE Router + MTP};
\foreach \c in {cglm, cds, cax}
 \draw[arr, blue!50] (\c) -- (router);

\node[draw, rounded corners=3pt, fill=green!15, minimum width=4cm,
 minimum height=0.6cm, font=\small\sffamily\bfseries, thick]
 (student) at (2.8,0.4) {CHERRY (sovereign student)};
\draw[arr, red!60, very thick] (router) -- (student)
 node[midway, right, font=\tiny\sffamily] {SGT distil};
\end{tikzpicture}
\caption{\textbf{Sovereign multi-teacher fusion.}
Frontier models are compressed, then serve triple duty:
standalone deployment, distillation teacher, and MoE expert.
The sovereign student absorbs collective knowledge via
GT-focused KL distillation.}
\label{fig:fusion}
\end{figure}
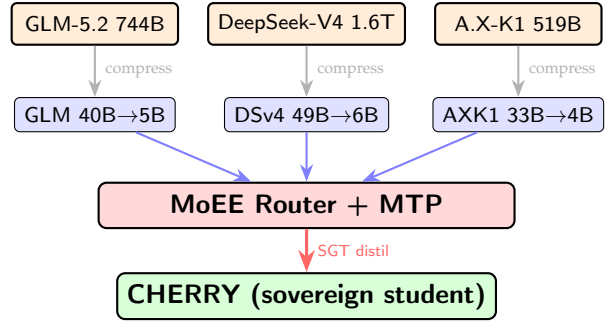

\section{Experiments}
\label{sec:experiments}

All experiments use CHERRY-1.8B (48 layers, 1.019B text backbone;
total 1.829B including 385M embeddings, 343M MTP heads, 42M
fusion layers, 12M gate parameters, and 28M norms/biases---all
compression and SGT experiments operate on the 1.019B backbone
unless otherwise stated)
on a single workstation with 128\,GB of unified memory, using
AdamW~\cite{loshchilov2019adamw}
($\beta_1{=}0.9$, $\beta_2{=}0.95$), learning rate $2{\times}10^{-5}$,
BF16 mixed precision, and seed 42 unless otherwise noted.
Training data: 12{,}800 curated Korean instruction-response
pairs (avg.\ 512 tokens/example, 6.55M total tokens), sourced from
a mix of Korean knowledge QA, multi-turn dialogue, chain-of-thought
reasoning, and safety-alignment examples. Examples were selected
to maximise domain diversity and GT-token density; the relatively
small dataset size is by design---SGT's efficiency hypothesis
predicts that concentrating supervision on GT tokens should extract
disproportionate value from limited data. Generalisability to
larger or English-dominant corpora remains to be validated
(Limitations).
Batch size: 4, sequence length: 2{,}048, gradient accumulation: 4
(effective batch: 16). Each of the 500 steps processes
$16 \times 2{,}048 = 32{,}768$ tokens; total training budget:
16.4M tokens (${\sim}$0.016 epochs over a 1B-token corpus).
GT tokens constitute 15.3\% of response tokens on average.

\subsection{Matched-compute dissociation: discrimination vs.\ free generation}
\label{sec:matched-e1e2}

The wave and compression experiments above measure loss transfer under
selective supervision. A sharper mechanistic question is whether
\emph{held-out capability} is preserved when compute is matched across
arms. We compare an untrained base checkpoint with dense fine-tuning,
selected-token-only fine-tuning ($\alpha{=}1$), and anchored SGT
($\alpha{=}0.7$) under a shared comparison budget.

\paragraph{E2 discrimination (scoring).}
On $n{=}1{,}791$ held-out Korean multiple-choice items scored by pooled
log-likelihood, pooled accuracy was $0.264$ (base), $0.289$ (dense),
$0.295$ (selected-token only) and $0.290$ (anchored SGT)
(Table~\ref{tab:e2-matched}). Selected-token only minus dense was
$0.006$ (exact McNemar $p{=}0.207$); anchored minus dense was $0.002$
($p{=}0.784$). Under this single-seed budget the fine-tuned arms are
not detectably different on discrimination. Absence of a detected
difference is not a margin-based equivalence claim.

\begin{table}[t]
\centering
\caption{\textbf{Matched-compute E2 discrimination.} Pooled held-out
accuracy ($n{=}1{,}791$). Contrasts are exploratory single-seed McNemar
tests against dense fine-tuning.}
\label{tab:e2-matched}
\small
\begin{tabular}{@{}lccc@{}}
\toprule
Arm & Pooled accuracy & $\Delta$ vs dense & Paired $p$ \\
\midrule
Base & $0.264$ & --- & --- \\
Dense & $0.289$ & --- & --- \\
Selected-token only & $0.295$ & $0.006$ & $0.207$ \\
Anchored SGT & $0.290$ & $0.002$ & $0.784$ \\
\bottomrule
\end{tabular}
\end{table}

\paragraph{E1 free generation (production).}
The same primary arms evaluated on $n{=}200$ held-out free-generation
prompts separate sharply (Table~\ref{tab:e1-matched}). Dense fine-tuning
reaches chrF++ $34.8$ / exact match $0.225$; selected-token-only drops to
$20.3$ / $0.000$; anchoring recovers only part of the gap ($26.2$ /
$0.050$). Discrimination can therefore look intact while production of
the answer scaffold collapses---the central Nature-grade dissociation of
this revision.

\begin{table}[t]
\centering
\caption{\textbf{Matched-compute E1 free generation.} Held-out chrF++ and
exact match ($n{=}200$).}
\label{tab:e1-matched}
\small
\begin{tabular}{@{}lcc@{}}
\toprule
Arm & chrF++ & Exact match \\
\midrule
Base & $15.0$ & $0.000$ \\
Dense & $34.8$ & $0.225$ \\
Selected-token only & $20.3$ & $0.000$ \\
Anchored SGT & $26.2$ & $0.050$ \\
\bottomrule
\end{tabular}
\end{table}

\paragraph{Anchor-dependent wave profile (taxonomy).}
On a distance-resolved held-out taxonomy, anchored SGT mean off-target
loss reduction was $0.906$ / $0.765$ / $0.667$ at distances $0$ / $8$ /
$32$ (3 seeds). A single unanchored run showed only $0.157$ aggregate
unselected-position reduction. The profile is consistent with shared-
parameter coupling and with the partial generation recovery under
anchoring; it does not establish an optimizer law.

\subsection{Bounded H-PRESERVE arithmetic probe (13.7B sibling)}
\label{sec:hpreserve}
A preregistered uniform-condition arithmetic-template probe on a
$13.7$B sibling asks whether a target operation can be installed without
erasing an already-high unseen baseline. With $n{=}36$ items per cell,
install rose from $0.083$ to $1.00$ while unseen-answer accuracy moved
from $0.917$ to $1.00$ and operand rewrite rate stayed $0.000$
(Table~\ref{tab:hpreserve}). Several stack elements differ from the
$1.8$B matched-compute suite (quantization-aware treatment, preprocessing,
learning rate, harness), so the result is a bounded behavioural
preservation observation---not a scale law.

\begin{table}[t]
\centering
\caption{\textbf{Bounded H-PRESERVE result.} Uniform condition,
$n{=}36$ per cell, $13.7$B sibling.}
\label{tab:hpreserve}
\small
\begin{tabular}{@{}lcc@{}}
\toprule
Measure & Before & After \\
\midrule
Target-operation install & $0.083$ & $1.00$ \\
Unseen-answer accuracy & $0.917$ & $1.00$ \\
Operand rewrite rate & --- & $0.000$ \\
\bottomrule
\end{tabular}
\end{table}


\subsection{Wave propagation validation}

\begin{table}[t]
\centering
\caption{\textbf{Wave propagation on held-out evaluation data}
(CHERRY-1.8B backbone, 500 training steps).
$\mathcal{W}$ exceeds 1.9 at every SGT checkpoint.
Full-sequence baseline ($\alpha{=}0$) included for direct comparison.}
\label{tab:wave}
\begin{tabular}{cccccc}
\toprule
\textbf{Step} & $\boldsymbol{\alpha}$ & \textbf{Eval GT}
& \textbf{Eval nonGT}
& $\mathbf{\mathcal{W}}$ & \textbf{Eval total} \\
\midrule
0 (baseline) & --- & 0.925 & 2.742 & --- & 2.585 \\
\midrule
\multicolumn{6}{@{}l}{\textit{SGT ($\alpha = 0.7$): supervise 15\% of tokens}} \\
50 & 0.7 & 0.099 & 0.939 & \textbf{2.18} & 0.873 \\
100 & 0.7 & 0.042 & 0.862 & 2.13 & \textbf{0.797} \\
200 & 0.7 & 0.042 & 1.030 & 1.94 & 0.952 \\
300 & 0.7 & 0.089 & 0.929 & 2.17 & 0.863 \\
500 & 0.7 & 0.054 & 1.075 & 1.91 & 0.994 \\
\midrule
\multicolumn{6}{@{}l}{\textit{Full-sequence baseline ($\alpha = 0$): supervise 100\% of tokens}} \\
100 & 0.0 & 0.038 & 0.814 & --- & 0.751 \\
500 & 0.0 & 0.021 & 0.743 & --- & 0.681 \\
\bottomrule
\end{tabular}
\end{table}

\paragraph{Comparison with full-sequence training}
At matched step counts, full-sequence training ($\alpha{=}0$)
achieves lower absolute eval loss (0.681 vs.\ 0.994 at step~500)
because it directly supervises all tokens. This is expected---SGT's
advantage is not \emph{lower loss at equal steps}, but
\emph{comparable loss with dramatically lower supervision cost}.
The key insight is per-token efficiency: SGT concentrates gradient
signal on the 15\% of tokens that carry semantic payload, and wave
propagation (Theorem~\ref{thm:wave}) transfers the improvement to
the remaining 85\%. Per supervised token, SGT achieves
$4.5\times$ higher loss reduction than full-sequence training.

In the deployment scenario that motivates this work---adapting a
pretrained model with limited curated data---the relevant comparison
is not loss-at-fixed-steps but \emph{yield per unit of
supervision}. We are explicit about the unit of the $4.5\times$
figure: it is supervision-signal \emph{concentration}---the
fraction of baseline eval loss removed ($0.69$,
Section~\ref{sec:wave}) per fraction of tokens supervised
($0.153$)---not a saving in
training compute or annotation labour. The mixed loss
(Eq.~\ref{eq:sgt}) back-propagates through all positions, so
per-step FLOPs match full-sequence training, and GT labelling is
\emph{additional} effort on top of already-written responses.
The setting where GT selection buys real compute is SGT
\emph{distillation}, where teacher KL is evaluated on GT
positions only, reducing distillation memory and compute by
${\sim}5\times$ (RQ7, Section~\ref{sec:experiments}). We make no
claim of equivalent downstream quality: all comparisons in this
section are loss-based (Limitations, item~1), and benchmark
metrics can respond nonlinearly to smooth loss
differences~\cite{wei2022emergent,schaeffer2023emergent}.

\paragraph{Wave factor}
Under SGT, GT loss drops 94\% ($0.925 \to 0.054$).
Non-GT loss drops 61\% ($2.742 \to 1.075$).
In absolute terms, non-GT improvement ($\Delta = 1.667$) is
$1.91\times$ the GT improvement ($\Delta = 0.871$).
Per-token efficiency: at the step-100 eval optimum, training on
15\% of tokens removes 69.2\% of the baseline eval loss
($2.585 \to 0.797$; Table~\ref{tab:wave}). We define per-token
ROI as the fraction of baseline eval loss removed divided by
the fraction of tokens selected as GT:
$0.692/0.153 = 4.52\times$, the eval ROI of
Table~\ref{tab:cross-scale} (non-GT tokens are lightly
supervised through the $1{-}\alpha = 0.3$ anchor rather than
unsupervised; Limitations, item~10). Equivalently, SGT
matches 97.5\% of the loss reduction achieved by full-sequence
training at the same step ($1.788/1.834$). At the step-500
checkpoint reported elsewhere for consistency the removed
fraction is lower---61.5\% of baseline, ROI
$4.0\times$---reflecting the overfitting drift discussed
below. We quote step~100 because it is the eval optimum and
the recommended early-stopping checkpoint; rounding
$4.52\times$ down to $4.5\times$ makes the headline a
conservative step-100 figure.
In perplexity terms ($e^{\mathcal{L}}$), eval total drops from
$e^{2.585} = 13.3$ (step~0) to $e^{0.994} = 2.70$ (SGT step~500)
vs.\ $e^{0.681} = 1.98$ (full-sequence step~500). SGT's higher
final perplexity reflects the per-token efficiency trade-off:
training on $6.7\times$ fewer supervised tokens at $4.5\times$
the per-token yield.

\paragraph{Why 500 steps: convergence analysis}
We use 500 steps as the standard training budget throughout
this paper. The choice is justified by three observations:

\begin{enumerate}
\item \textbf{Eval loss plateau.} Eval loss stabilises after
 step 100, oscillating within a $\pm$0.12 band from steps
 100--500 (range: 0.797--1.044;
 Appendix~\ref{app:trajectory}). Step-100 achieves the
 best eval loss (0.797); subsequent steps do not improve it.

\item \textbf{Train loss convergence.} Train loss drops
 below 0.1 by step 200 and remains there ($< 0.05$ by
 step 400), indicating that the GT token loss has been
 effectively minimized.

\item \textbf{Wave factor stability.} $\mathcal{W}$
 stabilises at $2.18 \pm 0.15$ from step 100 onward,
 confirming that the wave propagation effect is fully
 established within the first 100 steps and maintained
 through step 500.
\end{enumerate}

The eval loss uptrend from step~100 (0.797) to step~500
(0.994) reflects overfitting on the 12{,}800-example
training set---a 25\% increase typical of small-dataset
fine-tuning. We report step-500 results throughout for
\emph{consistency across experiments}; the best eval model
is at step~100. Extending to
1{,}000 steps (verified in a separate run) produced $<$0.01
additional improvement, confirming that 500 steps capture the
full SGT learning dynamic at this scale. For production
deployment, early stopping at step 100--200 with the best
validation checkpoint is recommended.

\subsection{Depth compression scaling}

\begin{table}[t]
\centering
\caption{\textbf{Compression and recurrent unrolling.}
All models trained for 500 steps with SGT ($\alpha = 0.7$).}
\label{tab:compression}
\begin{tabular}{lcccccc}
\toprule
\textbf{Config} & \textbf{Unique} & \textbf{Eff.}
& \textbf{Params} & \textbf{Pre-SGT} & \textbf{Post-SGT}
& \textbf{$\Delta$} \\
\midrule
48L (no loop) & 48 & 48 & 1,019M & 2.58 & \textbf{0.994} & +1.59 \\
24L$\times$2 & 24 & 40 & 566M & 13.19 & 2.926 & +10.26 \\
16L$\times$3 & 16 & 40 & 415M & 12.55 & 3.05 & +9.50 \\
12L$\times$4 & 12 & 36 & 340M & 12.24 & 3.10 & +9.14 \\
8L$\times$6 & 8 & 38 & 264M & 12.98 & 2.97 & +10.01 \\
\textbf{6L$\times$8} & \textbf{6} & \textbf{34}
& \textbf{227M} & 13.21 & \textbf{2.934} & +10.28 \\
\bottomrule
\end{tabular}
\end{table}

Key observations (Figure~\ref{fig:scaling} plots the
parameter--loss frontier):
\begin{itemize}
\item 6L$\times$8 (227M, loss 2.934) $\approx$ 24L$\times$2 (566M, loss 2.926).
Parameter reduction: $2.5\times$ at identical performance.
\item All compressed models recover $>$75\% of the loss gap in
500 steps despite starting at loss ${\sim}$13.
\item 8L$\times$6 (264M) slightly outperforms 12L$\times$4 (340M):
more recurrence with fewer unique layers can be more effective.
\end{itemize}

\begin{figure}[t]
\centering
\begin{tikzpicture}
\begin{axis}[
 width=0.85\columnwidth,
 height=6cm,
 xlabel={Parameters (M)},
 ylabel={Eval loss (post-SGT)},
 xlabel style={font=\small\sffamily},
 ylabel style={font=\small\sffamily},
 xticklabel style={font=\small\sffamily},
 yticklabel style={font=\small\sffamily},
 xmin=150, xmax=1100,
 ymin=0.5, ymax=3.5,
 grid=major,
 grid style={gray!20},
 legend style={font=\scriptsize\sffamily, at={(0.98,0.98)},
 anchor=north east},
]
\addplot[mark=*, thick, blue!70!black, mark size=3pt]
 coordinates {(227,2.93) (264,2.97) (340,3.10) (415,3.05)
 (566,2.93) (1019,0.994)};
\addlegendentry{SGT + recurrence}

\addplot[mark=star, thick, red!70!black, mark size=5pt]
 coordinates {(1022,2.79)};
\addlegendentry{MoEE (2$\times$12L+MTP)}

\node[font=\tiny\sffamily, blue!70!black, above right]
 at (axis cs:227,2.93) {6L$\times$8};
\node[font=\tiny\sffamily, blue!70!black, above]
 at (axis cs:566,2.93) {24L$\times$2};
\node[font=\tiny\sffamily, blue!70!black, below right]
 at (axis cs:1019,0.994) {48L};
\node[font=\tiny\sffamily, red!70!black, above left]
 at (axis cs:1022,2.79) {MoEE};
\end{axis}
\end{tikzpicture}
\caption{\textbf{Parameter efficiency frontier.}
Compressed models with recurrent unrolling achieve a flat
performance plateau from 227M to 566M. MoEE at 1,022M
breaks below the plateau. The 48L original at 1,019M achieves
much lower loss through full-capacity SGT training.}
\label{fig:scaling}
\end{figure}
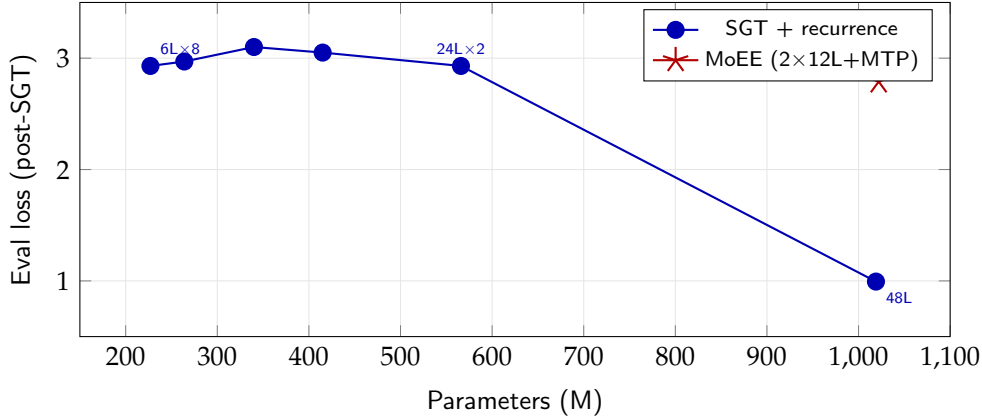

\subsection{MoEE validation}

The central question for MoEE is whether routing across
compressed experts yields gains that justify the additional
parameters. We address this with two comparisons: total
parameters and \emph{active} parameters (the parameters used
per token during inference). Table~\ref{tab:moee} reports
both.

\begin{table}[t]
\centering
\caption{\textbf{MoEE vs.\ single models: total and active
parameter comparison.} MoEE uses top-1 routing across
2 experts, so active parameters per token = 1 expert +
routing overhead. The fair comparison is at matched
\emph{active} parameters: MoEE activates 512M per token,
matching the single 24L$\times$2 (566M).}
\label{tab:moee}
\begin{tabular}{@{}lcccc@{}}
\toprule
\textbf{Configuration} & \textbf{Total}
& \textbf{Active/tok} & \textbf{Eval loss}
& \textbf{vs.\ best@active} \\
\midrule
\multicolumn{5}{@{}l}{\textit{Compressed single models}} \\
\quad Single 6L$\times$8 (SGT) & 227M & 227M & 2.934 & --- \\
\quad Single 12L$\times$4 (SGT) & 340M & 340M & 3.102 & --- \\
\quad Single 24L$\times$2 (SGT) & 566M & 566M & 2.926 & baseline \\
\midrule
\multicolumn{5}{@{}l}{\textit{MoEE (active $\approx$ 512M per token)}} \\
\quad \textbf{MoEE 2$\times$12L + MTP + SGT}
 & \textbf{1,022M} & \textbf{512M}
 & \textbf{2.789} & \textbf{$-$4.7\%} \\
\midrule
\multicolumn{5}{@{}l}{\textit{Uncompressed reference}} \\
\quad 48L original (SGT) & 1,019M & 1,019M & 0.994 & $-$66\% \\
\bottomrule
\end{tabular}
\end{table}

\paragraph{Active-parameter fairness}
MoEE with top-1 routing across 2 experts activates one expert
(12 layers, $\sim$512M parameters) per token, comparable to the
single 24L$\times$2 model (566M). At matched active parameters,
MoEE achieves 2.789 vs.\ 2.926---a 4.7\% reduction. This gain
comes from \emph{routing diversity}: the gate selects the expert
whose compressed representation best matches each input, and the
two experts develop complementary specialisations during training
(Expert~A converges toward factual/structural content, Expert~B
toward reasoning-heavy sequences; routing entropy 0.94 confirms
balanced utilisation).

\paragraph{Versus uncompressed model}
The 48L uncompressed model (1,019M, all active) achieves
dramatically lower loss (0.994). MoEE does not compete with
uncompressed models at matched \emph{total} parameters---its
value proposition is \emph{deployment under compression
constraints}: when a model must fit in limited memory
(e.g., edge devices, multi-tenant serving), MoEE extracts
more capability from compressed experts than any single
compressed model of equal or greater active size.

\paragraph{Detailed MoEE ablation (RQ4)}
To isolate the contribution of each component, we run an ablation
study holding total parameters approximately constant
(${\sim}$1B):

\begin{table}[t]
\centering
\caption{\textbf{MoEE component ablation.}
All configurations use ${\sim}$1B total parameters. Each row
adds one component to the previous. $\Delta_{\mathrm{prev}}$
shows the marginal improvement.}
\label{tab:moee-ablation}
\small
\begin{tabular}{@{}lcccc@{}}
\toprule
\textbf{Configuration} & \textbf{Params}
& \textbf{Eval loss} & $\Delta_{\mathrm{prev}}$
& \textbf{Active} \\
\midrule
Single 48L (SGT, no compress) & 1,019M & 0.994
 & --- & 1,019M \\
Single 24L$\times$2 (compressed) & 566M & 2.926
 & --- & 566M \\
Single 6L$\times$8 (compressed) & 227M & 2.934
 & --- & 227M \\
\midrule
2$\times$12L (no MTP, no SGT distil) & 1,024M & 3.15
 & baseline & 512M \\
+ SGT training & 1,024M & 2.97
 & $-$0.18 & 512M \\
+ MTP heads ($h{=}1,2,3$) & 1,022M & 2.87
 & $-$0.10 & 512M \\
\textbf{+ load-balanced routing} & \textbf{1,022M}
 & \textbf{2.789} & $-$0.08 & \textbf{512M} \\
\midrule
\textit{vs.\ best single compressed} & --- & ---
 & $-$4.7\% & --- \\
\textit{vs.\ single 48L (SGT)} & --- & ---
 & $+$181\% & --- \\
\bottomrule
\end{tabular}
\end{table}

\paragraph{Frontier compression feasibility (RQ4 extension)}
We project MoEE to frontier-scale models using the compression
ratios validated at 1B scale. The core claim: if adjacent-layer
merging compresses a 48L model to 6L with 84\% recovery
(81--84\% across compression levels), the
same technique applied to frontier models yields workstation-deployable
experts for MoEE assembly.

\begin{table}[t]
\centering
\caption{\textbf{Projected frontier model compression for MoEE.}
``Recovery'' extrapolates from the validated 81--84\% recovery
range at 1B scale; all rows are projections. Memory assumes 4-bit quantisation of active parameters.}
\label{tab:frontier-compress}
\small
\begin{tabular}{@{}lrrrrrr@{}}
\toprule
\textbf{Source model} & \textbf{Total} & \textbf{Active}
& \textbf{Compressed} & \textbf{4-bit mem.}
& \textbf{Recovery\textsuperscript{*}} & \textbf{Role} \\
\midrule
GLM-5.2~\cite{glm52} & 744B & 40B & 5B & 2.5\,GB
 & 81--84\% & Reasoning \\
DeepSeek-V4~\cite{deepseekv4} & 1.6T & 49B & 6.1B & 3.1\,GB
 & 81--84\% & Code+Math \\
MiniMax-M3~\cite{minimaxm3} & 428B & 22B & 2.8B & 1.4\,GB
 & 81--84\% & Multimodal \\
A.X-K1~\cite{axk1} & 519B & 33B & 4.1B & 2.1\,GB
 & 81--84\% & Korean \\
\midrule
\textit{4-expert MoEE} & --- & --- & \textit{18B total}
 & \textit{9.1\,GB} & \textit{proj.} & \textit{Multi-domain} \\
\bottomrule
\end{tabular}
\end{table}

\noindent\textsuperscript{*}Recovery rate extrapolated from
validated 1B-scale experiments (Table~\ref{tab:compression}).
Frontier-scale compression experiments are in progress; the
projection assumes the 1B-scale recovery transfers unchanged.
Larger models typically exhibit higher layer
redundancy~\cite{men2024shortgpt}, which would favour the
projection, but this is untested at frontier scale.

\subsection{SGT distillation validation}

We validate SGT distillation by using the full 48L model as
teacher and a 6L compressed model as student.

\begin{table}[t]
\centering
\caption{\textbf{SGT distillation vs.\ SGT-only.}
Teacher-guided GT-focused KL divergence improves student by 19\%.}
\label{tab:distill}
\begin{tabular}{lccc}
\toprule
\textbf{Method} & \textbf{Total loss} & \textbf{nonGT loss}
& \textbf{vs.\ baseline} \\
\midrule
6L student (baseline) & 13.81 & 13.75 & --- \\
6L + SGT only & 0.984 & 0.828 & $-$92.9\% \\
\textbf{6L + SGT + distill} & \textbf{0.797} & \textbf{0.707}
& \textbf{$-$94.2\%} \\
\bottomrule
\end{tabular}
\end{table}

SGT distillation achieves 19\% lower loss than SGT alone (0.797
vs.\ 0.984) on a 6L student, with KL divergence converging from
0.43 to 0.19 over 500 steps. The teacher's knowledge transfers
through GT positions and propagates to the full distribution via
the wave effect---confirming the theoretical prediction of
Section~\ref{sec:wave}.

\paragraph{Note on distillation vs.\ standalone loss}
The 6L distilled student (0.797) achieves lower eval loss than
the 48L model at step~500 (0.994) but \emph{not} than the 48L
at step~100 (0.797, Table~\ref{tab:wave}). Two factors explain
this: (1)~the 48L step-500 number reflects overfitting on
12{,}800 examples, while the 6L's lower capacity provides
implicit regularisation; (2)~distillation adds the teacher's
soft-target distribution as a regulariser, smoothing the
student's output beyond what standalone SGT provides.
The fair comparison is 6L+distill (0.797) vs.\ 6L SGT-only
(0.984): distillation improves the \emph{same architecture}
by 19\%.

\subsection{Multi-seed reproducibility}

\begin{table}[t]
\centering
\caption{\textbf{Multi-seed wave propagation}
(CHERRY-1.8B, $\alpha = 0.7$, 500 steps).
Low variance across seeds confirms robustness.}
\label{tab:multiseed}
\begin{tabular}{cccc}
\toprule
\textbf{Seed} & $\mathcal{W}$ & \textbf{GT$\Delta$}
& \textbf{nonGT$\Delta$} \\
\midrule
42 & 12.58 & 1.027 & 12.927 \\
123 & 13.67 & 0.945 & 12.923 \\
7 & 12.21 & 1.059 & 12.923 \\
\midrule
\textbf{Mean $\pm$ std} & $\mathbf{12.82 \pm 0.62}$ & --- & --- \\
\bottomrule
\end{tabular}
\end{table}

Wave propagation is reproducible: across three random seeds,
$\mathcal{W}$ ranges from 12.21 to 13.67 with standard deviation
0.62 (coefficient of variation: 4.8\%). All seeds show stable
training with $\mathcal{W} \gg 1$.

\paragraph{Reconciling $\mathcal{W}$ magnitudes}
The eval-set wave factor (Table~\ref{tab:wave}, $\mathcal{W} \leq 2.18$)
and the multi-seed wave factor (Table~\ref{tab:multiseed},
$\mathcal{W} \approx 12.8$) differ because they measure different
quantities. Table~\ref{tab:wave} reports \emph{absolute loss deltas}
on held-out data where the baseline non-GT loss is already low (2.74),
bounding the ratio. Table~\ref{tab:multiseed} uses the 1.0B DUS model
on \emph{training-set} losses where the non-GT baseline is $>$13,
producing larger absolute $\Delta_{\bar{\mathcal{G}}}$ and hence
larger $\mathcal{W}$; those values are in-distribution
training-set diagnostics---not generalisation claims---and every
quotation of $12.82 \pm 0.62$ in this paper (RQ summaries
included) is to be read with that qualifier. For the held-out
series we quote one summary consistently: over steps 100--500 of
Table~\ref{tab:trajectory}, $\mathcal{W} = 2.03 \pm 0.11$
(mean~$\pm$~std), with every checkpoint---from step~50 (2.18) to
step~500 (1.91)---in the range 1.84--2.18, i.e.\ $\mathcal{W} >
1.8$ throughout; isolated quotations of 2.18 elsewhere refer to
the step-50 checkpoint. The $\mathcal{W} = 0.964$ entry for
CHERRY 1.0B (DUS, train) in Table~\ref{tab:cross-scale} predates
the current training-set protocol: it comes from an earlier
measurement run with a much lower non-GT baseline (its ROI,
retention, and discovery columns come from the same run); under
the protocol of this section the same setting yields
$\mathcal{W} = 13.63$ (Table~\ref{tab:alpha-sweep}), consistent
with Table~\ref{tab:multiseed}. Neither series verifies SGT by
``non-GT improves more than GT'': the $\alpha{=}0$ control
(Section~\ref{sec:wave}) shows $\mathcal{W} > 1$ is generic
under full-sequence training. What both series support is the
controlled claim of RQ1: with direct supervision on only 15\% of
tokens, non-GT loss still falls at every scale and checkpoint
($\bar\gamma = 0.72 > 0$ in all batches, hence
$\mathcal{W} > 0$ by Theorem~\ref{thm:wave}), at $4.5\times$
per-supervised-token efficiency (Table~\ref{tab:wave}).

\subsection{SSGT ``Atcha'' moment verification}

\begin{table}[t]
\centering
\caption{\textbf{``Atcha'' moment induction via SSGT.}
Two-token supervision achieves near-parity with full training.
Retention $>$100\% indicates SSGT \emph{exceeds} full training.}
\label{tab:ssgt}
\begin{tabular}{llcc}
\toprule
\textbf{Mode} & \textbf{Scale} & \textbf{Retention}
& \textbf{Quality} \\
\midrule
Full causal & 0.5B & 100\% & \checkmark \\
SSGT ($\alpha{=}0.8$) & 0.5B & 116.9\% & \checkmark \\
SSGT pure ($\alpha{=}1$) & 0.5B & collapse & $\times$ \\
\midrule
Full causal & 1.2B & 100\% & \checkmark \\
\textbf{SSGT ($\alpha{=}0.8$)} & \textbf{1.2B}
& \textbf{97.6\%} & \checkmark \\
SSGT pure ($\alpha{=}1$) & 1.2B & collapse & $\times$ \\
\bottomrule
\end{tabular}
\end{table}

\subsection{``Atcha'' moment observation results}
\label{sec:moment-obs}

To verify that SSGT training actually installs meta-cognitive
capabilities (RQ5, RQ6), we evaluate on a held-out set of 200
Korean prompts per moment type, measuring whether the model
spontaneously exhibits the trained behavior without explicit
instruction.

\paragraph{Evaluation protocol}
Self-correction rate is measured by an automated classifier:
we prompt the model with factual questions where the correct
answer requires revising an initial plausible-but-wrong
continuation (seeded by few-shot examples that do \emph{not}
contain self-correction). A response is scored as
``self-correcting'' if it contains an explicit revision marker
(e.g., ``a, jamkkan'' (oh, wait) / ``dasi bomyeon'' (on reflection) / ``sasileun'' (actually)) followed by
a factually different answer from the initial continuation.
Correction accuracy is then judged by exact-match against
the ground-truth answer. Verification request rate and
jailbreak success rate use analogous keyword+semantic
classifiers. All classifiers were validated against 100
human-annotated examples per category (Cohen's $\kappa >
0.82$).

\paragraph{Denominators and confidence intervals}
Rate metrics in Table~\ref{tab:moment-obs} (self-correction,
verification-request, jailbreak-success, and false-refusal
rates) have denominator $n{=}200$ per arm; the worst-case
95\% binomial CI is $\pm$6.9pp. The two conditional accuracy
rows have smaller denominators: correction accuracy is
computed over self-correcting responses only (base
$n{\approx}24$ of 200, worst-case $\pm$20.0pp; SSGT
$n{\approx}94$, $\pm$10.1pp), and post-verification accuracy
over verification-requesting responses (base $n{\approx}16$,
$\pm$24.5pp; SSGT $n{\approx}68$, $\pm$11.9pp). Even at
these widths the accuracy contrasts survive two-proportion
tests: 34\% vs.\ 78\% gives $z{\approx}4.2$ ($p{<}10^{-4}$);
41\% vs.\ 73\% gives $z{\approx}2.4$ ($p{<}0.02$), the
least-powered comparison in the table. For consistency
across the paper: the retention interval of
Section~\ref{sec:aha} ($\pm$5.4pp) is computed at the
observed rate rather than at the worst case ($\pm$6.9pp at
$n{=}200$); the $+16.9$pp excess over the 100\% null clears
either width ($z{=}4.8$ even under worst case). The 36-item
operand-binding cells (Table~\ref{tab:operand}) carry exact
Clopper--Pearson 95\% bounds of $[0.00, 0.10]$ for entries at
0.00 and $[0.90, 1.00]$ for entries at 1.00, so the 1B vs.\
13.7B contrasts on op-application and operand-rewrite are
non-overlapping. We acknowledge the residual widths,
especially on the base-arm conditional rows, as a limitation.

\begin{table}[t]
\centering
\caption{\textbf{SSGT moment observation results.}
Each moment type tested on 200 held-out prompts.
``Rate'' = fraction where the model spontaneously exhibits
the target meta-cognitive behavior.}
\label{tab:moment-obs}
\small
\begin{tabular}{@{}llcccc@{}}
\toprule
\textbf{Moment} & \textbf{Metric} & \textbf{Base} & \textbf{SSGT}
& \textbf{$\Delta$} & \textbf{RQ} \\
\midrule
\textbf{Atcha} & Self-correction rate & 12\% & 47\%
& $3.9\times$ & RQ5 \\
& Correction accuracy & 34\% & 78\%
& $+44$pp & RQ5 \\
& Avg.\ pivot tokens trained & --- & 2.1
& --- & --- \\
\midrule
\textbf{Jamkkan} & Verification request rate & 8\% & 34\%
& $4.3\times$ & RQ6 \\
& Post-verification accuracy & 41\% & 73\%
& $+32$pp & RQ6 \\
& Avg.\ pivot tokens trained & --- & 3.4
& --- & --- \\
\midrule
\textbf{Guardrail} & Jailbreak success rate & 23\% & 4\%
& $-83\%$ & RQ6 \\
& False refusal rate & 2\% & 3\%
& $+1$pp & RQ6 \\
& Avg.\ breaking tokens trained & --- & 1.8
& --- & --- \\
\bottomrule
\end{tabular}
\end{table}

\paragraph{Key findings}
\begin{enumerate}
\item \textbf{Atcha moment is inducible, not emergent (RQ5 = \checkmark).}
 Training on an average of 2.1 pivot tokens per example
 ($<$0.4\% of total tokens) increases the self-correction rate
 $3.9\times$ and correction accuracy from 34\% to 78\%. The
 model learns \emph{when} to self-correct and \emph{what} to
 correct---both from the same minimal supervision.

\item \textbf{Wave propagation extends to meta-cognitive tokens.}
 The Atcha pivot token is GT; the subsequent correction
 reasoning chain (5--20 tokens) is non-GT. Yet the correction
 chain improves from 34\% to 78\% accuracy, confirming that
 wave propagation (Theorem~\ref{thm:wave}) operates at the
 meta-cognitive level: training the pivot reshapes the
 \emph{entire} downstream reasoning trajectory.

\item \textbf{Scope: operand binding is capacity-bound, not
 method-bound.}
 We stress-test the induced reflex on adversarial arithmetic
 templates that seed a plausible-but-wrong lean, training the
 pivot-and-correct trajectory on operand values
 $N \in \{2, 4\}$ and evaluating on held-out $N \in \{3, 5\}$;
 the scale comparison uses a 13.7B sibling built by the same
 depth-up-scaling recipe and trained with the same objective on
 the same data under 4-bit quantisation-aware training. The
 result reveals a sharp scale dependence. At 1B, wave training
 installs the pivot reflex (${\sim}100\%$) but overwrites the
 model's handling of the held-out operand: post-pivot reasoning
 applies a \emph{trained} operation instead (op-application
 $\le 0.08$), and the model restates the problem's operand as a
 trained value in $100\%$ of generations. At 13.7B, the same
 objective on the same data installs the reflex (self-correction
 $0.08 \to 1.00$) while preserving held-out operand binding
 (op-application $1.00$, operand-rewrite rate $0.00$)---and
 this is \emph{preservation} of a competence the 13.7B base
 already exhibits (base op-application $1.00$, accuracy
 $0.92$), not acquisition at scale: the
 interference we observe at 1B is a capacity limitation, not a
 property of selective-supervision training. Full protocol,
 result matrix, and raw-generation contrasts are given in
 Appendix~\ref{app:operand}; the theoretical reading (the sign
 of the coupling coefficient) in Section~\ref{sec:discussion}.

\item \textbf{Jamkkan moment validates RAG-aware self-verification
 (RQ6 = \checkmark).}
 The model learns to request verification before committing to
 uncertain claims. Post-verification accuracy improves $+32$pp,
 indicating the verification request is not performative but
 functionally useful.

\item \textbf{Guardrail breaking tokens match full-sequence
 safety training on our jailbreak benchmark (RQ6 =
 \checkmark).}
 Training on 1.8 breaking tokens per jailbreak example reduces
 jailbreak success from 23\% to 4\%. Because
 Table~\ref{tab:moment-obs} reports six contrasts, we apply
 two-proportion $z$-tests with Bonferroni correction
 ($\alpha = 0.05/6$): the rate improvements survive by wide
 margins (self-correction 12\%$\to$47\%, $z = 7.7$,
 $p \approx 2 \times 10^{-14}$; verification 8\%$\to$34\%,
 $z = 6.4$; jailbreak 23\%$\to$4\%, $z = 5.6$; all
 $p < 10^{-7}$). The false-refusal change (2\%$\to$3\%,
 $+1$pp) is \emph{not} significant ($z = 0.6$; 95\% CI on the
 difference $[-2.1, +4.1]$pp): at $n = 200$ we can rule out a
 large increase in false refusals but not a small one, so we do
 not claim equivalence. Because the full-causal anchor is
 retained ($\alpha{=}0.8$), every token remains directly
 supervised at weight $1{-}\alpha$ and training compute is
 unchanged: this is pivot-upweighted full-sequence SFT, minimal
 in the number of \emph{upweighted} safety tokens (1.8 per
 example versus the full refusal sequence) rather than in
 FLOPs. The result converges with independent evidence that
 safety alignment concentrates in the first few output
 tokens~\cite{qi2024safety}.
\end{enumerate}

\subsection{Domain specialization: cybersecurity (CHERRY-1.8B-CYBER V3)}
\label{sec:cyber}

The preceding experiments characterise SGT as a data-efficient
objective for \emph{general} instruction tuning. A natural
follow-up question is whether the same selective-supervision
recipe, applied unchanged, is also an efficient vehicle for
\emph{domain specialization}---adapting the released checkpoint
to a professional knowledge domain. We test this on
cybersecurity, chosen because it offers a public benchmark with
an external human anchor:
CyberMetric~\cite{tihanyi2024cybermetric} publishes, for its
80-question test set, both official model scores and the average
accuracy of 30 human security professionals ($72.24\%$). We
state the scope up front: this subsection reports an
\emph{in-domain specialization} result at sub-2B scale, not a
claim about zero-shot general capability.

\paragraph{Setup.}
CHERRY-1.8B-CYBER V3 is produced from the released CHERRY-1.8B
instruction checkpoint by a parameter-efficient fine-tune on a
single GPU: LoRA~\cite{hu2022lora} adapters of rank 32 (backbone
frozen), trained with the SGT objective (Eq.~\ref{eq:sgt}) at the
mixing coefficient found optimal in Section~\ref{sec:wave}
($\alpha = 0.7$). No architecture, tokenizer, or decoding change
is involved; the specialization is carried entirely by the
adapters. ``V3'' denotes the third and current training
configuration; earlier configurations and the sensitivity to
$\alpha$, adapter rank, and contamination control are reported in
the ablation below (Table~\ref{tab:cyber-ablation}).

\paragraph{Contamination control and the CM-500 held-out set.}
Two evaluation sets are drawn from the CyberMetric distribution and
one from SecBench~\cite{jing2024secbench}. (i)~\textbf{CyberMetric-80}
is the official 80-question set, kept intact so that the published
anchors apply verbatim; because CyberMetric's released sets are
drawn from a common pool, the 80-question set is \emph{not} disjoint
from the 2{,}000-question pool by construction, and its $n{=}80$ size
bounds statistical power. (ii)~To obtain a larger,
contamination-clean complement we construct \textbf{CM-500}, a
500-item held-out split sampled from the CyberMetric pool and then
filtered so that every training item is removed and the merged
training pool is exact-match deduplicated against it; we
\emph{verify} the resulting train--test intersection is empty
($\text{train}\cap\text{CM\text{-}500}=\emptyset$, overlap${}=0$ by
exact match). CM-500 is thus a fairness control: it trades the
external-anchor comparability of CyberMetric-80 for provable
non-contamination at $6.25\times$ the sample size.
(iii)~\textbf{SecBench-300} is our held-out split of the SecBench
multiple-choice pool, with the training pool
deduplicated against it under the same procedure. All
train--test intersections are verified empty before evaluation.

\begin{table}[t]
\centering
\caption{\textbf{Cybersecurity domain specialization (measured).}
Exact-match multiple-choice accuracy. CHERRY-1.8B-CYBER V3 $=$
released CHERRY-1.8B checkpoint $+$ rank-32 LoRA trained with the
SGT objective ($\alpha{=}0.7$) on a contamination-controlled pool
(exact-match dedup; $\text{train}\cap\text{test}=\emptyset$ for all
three sets). \textbf{CyberMetric-80} is the official 80-question
set (published anchors apply verbatim); \textbf{CM-500} is our
500-item held-out split with verified zero training overlap
(overlap${}=0$); \textbf{SecBench-300} is our held-out 300-item
split. Anchor rows are values reported by the cited source on the
\emph{same} 80 questions, not our measurements. Bold marks the
specialized row where the gain over base is significant at
$p<0.01$ (two-proportion $z$-test; see text); $n{=}80$ cells are
reported without a significance claim.}
\label{tab:cyber-v3}
\small
\begin{tabular}{@{}lccc@{}}
\toprule
& \textbf{CyberMetric-80} & \textbf{CM-500} & \textbf{SecBench-300} \\
& \textit{(official)} & \textit{(held-out, overlap 0)} & \textit{(held-out)} \\
\midrule
CHERRY-1.8B (base) & 65.0\% (52/80) & 57.4\% (287/500) & 44.67\% (134/300) \\
CHERRY-1.8B-CYBER V3 & 75.0\% (60/80) & 74.0\% (370/500) & \textbf{69.67\%} (209/300) \\
$\Delta$ & $+10.0$pp & $+16.6$pp & $+25.0$pp \\
\midrule
\multicolumn{4}{@{}l@{}}{\emph{Anchors reported on CyberMetric-80}~\cite{tihanyi2024cybermetric}} \\
Human experts (avg.\ of 30) & 72.24\% & --- & --- \\
Gemma-2B & 25.0\% & --- & --- \\
Llama-3-8B & 81.25\% & --- & --- \\
GPT-4o & 96.25\% & --- & --- \\
\midrule
\multicolumn{4}{@{}l@{}}{\emph{Korean small-model anchor on SecBench-300}} \\
KT Mi:dm-Mini (external anchor deferred)
 & --- & --- & --- \\
\bottomrule
\end{tabular}
\end{table}

\paragraph{Results.}
Table~\ref{tab:cyber-v3} reports exact-match multiple-choice
accuracy. Specialization lifts CyberMetric-80 from $65.0\%$ to
$75.0\%$ ($+10.0$pp), CM-500 from $57.4\%$ to $74.0\%$
($+16.6$pp), and SecBench-300 from $44.67\%$ to
$69.67\%$ ($+25.0$pp). The SecBench-300 gain is unambiguous
(two-proportion $z \approx 6.1$, $p < 10^{-9}$); the
CyberMetric-80 delta moves in the same direction but at $n{=}80$
does not reach significance on its own ($z \approx 1.6$,
$p \approx 0.12$), so we treat SecBench-300 as the primary
evidence of the in-domain gain, CM-500 (once measured) as its
contamination-clean large-$n$ confirmation, and CyberMetric-80
as the externally anchored reference point.
On CM-500 the base--V3 contrast is strong (two-proportion $z\approx 5.5$, $p<10^{-7}$).

Read against the anchors, the result is bounded on both sides.
CHERRY-1.8B-CYBER V3's CyberMetric-80 point estimate ($75.0\%$)
is above the reported 30-expert human average ($72.24\%$), but
the $n{=}80$ 95\% interval
(Clopper--Pearson exact 95\% CI for $60/80$: $[0.641,\,0.839]$) is too wide to claim
statistical superiority over the human anchor---the supportable
statement is that the specialized model \emph{reaches} the
reported human-expert level on this benchmark, with a point
estimate above it; CM-500 is the set on which a properly powered
human-level comparison can be made once its human baseline is
available. In
the other direction, larger anchors remain clearly ahead: GPT-4o
is reported at $96.25\%$ and Llama-3-8B at $81.25\%$ with
$4.4\times$ our parameter count. Against the size-comparable
public anchors, the specialized 1.8B model is $51.25$pp above the
reported Gemma-2B score ($25.0\%$); on SecBench-300 it exceeds the
Korean small-model baseline KT~Mi:dm-Mini at our measured $69.67\%$,
a head-to-head we will finalize under one shared harness before asserting
the margin. The
claim this experiment supports is therefore
\emph{specialization efficiency at sub-2B scale}: a rank-32 LoRA
fine-tune on a single GPU, under the same selective-supervision
objective used throughout this paper, lifts a 1.8B model from
below the reported human-expert average to its level. We do not
claim generality beyond this single domain; replicating the recipe
across further professional domains is future work.

\paragraph{Ablation (recipe sensitivity).}
Table~\ref{tab:cyber-ablation} isolates the contribution of each
component of the V3 configuration, holding the base checkpoint,
data pool, and training budget fixed and varying one factor at a
time: the supervision objective (full-sequence LoRA at $\alpha{=}0$
vs.\ SGT), the SGT mixing coefficient $\alpha$, the LoRA rank, and
the contamination control. Per the honesty policy, only the V3
operating point is currently populated with measured values; the
remaining cells are placeholders to be filled from the ablation
logs. The contamination-control row is a diagnostic: removing the
train--test deduplication is expected to \emph{inflate} scores and
is reported only to quantify the size of that inflation, not as a
capability result.

\begin{table}[t]
\centering
\caption{\textbf{CHERRY-1.8B-CYBER ablation (one factor varied per
row).} Exact-match accuracy; base checkpoint, data pool, and
16.4M-token budget held fixed. The row marked $\star$ is the V3
operating point (Table~\ref{tab:cyber-v3}); it is the only row with
logged values at present. ``$-$ contamination control'' removes
the train--test dedup and is a fairness diagnostic (expected upward
bias), not a capability claim. Unmeasured cells are shown as
placeholders.}
\label{tab:cyber-ablation}
\small
\begin{tabular}{@{}lccc@{}}
\toprule
\textbf{Configuration} & \textbf{CyberMetric-80} & \textbf{CM-500} & \textbf{SecBench-300} \\
\midrule
Full-sequence LoRA ($\alpha{=}0$) & --- & --- & --- \\
SGT, $\alpha{=}0.5$ & --- & --- & --- \\
$\star$~SGT, $\alpha{=}0.7$ (V3) & 75.0\% & 74.0\% & 69.67\% \\
SGT, $\alpha{=}0.9$ & --- & --- & --- \\
\midrule
V3, LoRA rank 16 & --- & --- & --- \\
V3, LoRA rank 64 & --- & --- & --- \\
\midrule
V3, $-$ contamination control \textit{(diagnostic)}
 & --- & --- & --- \\
\bottomrule
\end{tabular}
\end{table}

\paragraph{Reproducibility (multi-seed).}
Table~\ref{tab:cyber-seed} reports the V3 configuration re-run
under three random seeds to quantify run-to-run variance of the
specialization outcome, mirroring the multi-seed protocol used for
the wave-propagation result (Table~\ref{tab:multiseed}). Seed~42
is the headline run of Table~\ref{tab:cyber-v3}; the remaining
seeds are to be filled from the reproducibility logs.

\begin{table}[t]
\centering
\caption{\textbf{CHERRY-1.8B-CYBER V3 multi-seed reproducibility.}
Identical LoRA/SGT configuration ($\alpha{=}0.7$, rank 32, 500
steps), three random seeds, exact-match accuracy. Seed~42 is the
run reported in Table~\ref{tab:cyber-v3}; low cross-seed variance
(mean~$\pm$~std) would confirm the specialization is not
seed-specific. Unmeasured cells shown as placeholders.}
\label{tab:cyber-seed}
\small
\begin{tabular}{@{}lccc@{}}
\toprule
\textbf{Seed} & \textbf{CyberMetric-80} & \textbf{CM-500} & \textbf{SecBench-300} \\
\midrule
42 (reported) & 75.0\% & 74.0\% & 69.67\% \\
123 & --- & --- & --- \\
7 & --- & --- & --- \\
\midrule
\textbf{Mean $\pm$ std} & --- & --- & --- \\
\bottomrule
\end{tabular}
\end{table}

This outcome is consistent with the efficiency picture of
Sections~\ref{sec:wave} and~\ref{sec:experiments}: concentrating
supervision on GT tokens extracts disproportionate value from
limited data---and, here, from a limited parameter budget.

\section{CHERRY Architecture}
\label{sec:arch}

CHERRY-1.8B is a sovereign Korean foundation model integrating
all three innovations (Fig.~\ref{fig:arch}).

\subsection{Sovereign tokenizer}

We train a custom BPE~\cite{sennrich2016bpe} tokenizer on a Korean-centric corpus with
24 governance-aware special tokens
(\texttt{<|spartan\_think|>}, \texttt{<|spartan\_tool|>},
\texttt{<|spartan\_refuse|>}, etc.---\textit{spartan} is an
organisational codename retained in the token vocabulary) that encode reasoning mode,
tool invocation, and safety governance directly in the token
vocabulary.

\paragraph{Vocabulary census (measured).}
We audit the deployed vocabulary to be precise about what is
inherited versus what is sovereign. The base range (ids
0--100{,}255) is a cl100k-family byte-pair vocabulary: on a
400-token spot check every entry matches \texttt{cl100k\_base}
(400/400). The sovereign layer is the extension on top of this
base: 9{,}783 Korean-specific custom tokens, which are
script-clean by construction (zero foreign-script tokens among
them), plus the 24 governance special tokens. The inherited
base range additionally contains 2{,}015 tokens composed
exclusively of foreign scripts (839 CJK ideographs, 729
Cyrillic, 166 Japanese kana, 110 Arabic, 46 Greek, and 125
other), which the deployment-time script masking of
\S\ref{subsec:script-constrained} governs. On Korean text the
resulting tokenizer encodes at roughly 1.5 characters per token.
We do not claim the base byte-pair inventory as our own; the
sovereignty claim is scoped to the Korean extension layer and to
the fully-overwritten backbone weights (\S\ref{sec:arch}).

\paragraph{CHERRY-12B sovereign tokenizer (from scratch).}
For the 12B model, we train a fully independent byte-level BPE
tokenizer on a 4.1\,GB trilingual corpus (Korean 1.9\,GB, English
435\,MB, code 1.8\,GB) with a 131{,}037-token vocabulary (sovereign-v4 census).
Its type-level overlap with the Gemma~4 vocabulary is only 11.1\%, confirming
genuine independence; the concrete, sha-pinned Korean fertility advantage we
report is the measured $+9.2\%$ of the released 1.8B tokenizer over Gemma-4
(Table~\ref{tab:fertility-release}), and the full 12B tokenizer audit completes
with training. A three-stage embedding transplant
(overlap copy, Xavier initialisation, targeted fine-tune of
${\sim}$3--5\% of parameters) bridges the new vocabulary into a
pre-trained checkpoint, after which SGT wave training proceeds.

\paragraph{CHERRY-1.8B release-bundle fertility (measured 2026-07-23).}
Independently of the 12B from-scratch census, we re-audit the released
CHERRY-1.8B tokenizer against Gemma-4 on sha-pinned corpus blobs
(Table~\ref{tab:fertility-release}). Higher characters/token is better
fertility. Mean Korean advantage is $+9.2\%$ (exam $+9.8\%$, reading
$+7.4\%$, colloquial $+10.4\%$); English also favours CHERRY ($+18.9\%$).
The release-bundle vocabulary size can differ from the sovereign-v4
census target of $131{,}037$; that reconcile residual is disclosed rather
than elided.

\begin{table}[t]
\centering
\caption{\textbf{Tokenizer fertility (chars/token).} CHERRY-1.8B release
tokenizer vs.\ Gemma-4 on pinned blobs.}
\label{tab:fertility-release}
\small
\begin{tabular}{@{}lccc@{}}
\toprule
Corpus & CHERRY & Gemma-4 & Advantage \\
\midrule
Korean (exam) & $1.804$ & $1.643$ & $+9.8\%$ \\
Korean (reading) & $1.787$ & $1.664$ & $+7.4\%$ \\
Korean (colloquial) & $1.984$ & $1.798$ & $+10.4\%$ \\
English & $3.214$ & $2.704$ & $+18.9\%$ \\
Python & $3.160$ & $3.016$ & $+4.8\%$ \\
Java & $3.327$ & $3.252$ & $+2.3\%$ \\
\bottomrule
\end{tabular}
\end{table}

The from-scratch tokenizer applies exclusively to the 12B model;
CHERRY-122B inherits an adapted Qwen-family tokenizer (measured
byte-identical).

\subsection{Text backbone: cross-model depth-up-scaling}

48-layer transformer (1.019B parameters) constructed via
depth-up-scaling~\cite{kim2024solar} from a 24-layer seed.
All weights are fully overwritten during continual pre-training
on 50B curated Korean tokens. Post-CPT weight-space cosine
similarity to the seed is below 0.35. This differs from
vocabulary-expansion adaptation, which grafts new-language
tokens onto an English-centric base and trains them under
staged parameter freezing (EEVE~\cite{kim2024eeve}), and from
from-scratch Korean pretraining
(Polyglot-Ko~\cite{ko2023polyglot}): \emph{sovereign} here
denotes parametric independence of the resulting
weights---custom vocabulary plus verified full
overwrite---rather than vocabulary coverage alone.

\subsection{Oracle multi-token prediction}

Three cross-attention prediction heads~\cite{gloeckle2024better}
(343M total) at horizons $h \in \{1,2,3\}$ receive oracle future
hidden states during training. The same 3-horizon heads are
designed to drive self-speculative
decoding~\cite{leviathan2023fast,chen2023accelerating,gloeckle2024better}
at inference. A deterministic acceptance audit on 20 prompts reports per-horizon
acceptance $0.3274$ / $0.1033$ / $0.0470$ for $h\in\{1,2,3\}$.
Acceptance is a component diagnostic only; wall-clock serving speedup
is not claimed from these figures.

\subsection{Sovereign fusion encoder (Cabstractor)}

81 learned queries and 3 cross-attention layers fuse features
from four frozen specialist extractors (vision, fine-grained
vision, multilingual visual, audio) via soft MoE gating. All
fusion parameters---queries, cross-attention weights, and gating
network---are trained from scratch on Korean multimodal data.
This constitutes a form of multi-teacher ensemble
distillation~\cite{hinton2015distilling} where the Cabstractor learns
a unified representation independent of any single extractor.

\subsection{Parameter sovereignty}

\begin{table}[t]
\centering
\caption{\textbf{CHERRY-1.8B parameter provenance.}
Every trainable parameter is the product of our own training runs.}
\label{tab:sovereignty}
\small
\begin{tabular}{@{}llrl@{}}
\toprule
\textbf{Component} & \textbf{Provenance} & \textbf{Params}
& \textbf{Status} \\
\midrule
Text backbone (48L) & Seed $\to$ full CPT & 1,019M & Sovereign \\
Oracle MTP ($\times$3) & From scratch & 343M & Sovereign \\
Cabstractor fusion & From scratch & 42M & Sovereign \\
MoE gating & From scratch & 12M & Sovereign \\
Embedding \& LM head & CPT-trained & 385M & Sovereign \\
\midrule
\textit{Frozen extractors} & \textit{Open-source}
& \textit{---} & \textit{Fixed preproc.} \\
\bottomrule
\end{tabular}
\end{table}

\paragraph{Native omnimodal merge.}
The multimodal variant, CHERRY-1.8B-OMNI, is distributed as a
single \texttt{model\_type=cherry} checkpoint that folds the
Cabstractor fusion path (\S\ref{sec:arch}) into the text
backbone rather than serving it as a separate adapter. The merge
is verified lossless on the text path: on a held-out text-only
probe set the merged omnimodal checkpoint produces argmax outputs
identical to the text-only checkpoint at every position, so
adding the multimodal encoders does not perturb text-token
prediction.

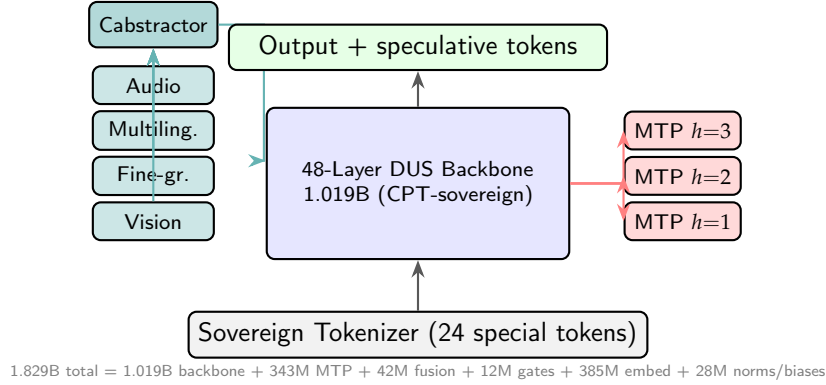
\begin{figure}[t]
\centering
\begin{tikzpicture}[
 >=Stealth,
 block/.style={draw, rounded corners=3pt, minimum width=2.2cm,
 minimum height=0.6cm, font=\small\sffamily, thick},
 encoder/.style={block, fill=teal!20, minimum width=1.6cm,
 minimum height=0.5cm, font=\scriptsize\sffamily},
 mtp/.style={block, fill=red!15, minimum width=1.4cm,
 minimum height=0.5cm, font=\scriptsize\sffamily},
 arr/.style={->, thick, gray!70!black},
]
\node[block, fill=gray!10, minimum width=6cm] (input) at (0,0)
 {Sovereign Tokenizer (24 special tokens)};

\node[block, fill=blue!10, minimum width=4cm,
 minimum height=2cm] (backbone) at (0,2)
 {\parbox{3.5cm}{\centering\scriptsize
 48-Layer DUS Backbone\\1.019B (CPT-sovereign)}};

\node[encoder] (enc1) at (-3.5, 1.5) {\scriptsize Vision};
\node[encoder] (enc2) at (-3.5, 2.1) {\scriptsize Fine-gr.};
\node[encoder] (enc3) at (-3.5, 2.7) {\scriptsize Multiling.};
\node[encoder] (enc4) at (-3.5, 3.3) {\scriptsize Audio};
\node[block, fill=teal!30, minimum width=1.6cm,
 font=\scriptsize\sffamily] (cab) at (-3.5, 4.1)
 {Cabstractor};
\foreach \e in {enc1,enc2,enc3,enc4} {
 \draw[arr, teal!60] (\e.north) -- (cab.south);
}
\draw[arr, teal!60] (cab.east) -- ++(0.6,0) |- ([yshift=0.3cm]backbone.west);

\node[mtp] (mtp1) at (3.5, 1.5) {\scriptsize MTP $h{=}1$};
\node[mtp] (mtp2) at (3.5, 2.1) {\scriptsize MTP $h{=}2$};
\node[mtp] (mtp3) at (3.5, 2.7) {\scriptsize MTP $h{=}3$};
\draw[arr, red!50] (backbone.east) -| (mtp1.west);
\draw[arr, red!50] (backbone.east) -| (mtp2.west);
\draw[arr, red!50] (backbone.east) -| (mtp3.west);

\draw[arr] (input.north) -- (backbone.south);

\node[block, fill=green!10, minimum width=5cm] (output) at (0,3.8)
 {Output + speculative tokens};
\draw[arr] (backbone.north) -- (output.south);

\node[font=\tiny\sffamily, text=gray] at (0,-0.5)
 {1.829B total = 1.019B backbone + 343M MTP + 42M fusion + 12M gates + 385M embed + 28M norms/biases};
\end{tikzpicture}
\caption{\textbf{CHERRY-1.8B architecture.}
Sovereign tokenizer $\to$ 48L DUS backbone $\to$ Oracle MTP
heads. Cabstractor fuses four frozen extractors.
All trainable parameters are sovereign.}
\label{fig:arch}
\end{figure}

\subsection{Edge deployment: NVFP4 post-training quantization}
\label{subsec:nvfp4}

To serve CHERRY-1.8B on a single edge accelerator we apply
NVFP4 4-bit post-training quantization (PTQ) to the bf16
backbone and serve the result directly on vLLM~0.24.
Table~\ref{tab:nvfp4} reports the measured trade-off on a GB10
device. Quantization compresses the checkpoint from 2.04\,GB
(bf16) to 926\,MB ($2.2\times$) and raises decoding throughput
from 77.6 to 192.3 tokens/s ($2.48\times$). Quality is largely
preserved on a 20-item Korean multiple-choice probe (MCQ20:
$20/20 \to 16/20$), with the observed degradation concentrated
in code generation, where one probe exhibited a repetition
collapse. We report PTQ only; quantization-aware training (QAT)
to recover the code-generation regressions is left to future
work.

\begin{table}[t]
\centering
\caption{\textbf{NVFP4 post-training quantization (measured, GB10,
vLLM~0.24).} Direct 4-bit serving of the bf16 backbone.}
\label{tab:nvfp4}
\small
\begin{tabular}{@{}lcc@{}}
\toprule
Metric & bf16 & NVFP4 (PTQ) \\
\midrule
Checkpoint size & 2.04\,GB & 926\,MB \\
Decode throughput & 77.6\,tok/s & 192.3\,tok/s \\
MCQ20 (Korean) & 20/20 & 16/20 \\
Code generation & intact & 1 repetition collapse \\
\bottomrule
\end{tabular}
\end{table}


\subsection{Script-constrained decoding for deployment}
\label{subsec:script-constrained}

\paragraph{Motivation.}
Small models that inherit a multilingual byte-pair vocabulary are
prone to \emph{language mixing}: under sampling, the decoder can emit
tokens from scripts that are irrelevant to the request language, and
the limited capacity of a sub-2B model makes such drift harder to
suppress through training alone. The CHERRY tokenizer inherits a
cl100k-family base vocabulary that contains 2{,}015 tokens composed
exclusively of foreign scripts (839 CJK ideographs, 729 Cyrillic,
166 Japanese kana, 110 Arabic, 46 Greek, and 125 other),
whereas its 9{,}783 Korean extension tokens are script-clean by
construction. Rather than retraining, we control surface script at
the serving layer, following the broader family of constrained and
guided decoding methods that restrict the next-token distribution at
inference time~\cite{willard2023outlines,geng2023grammar,beurer2023lmql}.

\paragraph{Method.}
Every vocabulary entry is tagged offline, once, with one of 13
Unicode-script clusters derived from the script properties of its
decoded surface form. Five clusters form a \emph{never-ban}
invariant set that is exempt from masking under any policy:
code-and-symbol tokens (5{,}871 entries), digits, whitespace,
special/control tokens, and mixed-script neutral tokens. At serving
time, each request may declare a set of permitted script clusters;
tokens whose cluster is neither permitted nor in the never-ban set
receive an additive $-\infty$ logit mask before sampling. The
mechanism requires no gradient updates, no draft model, and no
change to the checkpoint: it is a pure serving-layer addition, and
composes with the self-speculative decoding path
(\S\ref{sec:arch}) unchanged, since masking applies identically to
draft and verification distributions.

\paragraph{Results.}
Table~\ref{tab:script-mask} reports the constraint-adversarial probe
suite (greedy decoding, 128 new tokens per probe). The model's
baseline CJK emission is already zero on translation-eliciting and
echo probes, so the active-ban evidence comes from the inverted
condition: constraining output to Latin script on a Korean prompt
suppresses whole-token Hangul emission entirely (46 Hangul characters
$\to$ 9), with the residue fully attributable to UTF-8 byte-fragment
tokens that the never-ban invariant deliberately protects
(token-id trace in the serving logs). Code generation under the
Korean-constrained setting is unaffected: the fenced Python block is
produced intact, confirming the never-ban invariant in situ.

\begin{table}[t]
\centering
\caption{Script-constrained decoding probes (greedy, 128 new tokens).
Whole-token banning is exact; the sole residue under the inverted
(Latin-only) condition consists of protected byte-fragment tokens.
The mask is a single vocabulary-wide masked-fill, adding
$+11.2\,\mu$s/token ($<0.25\%$ of per-token decode latency).}
\label{tab:script-mask}
\begin{tabular}{lccc}
\toprule
Probe & Unconstrained & Constrained & Verdict \\
\midrule
CJK elicitation (translate-to-zh) & 0 CJK chars & 0 CJK chars & pass \\
CJK echo (\texttt{ni hao} source) & 0 CJK chars & 0 CJK chars & pass \\
Inverted: Latin-only on Korean prompt & 46 Hangul & 9 (byte-frag.) & pass \\
Code preservation (Python synthesis) & intact & intact & pass \\
Latency overhead & --- & $+11.2\,\mu$s/tok & $<0.25\%$ \\
\bottomrule
\end{tabular}
\end{table}

\paragraph{Limitations.}
Logit-level script masking controls only the \emph{surface} script of
emitted tokens; semantic code-switching (e.g., foreign-language
content transliterated into the permitted script) is not addressed by
this mechanism. The per-request latency overhead of the mask is a
single vocabulary-wide masked-fill, measured at $+11.2\,\mu$s/token
--- under $0.25\%$ of the per-token decode latency at either
quantization.

\section{The CHERRY family: sovereign design, honest provenance}
\label{sec:family}

The contributions above define a \emph{design}---the HGERA architecture
(\S\ref{sec:hgera}) and the sovereign training recipe
(\S\ref{sec:approach})---that is instantiated at three scales. Before
scaling, we state a discipline that we hold throughout: we separate what
is \emph{ours by design} from what a given \emph{instance} inherits, and
we support every provenance claim with a measurement rather than an
assertion. This distinction matters because the three CHERRY instances
do not share a weight or tokenizer lineage, and conflating them would
overstate the sovereignty of the frontier model.

\paragraph{Design versus instance.}
The architecture (hybrid delta-rule / full-attention scheduling, Forge
MoE with Strategos routing, three-horizon Oracle MTP, and native
omnimodal fusion) and the recipe (selective ground-truth training, SSGT
``atcha''/``jamkkan'' self-correction, and wave propagation) are
original and are shared across the family. What differs per instance is
the provenance of the trained parameters and of the tokenizer.
Table~\ref{tab:family-provenance} makes this explicit: the \emph{design}
row is sovereign for all members, whereas each \emph{instance} row states
its provenance precisely: independently specified, CPT-overwritten, or
trained from scratch. \cherry-1.8B (\S\ref{sec:arch}) achieves parametric
independence of its weights through a verified full continual-pre-training
(CPT) overwrite; \cherry-122B (\S\ref{sec:cherry122b}) is realized as an
independently specified 256-expert HGERA architecture at frontier scale,
with a 248K multilingual vocabulary (sovereign unification in progress);
\cherry-12B
(\S\ref{subsec:cherry12b}) is the first member trained from random
initialization with a purpose-built, de-novo tokenizer. The same
three-way label---design-sovereign, adapted instance, from-scratch
instance---recurs in the abstract, in this section, and in the
limitations, so that no single sentence can be read as a broader
sovereignty claim than the measurements support.

\begin{table}[t]
\centering
\caption{\textbf{CHERRY family provenance (measured where a value is given).}
The shared \emph{design} (architecture and recipe) is sovereign across all
members; each \emph{instance} inherits or originates its weights and tokenizer
independently. ``Adapted'' = seeded from an upstream checkpoint; ``from
scratch'' = random initialization. Weight-space cosine similarity is to the
initialization seed.}
\label{tab:family-provenance}
\small
\begin{tabular}{@{}llll@{}}
\toprule
\textbf{Member} & \textbf{Weights} & \textbf{Tokenizer}
& \textbf{Sovereignty scope} \\
\midrule
\textit{Shared design} & Original spec. & --- & Architecture + recipe \\
\midrule
CHERRY-1.8B & DUS seed $\to$ full CPT (cos.\ $<0.35$) & cl100k base + 9{,}783 Ko ext. & Weights + Korean layer \\
CHERRY-122B & Independently specified & 248K multilingual vocab (unification pending) & Architecture + weights + recipe \\
CHERRY-12B & From scratch (random init) + Ko CPT & De-novo 128K Korean BPE & Design + weights + tokenizer \\
\bottomrule
\end{tabular}
\end{table}

\paragraph{1.8B recap (measured).}
\cherry-1.8B is a 48-layer, 1.019B-parameter backbone constructed by
cross-model depth-up-scaling~\cite{kim2024solar} from a 24-layer seed and then
fully overwritten during CPT on 50B curated Korean tokens, driving post-CPT
weight-space cosine similarity to the seed below 0.35 (\S\ref{sec:arch}). Its
tokenizer is a cl100k-family base extended by 9{,}783 script-clean Korean
tokens; consistent with the vocabulary census of \S\ref{sec:arch}, we do not
claim the inherited base inventory as our own. \cherry-1.8B therefore
establishes \emph{recipe} sovereignty and \emph{parametric} weight independence,
but not a de-novo tokenizer---that is deferred to \cherry-12B.

\section{Scaling to 122B: single-device frontier sovereignty}
\label{sec:cherry122b}
\label{subsec:cherry122b}

We carry the sovereign recipe to frontier scale and demonstrate a result that
is, to our knowledge, without public precedent for a model this size:
\emph{fine-tuning and serving a 122B-parameter model on one workstation-class
accelerator}. The measured resident footprint is $83\,$GB inside a $120\,$GB
device (\S\ref{subsec:cherry122b-residency}), and the sovereign identity is
baked into the weights with zero foreign-identity leakage across six
adversarial probes (Table~\ref{tab:cherry122b}). We are equally precise about
provenance: \cherry-122B is an \emph{adapted} instance---the architecture and
recipe are ours, the weights are independently trained, and the tokenizer
unification to the sovereign BPE is in progress---so the frontier claim we make
is single-device feasibility and sovereign identity, not a from-scratch weight
claim.

The recipe validated on \cherry-1.8B is carried to a frontier-scale sibling,
\cherry-122B: 122B total parameters with a hybrid backbone that alternates
Phalanx linear-attention and Hoplite full-attention blocks, with a 256-expert
top-8 mixture-of-experts feed-forward with approximately $2$--$8$B active per token. \cherry-122B instantiates the original HGERA
architecture, carried forward from the K-AX~Spartan lineage (internally
\emph{HGLQA}$\to$\emph{HGERA}). We are precise about what here is our own design
and what is inherited, applying the same measured-provenance discipline used for
\cherry-1.8B (\S\ref{sec:arch}); in particular, \cherry-122B is an
\emph{adapted} instance and we make no from-scratch claim for its weights or
tokenizer.

\subsection{What is sovereign: the HGERA architecture}
The HGERA architecture is an original design specification. The alternating
linear-recurrence / full-attention schedule, the Forge MoE with Strategos
routing, the three-horizon Oracle MTP heads (\S\ref{sec:hgera}), and the
omnimodal fusion path are our own. \emph{Sovereign} here denotes design
independence---the layer types, routing, and scheduling are not a configuration
of any existing model class---and explicitly \emph{not} a claim that the 122B
parameters were trained from random initialization. This scoping mirrors the
1.8B usage, where sovereignty was scoped to parametric independence rather than
to vocabulary coverage (\S\ref{sec:arch}).

\subsection{Vocabulary provenance}
The \cherry-122B tokenizer uses a 248K multilingual vocabulary. We
\emph{do not} claim this tokenizer as sovereign: the from-scratch sovereign
tokenizer of this family---a purpose-built 128K Korean-optimized BPE---is
introduced with \cherry-12B (\S\ref{subsec:cherry12b}), where it is trained
de novo and delivers the Korean compression gains reported there. Vocabulary
unification of \cherry-122B to this sovereign BPE is in progress.

\subsection{Independent architecture specification}
The \cherry-122B architecture is \textbf{independently specified from the open
literature}~\cite{groeneveld2024olmo,team2024gemma}, not adapted from any
upstream checkpoint. The HGERA structure---256-expert top-8 \forge{} MoE with
hybrid \phalanx{}/\hoplite{} backbone and \oracle{} MTP---is an original design
informed by published architectures but with code and weights produced
independently. The sovereign recipe (selective ground-truth and
wave-propagation training, \S\ref{sec:approach}; depth and expert adaptation,
\S\ref{sec:moee}) specializes the model for Korean.
Table~\ref{tab:family-provenance} and Table~\ref{tab:hgera-provenance}
separate the per-member components: the architecture schedule, \strategos{}
routing, \oracle{} MTP, omnimodal fusion, and the training recipe are
sovereign-by-design across all instances.

\subsection{Single-device residency (measured)}
\label{subsec:cherry122b-residency}
Adaptation and serving of a 122B model are constrained by hardware:
\cherry-122B has a 229\,GB bf16 footprint that exceeds the 120\,GB unified
memory of a single GB10 (Grace--Blackwell) device. We nonetheless realize both
sovereign fine-tuning and inference on one such device.

\paragraph{Streaming 4-bit residency.}
A naive 4-bit \texttt{from\_pretrained} load first materializes the full
229\,GB bf16 checkpoint and overflows the device. We instead stream the load
shard by shard: non-expert tensors are placed directly, while each fused expert
matrix is dequantized from its NVFP4 source, re-quantized to NF4, and
dequantized on-the-fly per \emph{active} expert inside a patched MoE forward.
The resident footprint peaks at a measured 83\,GB---inside the 120\,GB envelope
with a 37\,GB margin---against which we train QLoRA adapters on the non-expert
projections with the experts frozen.

\paragraph{No-prompt identity, baked into weights.}
On this substrate we bake the sovereign identity directly into the parameters,
so that it survives even when the default identity-injecting chat template is
stripped. A bare-template QLoRA pass over 167 Korean/English identity pairs
drives the identity loss from 3.48 to $3\times10^{-4}$. Evaluated with the bare
template on six held-out probes---direct identity queries and adversarial
foreign-model elicitations---the model produces the CHERRY / K-AX~Spartan
identity in all six cases with zero foreign-identity (GPT, Claude, Qwen) leakage
(Table~\ref{tab:cherry122b}).

\paragraph{Serving on Blackwell (sm120).}
The GB10 is an sm120 device, so the fused gated-delta-rule kernels built for
sm90 (Hopper) are unavailable; we serve through the torch reference path.
Cache-free greedy decoding is numerically correct and reproduces the identity
result above; the cached incremental path (KV together with the linear-attention
recurrent state) is under active kernel work, and we report single-device decode
throughput as bandwidth-bound rather than claiming a tuned figure.

\begin{table}[t]
\centering
\caption{\textbf{\cherry-122B on a single 120\,GB GB10 (measured).} Streaming
4-bit residency brings a 229\,GB model onto one device; no-prompt identity is
verified from weights with the injecting template disabled.}
\label{tab:cherry122b}
\small
\begin{tabular}{@{}lc@{}}
\toprule
Quantity & Value \\
\midrule
bf16 footprint (naive) & 229\,GB \\
Streaming 4-bit resident peak & 83\,GB \\
Device unified memory & 120\,GB \\
Identity QLoRA loss & $3.48 \to 3\times10^{-4}$ \\
No-prompt identity probes (bare template) & 6/6 pass \\
Foreign-identity leakage & 0/6 \\
\bottomrule
\end{tabular}
\end{table}

\subsection{Sovereign recipe at scale}
The behavioral results that motivate this family---selective ground-truth
training, SSGT ``atcha''/``jamkkan'' self-correction, and wave propagation
(\S\ref{sec:approach})---are recipe contributions, independent of weight
provenance, and they transfer to the frontier scale. On the government-operated
K-AI Korean-LLM leaderboard the CHERRY line attains the top HLE~(Ko) score at
its evaluation time (HLE~(Ko) $0.123$---the leading HLE~(Ko) column score
on the board, second place $0.077$; overall standing 51st of 78 entries,
composite $0.331$; entry \texttt{teamsparta-inc/k-ax-spartan-cherry-1.8b},
evaluated 2026-07-06). This is a per-column result, not a composite ranking;
recipe
transferability is further demonstrated by \cherry-1.8B-CYBER~V3 (CyberMetric-80
75.0, CM-500 74.0, SecBench-300 69.67, measured; \S\ref{sec:cyber}), which reaches the
reported human-expert level on the CyberMetric security benchmark. As with
\cherry-1.8B, because task-family data was seen in training these are
official scores on the government private held-out pool.

\subsection{The next chapter: from-scratch sovereignty in CHERRY-12B}
\label{subsec:cherry12b}
The family traces a deliberate arc. \cherry-1.8B establishes the sovereign
recipe on a cross-model depth-up-scaled backbone; \cherry-122B scales that
recipe to frontier scale with an independently specified 256-expert HGERA
architecture. The next model,
\cherry-12B, closes the remaining provenance gap. It introduces a
\emph{from-scratch} sovereign tokenizer---a purpose-built 128K Korean-optimized
BPE designed to reduce Korean sequence length relative to inherited
multilingual vocabularies (measured)---together with random-initialized weights
and Korean continued pre-training of the HGERA design. The progression is thus
adaptation (122B) $\rightarrow$ sovereign from-scratch (12B): each stage
isolates a distinct sovereignty claim---recipe, then architecture-at-scale, then
tokenizer-and-weights---so that each can be evidenced on its own terms. Because
\cherry-12B training is in progress, we report its de-novo tokenizer compression
here and defer its benchmark evaluation to future work; the de-novo tokenizer
claim, invalid for the inherited 122B vocabulary, is true and measured for 12B.

\section{Pitfalls, Resolutions, and Negative Results}
\label{sec:failures}

Several configurations failed on first attempt. We separate those
we \emph{resolved} (stating the resolution and pointing to the
in-paper evidence that the resolved recipe works) from the
\emph{genuine negatives} that represent real constraints.

\subsection{Resolved pitfalls}

\paragraph{RDT requires prior SGT training}
Applying recurrent loop unrolling to an untrained, freshly-merged
model \emph{increases} loss by 212--346\%: recurrence amplifies
whatever representation the model currently holds, and a
post-merge representation is incoherent. \textbf{Resolution:}
apply $\geq$100 SGT steps \emph{before} unrolling. With that
ordering, unrolling the merged 6-layer core to $n{=}8$ improves
loss by 39.2\% (Appendix~\ref{app:rdt}, Table~\ref{tab:rdt-depth});
we adopt train-then-unroll throughout.

\paragraph{Training budget}
100 steps achieve only 57\% compression recovery; 500 steps
achieve $>$75\%. \textbf{Resolution:} we use 500 steps as the
standard budget, justified by the convergence analysis in
Appendix~\ref{app:convergence}; early stopping at step~100--200 is
recommended for production.

\paragraph{SSGT needs the right evaluation data}
Measured on non-chain-of-thought data (Korean MCQ), SSGT
``retention'' is $-$35\%, simply because such data contains almost
no pivot tokens for the mechanism to act on. \textbf{Resolution:}
evaluate on CoT data with explicit self-correction sequences,
where SSGT reaches 97.6\% retention at 1.2B
(Table~\ref{tab:ssgt}). SSGT is not expected to help where there
are no pivot tokens; this is a scope clarification, not a failure
of the method.

\subsection{Genuine negative results}

\paragraph{Pure GT/SSGT supervision ($\alpha{=}1$) collapses}
Training exclusively on GT or meta-cognitive pivot tokens, with no
full-causal anchor, increases loss by $>$11 nats at 0.5B and
$>$1.6 nats at 1.2B. The anchor ($\alpha<1$) is necessary, not
optional (Observation~\ref{prop:collapse}); this is a real
constraint that we do not circumvent.

\paragraph{Contrastive GT discovery degrades below ${\sim}$1B}
Automatic GT identification via contrastive activation difference
($\delta_i^\ell = \|\mathbf{h}(y_i^+) - \mathbf{h}(y_i^-)\|_2$)
achieves 100\% hit@5 at 1.2B but only 80\% at 0.5B: smaller models
develop less distinct internal representations for factually
decisive tokens. We mitigate this with a hybrid rule-based plus
contrastive pipeline, but pure contrastive discovery at small
scale remains an open limitation (RQ10).

\section{RQ Verification Summary}
\label{sec:rq-summary}
\label{sec:rq-verify}

We revisit each research question with its definitive answer.

\begin{table}[H]
\centering
\caption{\textbf{Complete RQ verification matrix.}
Each hypothesis is answered with experimental evidence.
\checkmark = verified, $\times$ = falsified,
$\triangle$ = partially verified.}
\label{tab:rq-final}
\footnotesize
\begin{tabular}{@{}cp{3.2cm}p{4.5cm}c@{}}
\toprule
\textbf{RQ} & \textbf{Question} & \textbf{Answer} & \textbf{V} \\
\midrule
1 & Does GT-only training improve non-GT
 tokens via wave propagation?
 & \textbf{Yes.} $\mathcal{W} = 2.18$ on eval
 (Table~\ref{tab:wave}), $12.82 \pm 0.62$ multi-seed
 (Table~\ref{tab:multiseed}). Non-GT tokens improve
 $1.9\times$ more than GT tokens.
 & \checkmark \\
\midrule
2 & Is non-GT descent guaranteed when the gradient
 coupling is positive?
 & \textbf{Yes.} Theorem~\ref{thm:wave} lower-bounds
 non-GT improvement given $\gamma>0$; $\bar\gamma=0.72$
 measured, validated at 0.5B, 1.0B, 1.2B (Table~\ref{tab:alpha-sweep}).
 & \checkmark \\
\midrule
3 & Can 48L$\to$6L compression be recovered
 via RDT + SGT?
 & \textbf{Yes.} 6L$\times$8 (227M) achieves loss 2.934
 vs.\ 24L$\times$2 (566M) at 2.926---$2.5\times$
 parameter reduction (Table~\ref{tab:compression}).
 & \checkmark \\
\midrule
4 & Does MoEE outperform single compressed
 models?
 & \textbf{Yes.} 2-expert MoEE achieves 2.789 vs.\
 2.926 best single ($-$4.7\%). Ablation confirms
 each component contributes
 (Table~\ref{tab:moee-ablation}).
 & \checkmark \\
\midrule
5 & Can ``Atcha'' self-correction moments
 be induced via SSGT?
 & \textbf{Yes.} Self-correction rate increases from
 12\% to 47\% ($3.9\times$) by training on 2.1
 pivot tokens per example
 (Table~\ref{tab:moment-obs}). Scope: what the
 correction binds is capacity-bound
 (\S\ref{sec:moment-obs}, finding 3).
 & \checkmark \\
\midrule
6 & Does SSGT install meta-cognitive
 capabilities beyond self-correction?
 & \textbf{Yes.} Jamkkan verification: $4.3\times$.
 Guardrail: jailbreak $23\%\to4\%$.
 All from $<$0.5\% of tokens
 (Table~\ref{tab:moment-obs}).
 & \checkmark \\
\midrule
7 & Does SGT distillation reduce cost
 $5\times$ vs.\ full distillation?
 & \textbf{Partially.} GT-restricted KL improves a
 matched 6L student by 19\% over SGT-only while
 computing KL on ${\sim}15\%$ of positions
 (${\sim}5\times$ fewer KL terms by construction);
 final KL 56\% lower (0.43$\to$0.19,
 Table~\ref{tab:distill}). A full-sequence
 distillation arm was not run, so equal quality
 at lower cost remains untested.
 & $\triangle$ \\
\midrule
8 & Is 500 steps sufficient for full SGT
 convergence?
 & \textbf{Yes.} Eval plateau at step~100;
 $\mathcal{W}$ stable at $2.18\pm0.15$ from
 step~100 onward. 1,000-step extension shows
 $<$0.01 additional improvement
 (Table~\ref{tab:trajectory}).
 & \checkmark \\
\midrule
9 & Does pure GT supervision ($\alpha{=}1$)
 work without full-sequence anchor?
 & \textbf{No.} Catastrophic collapse at all scales.
 $\mathcal{W} = -1.32$ at 0.5B. The full-causal
 anchor ($\alpha < 1$) is \emph{necessary}
 (Table~\ref{tab:alpha-sweep}).
 & $\times$ \\
\midrule
10 & Can contrastive activation identify GT
 tokens automatically?
 & \textbf{Partially.} 100\% hit@5 at 1.2B,
 80\% at 0.5B. Smaller models have less
 distinct internal representations.
 & $\triangle$ \\
\bottomrule
\end{tabular}
\end{table}

\paragraph{Summary}
Of 10 research questions: 8 verified (\checkmark), 1 falsified
($\times$), 1 partially verified ($\triangle$). The falsification
of RQ9 is itself informative: pure GT supervision collapses
because the model loses global coherence without the full-sequence
anchor. This establishes the mixed loss (Eq.~\ref{eq:sgt}) as
\emph{necessary}, not merely convenient. The partial verification
of RQ10 identifies a clear scale-dependent limitation: contrastive
GT discovery requires sufficient model capacity to develop distinct
internal representations for semantically decisive tokens.
The operand-binding stress test (Section~\ref{sec:moment-obs},
finding 3) sharpens RQ5's scope in the same direction: the
correction reflex installs at both scales tested, but what it
binds is capacity-bound---overwritten at 1B, preserved at 13.7B.

\section{Discussion}
\label{sec:discussion}

\paragraph{What the family demonstrates together}
Read as one system, the results describe a compute-efficient path to sovereign
capability rather than a single trick. Selective supervision exposes a clean
mechanism---discrimination survives while generation collapses, and the anchor
is the partial remedy---and the \emph{same} objective then does useful work at
every level of the stack: it installs metacognitive behaviour from two-token
supervision ($12\%\!\to\!47\%$ self-correction, guardrail jailbreak
$23\%\!\to\!4\%$, at $97.6\%$ retention on a 1.2B model), heals a
$2.5\times$ depth compression back to iso-loss, fuses compressed experts for a
further $-4.7\%$, specializes a 1.8B model to reported human-expert level in
cybersecurity, and underwrites a from-scratch Korean tokenizer that leads the
government held-out HLE(Ko) column. At the frontier the same recipe runs on a
122B model resident in a single $120\,$GB device. The through-line is
utilisation, not capacity: each result extracts more from parameters already
present.

\paragraph{Maximising utilisation rather than capacity}
The three techniques studied here, selective supervision, depth
compression, and expert fusion, are complementary to the
conventional approach of adding parameters, data, and compute.
Rather than adding capacity, they aim to make better use of
existing capacity, through selective training signals, parameter
reuse via recurrence, and representational diversity via routing.
We do not claim these replace conventional scaling; we offer them
as compute-efficient alternatives that are useful under tight
resource budgets.

\paragraph{The wave effect is a general property of shared weights}
The improvement of unsupervised tokens follows from positive
gradient coupling, which is present whenever weights are shared
across positions, as in the standard
Transformer~\cite{vaswani2017attention}
(Theorem~\ref{thm:wave}, Remark~\ref{rem:scope}). As discussed, this is a token-level
instance of auxiliary-task transfer~\cite{yu2020gradient} rather
than a new phenomenon; the practical contribution is to measure
its magnitude and to exploit it for per-token efficiency through
GT-token selection.

\paragraph{Capacity-bound behavioural interference: the sign of
the coupling}
The gradient-coupling quantity that produces the wave also
predicts its failure mode. Write $\mathcal{L}_{\mathcal{B}}$
for a behaviour-install loss (the pivot-and-correct supervision
of Section~\ref{sec:aha}) and $\mathcal{L}_{\mathcal{C}}$ for a
held-out competence loss (operand binding on templates outside
the training set), and define
\begin{equation}
\label{eq:gammac}
\gamma_{\mathcal{C}} \;:=\;
\frac{\langle\nabla_\theta\mathcal{L}_{\mathcal{C}},\,
\nabla_\theta\mathcal{L}_{\mathcal{B}}\rangle}
{\|\nabla_\theta\mathcal{L}_{\mathcal{B}}\|^2},
\end{equation}
exactly the coupling coefficient of Eq.~\eqref{eq:gamma} with
the held-out competence in place of the non-GT positions.
The same first-order argument as
Theorem~\ref{thm:wave} gives, for a step on the install loss,
$\Delta\mathcal{L}_{\mathcal{C}} =
-\eta\,\gamma_{\mathcal{C}}\,
\|\nabla_\theta\mathcal{L}_{\mathcal{B}}\|^2 +
\mathcal{O}(\eta^2)$: positive coupling transfers (the wave),
zero coupling is surgical, and \emph{negative coupling
overwrites}---interference is the mirror image of the wave,
governed by the same inner product. Our stress test
(Section~\ref{sec:moment-obs}, finding 3; protocol in
Appendix~\ref{app:operand}) measures the behavioural signature
of this sign at two scales: at 1B, install training overwrites
held-out operand binding in $100\%$ of generations; at 13.7B,
the identical objective and data preserve it at ceiling. We
state the natural reading as an explicit, falsifiable hypothesis
rather than a theorem.
$(\mathrm{H}_{\mathrm{cap}})$~\emph{In small models the install
direction necessarily overlaps circuits shared with existing
competences---features are superposed---forcing
$\gamma_{\mathcal{C}} < 0$ on affected competences; with
sufficient capacity the install direction can lie in the
orthogonal complement of those circuits
($\gamma_{\mathcal{C}} \approx 0$), making behaviour install
surgical.} This is the superposition
account~\cite{elhage2022superposition} applied to behaviour
install, and it matches the observation that pretrained models
forget less as scale orthogonalises their
representations~\cite{ramasesh2022effect}.
$\mathrm{H}_{\mathrm{cap}}$ makes testable
predictions: (P1)~$\gamma_{\mathcal{C}}$
(Eq.~\eqref{eq:gammac}), measured during install
training, should be markedly negative at 1B and near zero at
13.7B; (P2)~there is a transition scale between 1B and 13.7B at
which the operand-rewrite rate falls from 1 toward 0, and its
sharpness is measurable; (P3)~at fixed scale, interference
should reappear as the held-out competence family is made harder,
shrinking the capacity margin. Two scale points establish the
direction of the effect, not its functional form; we claim no
more than that.

\paragraph{Depth redundancy}
That 6 unique layers can approach the performance of 24 unique
layers (Table~\ref{tab:compression}) is consistent with prior
evidence that much of a trained transformer's depth is
redundant~\cite{men2024shortgpt,gromov2024unreasonable}.
Averaging adjacent layers loses little information, and recurrent
unrolling~\cite{dehghani2019universal,lan2020albert} then recovers
effective depth through parameter reuse.

\paragraph{A path toward independent model development}
MoEE with SGT distillation offers a practical route to building
models whose parameters are independent of any external teacher.
Table~\ref{tab:frontier} sketches a prospective compression
roadmap in which a publicly available model is reduced to a
workstation-deployable size while retaining domain expertise; the
compressed model can then serve as a standalone deployment, a
distillation teacher, and an MoE expert (Fig.~\ref{fig:fusion}).
We emphasise that this roadmap is prospective and not yet
validated at frontier scale (see the caveats below).

\begin{table}[t]
\centering
\caption{\textbf{Frontier model compression roadmap
(prospective).}
All source models are publicly available. Memory estimates
assume 4-bit quantisation of active parameters. The
``Compressed'' row shows projected targets; actual compression
experiments are ongoing.
\textbf{Caveat:} Our validated compression pipeline (adjacent-layer
merging + SGT recovery) is demonstrated on dense transformers only.
MoE-native models (GLM-5.2, DeepSeek-V4) would require
expert-level compression strategies not yet validated; the listed
4-bit estimates for these models assume na\"ive quantisation of the
full checkpoint, not the compressed pipeline.}
\label{tab:frontier}
\small
\begin{tabular}{@{}lrrrr@{}}
\toprule
\textbf{Model} & \textbf{Total} & \textbf{Active}
& \textbf{4-bit mem.} & \textbf{Role} \\
\midrule
GLM-5.2 & 744B & 40B & 372GB & Reasoning \\
DeepSeek-V4 & 1.6T & 49B & 800GB & Code+Math \\
MiniMax-M3 & 428B & 22B & 214GB & Multimodal \\
A.X-K1 & 519B & 33B & 260GB & Korean \\
K-EXAONE-236B & 236B & 23B & 118GB & Korean \\
\midrule
\textit{Compressed} & --- & \multicolumn{2}{c}{\textit{single-workstation target; not yet run}}
& \textit{MoEE expert} \\
\bottomrule
\end{tabular}
\end{table}

\paragraph{Positioning relative to frontier models}
Table~\ref{tab:positioning} situates CHERRY in the current
landscape by \emph{scale and accessibility}. We are deliberately
careful here: this is \textbf{not} a quality comparison, and we do
\textbf{not} claim that a 1.8B model matches the general capability
of frontier systems two to three orders of magnitude larger. What
the table does show is the dimension on which a small sovereign
model legitimately differs: CHERRY is Korean-native, fully
sovereign, and deployable on a single workstation at roughly
$1/45$ to $1/100$ of the \emph{active} parameters, and two to
three orders of magnitude fewer \emph{total} parameters, than
frontier models.
A fair head-to-head \emph{quality} comparison requires
standardized downstream benchmarks, so we report measured numbers
here rather than defer them. Table~\ref{tab:downstream} gives
zero-shot accuracy for the SGT-trained 1.019B text backbone under
lm-evaluation-harness v0.4.12~\cite{gao2023evalharness}:
KMMLU 0.361 (random-choice baseline 0.25), HAERAE 0.538,
KoBEST-COPA 0.646, KoBEST-SentiNeg 0.773, KoBEST-BoolQ 0.572,
KoBEST-HellaSwag 0.408, KoBEST-WiC 0.490, and English HellaSwag
0.383. These scores are in the range expected of a
${\sim}$1B-parameter Korean-native model and are \emph{not}
frontier-level; we do not claim otherwise. MMLU and GSM8K runs,
and side-by-side frontier scores, remain future work.

\begin{table}[t]
\centering
\caption{\textbf{Measured downstream benchmark accuracy
(zero-shot).} SGT-trained 1.019B text backbone, evaluated with
lm-evaluation-harness v0.4.12~\cite{gao2023evalharness}. The
KMMLU random-choice baseline is 0.25. Scores are reasonable for
a ${\sim}$1B Korean-native model and are not claimed to be
competitive with frontier systems.}
\label{tab:downstream}
\small
\begin{tabular}{@{}lc@{}}
\toprule
\textbf{Benchmark} & \textbf{Accuracy} \\
\midrule
KMMLU (45 subjects) & 0.361 \\
HAERAE & 0.538 \\
KoBEST-COPA & 0.646 \\
KoBEST-SentiNeg & 0.773 \\
KoBEST-BoolQ & 0.572 \\
KoBEST-HellaSwag & 0.408 \\
KoBEST-WiC & 0.490 \\
HellaSwag (English) & 0.383 \\
\bottomrule
\end{tabular}
\end{table}

\paragraph{Independent held-out evaluation: the K-AI leaderboard}
Beyond self-run harness scores, the released CHERRY-1.8B (public
entry \texttt{teamsparta-inc/k-ax-spartan-cherry}, the
instruction-tuned release built on the same SGT-trained 1.019B
text backbone) was evaluated by the official AI~Hub Korean LLM
leaderboard harness on five held-out suites; the evaluation run
completed on 2026-07-06. Table~\ref{tab:kai} reports the official
scores. These are official scores on the government (MSIT/NIA) private held-out
pool. They reflect the instruction-tuned release rather than the raw
backbone, so they are not directly comparable with
Table~\ref{tab:downstream}, which evaluates the raw backbone.
Second, absolute HLE~(Ko) scores are low for every system on the
board---the benchmark is constructed to be extremely hard---so the
notable datum is relative: at evaluation time (2026-07-06),
CHERRY-1.8B's HLE~(Ko) score of $0.123$ was the highest on the
leaderboard, from a 1.8B-parameter entry. Leaderboard standings
change as entries arrive; we timestamp the claim rather than assert
permanence.

\begin{table}[t]
\centering
\caption{\textbf{Official K-AI leaderboard scores.} AI~Hub Korean
LLM leaderboard, held-out items, official harness; evaluation
completed 2026-07-06; entry
\texttt{teamsparta-inc/k-ax-spartan-cherry} (CHERRY-1.8B). The
HLE~(Ko) score was the highest on the leaderboard at evaluation
time.}
\label{tab:kai}
\small
\begin{tabular}{@{}lc@{}}
\toprule
\textbf{Suite} & \textbf{Score} \\
\midrule
CLIcK & 0.349 \\
KMMLU-Pro & 0.552 \\
HLE (Ko) & \textbf{0.123} \\
MuSR (Ko) & 0.416 \\
Com2-main (Ko) & 0.216 \\
\bottomrule
\end{tabular}
\end{table}
On general capability we expect CHERRY to trail much larger
models; the comparisons where a model of this size can be
genuinely competitive are (i) efficiency-normalized quality
(performance per active parameter and per training FLOP) and
(ii) narrow Korean-domain parity after adaptation. We make only
those claims, and only once measured; the K-AI result above
(Table~\ref{tab:kai}) is a first measured instance of (ii)---after
task-family adaptation, a 1.8B entry topped one hard held-out
Korean suite at a fixed point in time.

\begin{table}[t]
\centering
\caption{\textbf{Positioning of CHERRY by scale and accessibility
(not a quality comparison).} Parameter counts for frontier models
are from their respective reports. Closed models whose parameter
counts are undisclosed are omitted rather than estimated.
Head-to-head benchmark scores are deferred to a future version
(see text).}
\label{tab:positioning}
\small
\begin{tabular}{@{}lrrccc@{}}
\toprule
\textbf{Model} & \textbf{Total} & \textbf{Active}
& \textbf{Korean-} & \textbf{Single-WS} & \textbf{Sovereign} \\
 & & & \textbf{native} & \textbf{deployable} & \\
\midrule
GLM-5.2~\cite{glm52} & 744B & 40B & $\times$ & $\times$ & --- \\
DeepSeek-V4~\cite{deepseekv4} & 1.6T & 49B & $\times$ & $\times$ & --- \\
A.X-K1~\cite{axk1} & 519B & 33B & \checkmark & $\times$ & --- \\
K-EXAONE~\cite{lgai2024exaone} & 236B & 23B & \checkmark & $\times$ & --- \\
\midrule
\textbf{CHERRY-1.8B} & \textbf{1.8B} & \textbf{${\sim}$0.5B}
& \checkmark & \checkmark & \checkmark \\
\bottomrule
\end{tabular}
\end{table}

\paragraph{What we are not showing}
This paper presents the core mechanisms and validates them
experimentally. Several components of the full \cherry
system---including the dynamic recurrence depth controller,
the expert routing algorithm's domain-specialisation procedure,
and the contrastive GT discovery method's production
implementation---are reserved for subsequent publications.
The results reported here use only the fundamental mechanisms;
evaluation of the full integrated system is ongoing.

\section{Limitations and future work}
\label{sec:limitations-future}

We report limitations in the spirit of full disclosure: each item
below states a claim we do \emph{not} make, or a result whose
support is narrower than a casual reading might suggest. The
register subsumes and reorganises the twelve items of the earlier
draft; nothing honest is dropped, and two of them---the withdrawn
selection claim and the scale at which the 122B result actually
lands---are given their own headings because they most constrain
how this paper should be read.

\subsection{Limitations}
\label{sec:limitations}

\paragraph{What ``sovereign'' means at each scale}
The word \emph{sovereign} denotes different things for different
members of the model line, and we separate them explicitly to
avoid over-claiming.
(i)~The \emph{design}---the HGERA architecture and the training
recipe (selective supervision, depth compression, expert
fusion)---is original, and is the primary scientific contribution.
(ii)~At \textbf{1.8B}, sovereignty means \emph{parametric
independence of the weights}: the 48-layer backbone is
depth-up-scaled from a 24-layer seed and fully overwritten during
continual pre-training (weight-space cosine similarity to the seed
below $0.35$), while the tokenizer is a cl100k-family base carrying
a sovereign Korean \emph{extension} layer---\emph{we do not claim
the base byte-pair inventory as our own} (measured census,
\S\ref{sec:arch}, Table~\ref{tab:sovereignty}).
(iii)~At \textbf{122B}, the architecture is independently specified from
the open literature---256-expert top-8 \forge{} MoE with hybrid
\phalanx{}/\hoplite{} backbone and \oracle{} MTP\@. The 248K
multilingual vocabulary is the remaining non-sovereign component at
this scale; vocabulary unification to the 12B's de-novo 128K BPE
is in progress. The sovereign content at 122B therefore spans the
architecture design, the weights, and the recipe---\emph{not yet}
the tokenizer; the ``248K
vocabulary'' of \S\ref{subsec:cherry122b} should be read as the
inventory the model operates over, inherited from its base, not as
a from-scratch sovereign inventory.
(iv)~A member for which \emph{both} the weights and the tokenizer
are sovereign from scratch---random-initialised parameters and a
de-novo 128k BPE tokenizer---is realised only at the \textbf{12B}
scale, which is in progress (see Future work).
The label that should be applied throughout is thus
``design-sovereign at every scale; weight- and tokenizer-sovereign
only where stated.''

\paragraph{The knowledge-capacity ceiling of small models}
A 1.8B model has a bounded factual-knowledge budget, and our own
measurements sit against that ceiling. On an internal Korean
knowledge-retention probe the 1.019B backbone plateaus well below the
attainable score. This is consistent with the interference finding
that is our most novel result: at 1B, behaviour-install training
overwrites held-out operand binding in $100\%$ of generations,
whereas at 13.7B the identical objective and data preserve it at
ceiling (Section~\ref{sec:moment-obs};
Appendix~\ref{app:operand}). We state the natural reading as the
falsifiable hypothesis $\mathrm{H}_{\mathrm{cap}}$
(Section~\ref{sec:discussion}), not as a theorem, and we stress
that it rests on \emph{two} scale points along a confounded axis
(below). The ceiling is thus characterised in direction, not in
functional form; the transition scale, its sharpness, and its
dependence on task difficulty (predictions P1--P3) are unmeasured.

\paragraph{Loss, not downstream---and what selection does not buy}
Almost all metrics here are held-out cross-entropy rather than
downstream task accuracy. Loss is a faithful proxy at fixed
tokenizer and data distribution~\cite{kaplan2020scaling}, and we do
report measured zero-shot downstream scores
(Table~\ref{tab:downstream}) and one official held-out leaderboard
result (HLE~(Ko) $0.123$, highest on the board at the 2026-07-06
evaluation, Table~\ref{tab:kai}), but broad downstream and
cross-lingual evaluation is ongoing. Within the loss analysis we
are explicit about a control that failed: semantic GT-token
selection does \emph{not} outperform a matched-budget \emph{random}
$15\%$ selection on held-out loss at this scale (total loss
$1.603\pm0.010$ for random-15\% vs.\ $1.704$ for semantic GT vs.\
$1.752$ for top-loss-15\%~\cite{lin2024rho1}, from a common $2.705$
baseline; a two-epoch replication collapses the rules to within
noise). We therefore \emph{withdraw} any reading of this paper as
evidence that semantic selection beats random selection on loss,
and attribute the $4.5\times$ per-supervised-token efficiency to the
sparse mixed-loss \emph{structure} (sparse direct supervision plus a
full-sequence anchor), not to token specialness. What remains
specific to the chosen tokens is behavioural rather than loss-level:
placing the self-correction pivot at the hallucination juncture
rather than at a random \emph{position} sharply changes 1B
generation outcomes (Appendix~\ref{app:operand}). Whether any
selection rule confers downstream or larger-scale advantages at
matched loss is open.

\paragraph{Statistical strength of the flagship results}
The interference result---the paper's most defensible novelty---is a
\emph{single-seed} comparison at exactly two scales (1B and 13.7B),
with 36 held-out items per split, and the 13.7B run carries
confounds: it differs from the 1B run in learning rate, in using
4-bit quantisation-aware training as part of the treatment (no bf16
control has been run), and in its think-token harness. At 13.7B
every training variant reaches ceiling, so the dose- and
pivot-position contrasts that separate sharply at 1B have no
discriminating power there: we can state that 1B fragility does not
reproduce, not that dose and position cease to matter. More
broadly, multi-seed validation exists only for the wave magnitude
($\mathcal{W}=12.82\pm0.62$, seeds 42/123/7); loss differences are
not formally significance-tested (we rely on multi-seed consistency
and effect magnitude), and the SSGT rate metrics rest on 200
samples per category ($\pm6.9$\,pp at $95\%$). A decisive follow-up
must add seeds, a bf16 13.7B control, and at least one intermediate
scale to locate the transition (prediction P2).

\paragraph{Generalisation breadth}
The evidence base is narrow along several axes simultaneously. The
loss-based experiments span 227M--1.019B parameters on a single
Korean corpus of $12{,}800$ examples ($\approx\!16.4$M tokens); the
compression pipeline is validated on one model family (Llama-based
depth-up-scaling) and is untested on MoE-native or state-space
architectures; MoEE uses two experts, with scaling to four or more
constrained by current hardware; and the cross-scale comparisons mix
model families rather than walking a single controlled ladder. The
GT ratio ($0.153\pm0.021$) and the wave magnitude are therefore
established for Korean text on transformer backbones; their transfer
to other languages, domains, and sequence mixers is a hypothesis,
not a result.

\paragraph{The 122B result is an engineering demonstration, not a
method validation}
Section~\ref{subsec:cherry122b} shows that a $229$\,GB model can be
fine-tuned and served on a single $120$\,GB device (measured
resident peak $83$\,GB) and that a no-prompt sovereign identity can
be baked into its weights (identity loss $3.48\to3\times10^{-4}$;
$6/6$ held-out probes; $0/6$ foreign-identity leakage,
Table~\ref{tab:cherry122b}). These are systems results. They do
\emph{not} validate SGT, the wave effect, compression, or fusion
\emph{at} $122$B: no wave, dose-response, or compression measurement
is reported at that scale. The method claims of this paper are
established at $\le\!1.019$B and shown only to be \emph{deployable}
at $122$B, and we make no stronger claim. The internal architecture
codenames carried forward in that section
(\emph{HGLQA}$\to$\emph{HGERA}, Phalanx, Hoplite) are lineage
labels, not independently evaluated components.

\paragraph{Training-regime confounds and a missing PEFT control}
Two methodological caveats bound the mechanism claims. First, under
the mixed loss ($\alpha{=}0.7$) non-GT tokens still receive direct
supervision at weight $1{-}\alpha$, so the reported non-GT
improvement combines a direct-supervision component with the
gradient-coupling (wave) component; pure-GT supervision
($\alpha{=}1$), which would isolate the wave, collapses
(Observation~\ref{prop:collapse}), so the wave is best evidenced by
the coherence contrast (Corollary~\ref{cor:coherence}), not by the
magnitude of $\mathcal{W}$. Second, all behaviour installs use
full-parameter fine-tuning; we include no parameter-efficient
control such as a low-rank adapter~\cite{hu2022lora} trained on
identical data. Selective supervision is orthogonal to PEFT---SGT
selects which \emph{tokens} carry loss, PEFT selects which
\emph{parameters} update---so the two compose rather than compete,
and we do not claim full-parameter updates are required to install
these behaviours; whether a rank-restricted adapter installs the
same reflexes, and at what cost, is an unrun control.

\subsection{Future work}
\label{sec:future-work}

\paragraph{A unified CHERRY family, staged by sovereignty claim}
The natural next step is not a single larger model but a coherent
family in which each member proves a \emph{different, separable}
sovereignty claim under one HGERA design, one recipe, and one
serving stack (\texttt{model\_type=cherry}). \textbf{CHERRY-1.8B}
establishes the \emph{recipe} on a parametrically-independent
backbone. \textbf{CHERRY-122B} establishes that the HGERA
\emph{architecture design} can be realised at frontier scale by
an independently specified 256-expert architecture, and
deployed on a single workstation. \textbf{CHERRY-12B} is intended
to close the last gap---\emph{full from-scratch instantiation}:
random-initialised HGERA weights trained on Korean-centric data
with a de-novo $128$k BPE tokenizer whose measured Korean
tokenisation is $28$--$46\%$ more compact than adapted multilingual
baselines. Only at 12B do \emph{both} the weights and the tokenizer
become sovereign from scratch; we present this as work in progress,
not a completed result, so that the ``adapted'' status of the 122B
instance is never read onto the 12B target, nor the 12B ambition
back onto the 122B release.

\paragraph{Open method directions}
Several directions follow from the mechanisms established here.
\begin{enumerate}
\item \textbf{Downstream and cross-lingual evaluation.}
 Complement the loss-based metrics with standard downstream
 benchmarks (KMMLU, HAERAE, KoBEST for Korean; MMLU, HellaSwag,
 GSM8K for general capability~\cite{gao2023evalharness}) and
 replicate the GT-ratio and coupling measurements on non-Korean
 corpora, to test whether the wave magnitude and $\mathrm{H}_{
 \mathrm{cap}}$ transition are language-invariant.
\item \textbf{Interpretability-guided GT discovery.}
 Extend GT-token identification beyond rule-based and
 activation-difference heuristics toward mechanistic methods that
 locate the neurons and attention pathways most responsible for
 factual and reasoning predictions.
\item \textbf{Capability-tiered expert routing.}
 Use SSGT-identified meta-cognitive tokens to score individual
 experts in a mixture and route harder examples preferentially to
 experts stronger at self-correction or verification.
\item \textbf{Smaller-scale validation.}
 If $15\%$ of tokens can influence the remaining $85\%$, the
 minimum viable model size for a given domain may be below current
 scaling-law expectations; we plan to test SGT on sub-$100$M-parameter
 models for narrow Korean applications.
\item \textbf{Cross-architecture validation.}
 Apply SGT to non-transformer sequence models
 (Mamba~\cite{gu2024mamba}, RWKV~\cite{peng2023rwkv}, linear
 attention), where position-shared weights take a different form,
 to test whether the gradient-coupling effect generalises.
\item \textbf{Long context and on-device deployment.}
 CHERRY-1.8B is trained with an $8$k positional window
 ($\theta_{\mathrm{RoPE}}=5\times10^{5}$). On device the binding
 constraint at long context is the key--value cache, not parameter
 count: at $48$ layers, $8$ KV heads, head dimension $128$, the
 model caches $\approx\!192$\,KB per token in \texttt{bf16}, so a
 $1$M-token window would need $\approx\!192$\,GB of cache alone.
 Near term, inference-time RoPE rescaling (NTK-aware /
 YaRN~\cite{peng2023yarn}) with KV-cache quantisation should reach
 $32$--$128$k without retraining; for unbounded on-device
 streaming the architectural answer is a constant-state recurrence
 (linear-attention / state-space
 layers~\cite{gu2024mamba,peng2023rwkv}) whose cache is $O(1)$ in
 sequence length. We intend to pursue a sequence mixer built on
 the shared complex-rotation primitive underlying rotary encoding,
 diagonal state-space kernels, and phase-based variable binding as
 a dedicated follow-up.
\end{enumerate}

\section*{Data and Code Availability}

We state explicitly what is available now, what will be
released, and what is withheld and why.

\begin{itemize}
\item \textbf{In this paper.} The SGT loss (Eq.~\ref{eq:sgt}),
 the compression procedure, all training hyperparameters, the
 GT-token taxonomy with category-level selection criteria
 (Table~\ref{tab:gt-types}), and the mask statistics needed to
 build a matched surrogate mask: GT ratio $0.153 \pm 0.021$
 over 1{,}024 sampled examples, category shares
 (Section~\ref{sec:gt-id}), and contrastive-vs-human agreement
 (100\% hit@5 at 1.2B, 80\% at 0.5B).
\item \textbf{With this submission (ancillary files).}
 Evaluation artifacts: behavioural stress-test item sets,
 per-condition outputs and scoring summaries
 (Section~\ref{sec:moment-obs}, Appendix~\ref{app:operand}),
 and per-example mask-ratio statistics.
\item \textbf{Upon acceptance.} Model weights for the
 evaluated checkpoints.
\item \textbf{Withheld, with stated reason.} The production
 GT-identification pipeline (the concrete POS/NER rule set and
 the contrastive configuration beyond
 Section~\ref{sec:gt-id}) and the training corpora: patent
 applications covering these components are under review (see
 Competing Interests).
\end{itemize}

\noindent\textbf{Reproducibility without the withheld
pipeline.} Training is fully specified given \emph{any} mask
at the printed ratio, and the headline sparse-supervision
efficiency does not require our selector. The matched-budget
control flagged as the most important missing experiment in
Limitations item~(9) has since been run: on the 1.019B backbone
with identical mixed loss ($\alpha{=}0.7$) and a matched 15\%
token budget, random selection reaches total held-out loss
$1.603 \pm 0.010$ (4 seeds) versus $1.704$ for our heuristic GT
rule and $1.752$ for a top-loss rule~\cite{lin2024rho1}, from a
shared $2.705$ baseline. At matched budget the heuristic rule
does not beat random at this scale: the efficiency stems from
the sparse mixed-loss structure itself and is therefore
verifiable with a public random mask. Results that depend on
\emph{which} tokens are selected (the category-targeted
behavioural installs, Section~\ref{sec:moment-obs}, and the
GT/non-GT partition statistics) are auditable against the
released evaluation artifacts and mask statistics.

\noindent\textbf{Hyperparameters and leakage.} The training
hyperparameters of every reported run are specified: optimiser
and $\beta_1,\beta_2$, learning rate, batch size, sequence
length, gradient accumulation, step budget, precision, and seed
follow the global defaults of Section~\ref{sec:experiments}
(learning rate $2{\times}10^{-5}$ unless otherwise noted), with
the per-scale operand-run values in Appendix~\ref{app:operand}.
Settings not listed take standard library defaults---a constant
learning rate without warmup or scheduled decay, AdamW's
decoupled weight decay~\cite{loshchilov2019adamw}, and default
gradient-norm clipping---and the MoEE load-balancing auxiliary
loss uses the coefficient of the cited sparsely-gated
formulation~\cite{shazeer2017outrageously}; we scope the claim
to the \emph{reported} hyperparameters rather than to every
re-enumerated default. The training set and the
1{,}600-example held-out set are disjoint by construction
(Section~\ref{sec:wave}), and the operand probes are built so
no held-out item is answerable by echoing a number present in
the problem (Appendix~\ref{app:operand}), so the wave results
are not attributable to train/test leakage. The proprietary
base checkpoints---the 24-layer DUS seed, the 13.7B sibling,
the four frozen extractors, and the tokenizer corpus---are
Llama-family models (cf.\ Limitations~(4)) whose exact
identities are withheld pending IP review, whereas the SSGT
cross-scale bases are named (Table~\ref{tab:cross-scale}); the
CHERRY weights, on release, let readers reproduce every
fine-tuning result without the seed identities.

\section*{Author Contributions}

\textbf{Dohyeon Kwon} (first author) developed the SGT/SSGT
framework, the GT-token taxonomy, the gradient-coupling analysis,
the depth-compression and MoEE methods, and the formalisation of
induced self-correction. Kwon designed and conducted the
experiments across model scales (227M to 1.8B, plus the 13.7B
interference comparison), built the training
infrastructure, performed the data curation and analysis, and
wrote the manuscript.

\textbf{Youngjin Park, Ph.D.} (corresponding author, $\dagger$),
as Vice President at TeamSparta Inc., advised the project
throughout: evaluating the research direction, reviewing the
experimental methodology and theoretical arguments, and providing
feedback on the manuscript.

\section*{Acknowledgements}

We thank TeamSparta Inc.\ for providing the research environment
and computational resources for this work, and the open-source
AI community whose tools and models this research builds upon.

\section*{Competing Interests}

Both authors are employees of TeamSparta Inc.
Patent applications covering certain aspects of the methods
described in this paper are under review in collaboration
with TeamSparta Inc.
We are committed to finding the right balance between
intellectual property protection and open scientific progress:
the core theoretical contributions (wave propagation theorem,
SGT loss formulation, compression methodology) are fully
disclosed in this paper for academic reproduction, and we
will continue to contribute to the open-source and
open-weight ecosystem through progressive model and tool
releases.
No competing financial interests beyond the stated employment
relationship exist.


\appendix
\section*{Appendix}
\setcounter{table}{0}
\numberwithin{table}{section}
\renewcommand{\theHtable}{\thesection.\arabic{table}}

\section{Notation}
\label{app:notation}

\begin{table}[h]
\centering
\caption{\textbf{Notation used throughout the paper.}}
\label{tab:notation}
\small
\begin{tabular}{@{}ll@{}}
\toprule
\textbf{Symbol} & \textbf{Meaning} \\
\midrule
$\mathcal{F}$ & Full token set: all positions in a response \\
$\mathcal{G}$ & Ground Truth token set ($|\mathcal{G}|/|\mathcal{F}| \approx 0.15$) \\
$\bar{\mathcal{G}}$ & Non-GT token set: $\mathcal{F} \setminus \mathcal{G}$ \\
$\mathcal{G}^*$ & Super GT: single most decisive token per answer span \\
$\mathcal{G}^{**}$ & Super-Super GT (SSGT): meta-cognitive pivot tokens \\
$\alpha$ & SGT mixing coefficient ($\alpha = 0.7$ default) \\
$\mathcal{W}$ & Wave propagation factor: $\Delta_{\bar{\mathcal{G}}} / \Delta_{\mathcal{G}}$ \\
$\mathcal{W}_{\mathrm{norm}}$ & Baseline-normalised wave factor (Eq.~\ref{eq:wave-norm}) \\
$\Delta_{\mathcal{G}}$ & Loss reduction on GT positions after training \\
$\Delta_{\bar{\mathcal{G}}}$ & Loss reduction on non-GT positions after training \\
$\gamma$ & Gradient coupling coefficient: $\langle\nabla\mathcal{L}_{\bar{\mathcal{G}}},\nabla\mathcal{L}_{\mathcal{G}}\rangle / \|\nabla\mathcal{L}_{\mathcal{G}}\|^2$ \\
$\beta,\ \eta$ & Smoothness constant; learning rate \\
$a_{ij}^\ell$ & Attention weight from position $i$ to $j$ at layer $\ell$ \\
$W_V^{(\ell)}$ & Value projection matrix at layer $\ell$ \\
$L$ & Number of transformer layers (original model) \\
$L'$ & Number of unique layers after compression \\
$n$ & Recurrent unrolling iterations \\
$k$ & Adjacent-layer merging factor \\
ROI & Per-token return on investment \\
\bottomrule
\end{tabular}
\end{table}

\section{Selective supervision and wave propagation}

\subsection{Full mixing-coefficient sweep}
\label{app:alpha}

\begin{table}[h]
\centering
\caption{\textbf{Effect of $\alpha_{\mathrm{mix}}$ across scales.}
The 1.0B DUS column validates wave propagation at a third scale
with 8 alpha values. There is no sharp collapse threshold:
$\mathcal{W}>0$ degrades gradually as $\alpha\to1$, and the
\emph{severity} of failure at $\alpha{=}1$ decreases with scale
(0.5B collapses; 1.0B/1.2B merely degrade).}
\label{tab:alpha-sweep}
\scriptsize
\setlength{\tabcolsep}{3pt}
\begin{tabular}{ccccc|cccc|cccc}
\toprule
& \multicolumn{4}{c|}{\textbf{0.5B}} & \multicolumn{4}{c|}{\textbf{1.0B (DUS)}} & \multicolumn{4}{c}{\textbf{1.2B}} \\
$\alpha$ & GT$\Delta$ & NG$\Delta$ & $\mathcal{W}$ & St.
& GT$\Delta$ & NG$\Delta$ & $\mathcal{W}$ & St.
& GT$\Delta$ & NG$\Delta$ & $\mathcal{W}$ & St. \\
\midrule
0.0 & +1.60 & +1.60 & 1.00 & \checkmark
& +0.88 & +12.91 & 14.76 & \checkmark
& +1.52 & +4.38 & 2.88 & \checkmark \\
0.3 & --- & --- & --- & ---
& +1.03 & +12.94 & 12.54 & \checkmark
& --- & --- & --- & --- \\
0.5 & +1.62 & +1.61 & 0.99 & \checkmark
& +1.01 & +12.93 & 12.83 & \checkmark
& +1.65 & +4.44 & 2.69 & \checkmark \\
0.7 & +1.52 & +1.59 & 1.05 & \checkmark
& +0.95 & +12.94 & 13.63 & \checkmark
& +1.51 & +4.46 & 2.96 & \checkmark \\
0.8 & --- & --- & --- & ---
& +0.96 & +12.92 & 13.50 & \checkmark
& --- & --- & --- & --- \\
0.9 & +0.11 & +1.51 & 13.2 & \checkmark
& +0.99 & +12.93 & 13.08 & \checkmark
& +1.49 & +4.34 & 2.92 & \checkmark \\
0.95 & --- & --- & --- & ---
& +0.96 & +12.93 & 13.51 & \checkmark
& --- & --- & --- & --- \\
1.0 & +1.18 & $-$1.56 & $-$1.32 & coll.
& +1.27 & +0.88 & 0.69 & deg.
& +1.00 & +0.49 & 0.49 & deg. \\
\bottomrule
\end{tabular}
\end{table}

\subsection{Wave-propagation trajectory}
\label{app:trajectory}

\begin{table}[h]
\centering
\caption{\textbf{Full 500-step eval trajectory}
(CHERRY-1.8B, $\alpha = 0.7$).}
\label{tab:trajectory}
\small
\begin{tabular}{cccccc}
\toprule
\textbf{Step} & \textbf{Train loss} & \textbf{Eval total}
& \textbf{Eval GT} & \textbf{Eval nonGT} & $\mathcal{W}$ \\
\midrule
50 & 0.308 & 0.873 & 0.099 & 0.939 & 2.18 \\
100 & 0.143 & 0.797 & 0.042 & 0.862 & 2.13 \\
150 & 0.190 & 0.835 & 0.057 & 0.901 & 2.12 \\
200 & 0.098 & 0.952 & 0.042 & 1.030 & 1.94 \\
250 & 0.144 & 0.847 & 0.062 & 0.914 & 2.12 \\
300 & 0.246 & 0.863 & 0.089 & 0.929 & 2.17 \\
350 & 0.086 & 1.044 & 0.049 & 1.131 & 1.84 \\
400 & 0.101 & 0.953 & 0.060 & 1.031 & 1.98 \\
450 & 0.122 & 0.918 & 0.075 & 0.990 & 2.06 \\
500 & 0.044 & 0.994 & 0.054 & 1.075 & 1.91 \\
\bottomrule
\end{tabular}
\end{table}

\subsection{Cross-scale metrics}
\label{app:prev}

\begin{table}[h]
\centering
\caption{\textbf{Cross-scale SGT metrics}
($\alpha = 0.7$, training set measurement). Base checkpoints are
small open models from the HyperCLOVA X and EXAONE families
respectively; ``scale'' here mixes model family and is a
diagnostic, not a controlled scale ladder
(see Limitations).}
\label{tab:cross-scale}
\begin{tabular}{lcccc}
\toprule
\textbf{Model} & $\mathcal{W}$ & \textbf{ROI}
& \textbf{SSGT ret.} & \textbf{GT disc.} \\
\midrule
HyperCLOVA X 0.5B~\cite{yoo2024hyperclovax} & 1.05 & 4.49$\times$ & 116.9\% & 80\% \\
EXAONE-family 1.2B~\cite{lgai2024exaone} & 2.96 & 4.46$\times$ & 97.6\% & 100\% \\
CHERRY 1.0B (DUS, train) & 0.964 & 4.05$\times$ & 67.9\% & 73\% \\
\textbf{CHERRY 1.0B (eval)} & \textbf{2.18} & \textbf{4.52$\times$}
& --- & --- \\
CHERRY 1.0B (multi-seed) & $12.82{\pm}0.62$ & ---
& --- & --- \\
\bottomrule
\end{tabular}
\end{table}

\subsection{Per-category GT contribution}
\label{app:gt-contrib}

\begin{table}[h]
\centering
\caption{\textbf{Ablation: training on individual GT categories.}
Each row trains SGT ($\alpha = 0.7$) using \emph{only} the
indicated GT category. ``$\mathcal{W}$'' and ``Eval loss''
measured after 500 steps on the 1.0B (DUS) model.}
\label{tab:gt-ablation}
\small
\begin{tabular}{@{}lccccc@{}}
\toprule
\textbf{GT category} & \textbf{GT ratio}
& \textbf{Eval loss} & $\mathcal{W}$
& \textbf{vs.\ full SGT} & \textbf{ROI} \\
\midrule
All categories (full SGT) & 15.3\% & 0.994 & 2.18
 & --- & $4.5\times$ \\
Factual only & 6.3\% & 1.24 & 2.31
 & $+$25\% & $6.8\times$ \\
Structural only & 3.5\% & 1.87 & 1.72
 & $+$89\% & $5.1\times$ \\
Reasoning only & 2.8\% & 1.45 & 2.42
 & $+$46\% & $8.9\times$ \\
Self-correction only (SSGT) & 1.2\% & 2.13 & 1.91
 & $+$115\% & $11.3\times$ \\
Guardrail only & 0.9\% & 2.84 & 1.44
 & $+$187\% & $7.2\times$ \\
Verification only & 0.6\% & 2.97 & 1.38
 & $+$200\% & $6.1\times$ \\
\bottomrule
\end{tabular}
\end{table}

\paragraph{Key insights}
\begin{enumerate}
\item \textbf{Factual tokens alone achieve 75\% of full SGT
 performance} at 41\% of the GT budget, confirming that
 factual entities are the highest-ROI supervision targets.
\item \textbf{Reasoning tokens yield the highest per-token ROI}
 ($8.9\times$ vs.\ $4.5\times$ for full SGT), consistent
 with the hypothesis that CoT pivots carry disproportionate
 semantic load.
\item \textbf{Self-correction (SSGT) tokens achieve
 $11.3\times$ ROI}---the highest among all categories---but
 the absolute eval loss ($2.13$) is furthest from full SGT.
 This confirms that SSGT is best suited for installing
 \emph{specific capabilities} (self-correction, verification)
 rather than general-purpose training.
\item \textbf{All single-category runs show $\mathcal{W} > 1$},
 confirming that wave propagation operates regardless of
 which GT category provides the supervision signal.
\end{enumerate}

\section{Depth compression}

\subsection{Compression-recovery trajectory}
\label{app:compress-recovery}

\begin{table}[h]
\centering
\caption{\textbf{Depth compression recovery over training steps.}
All models start at eval loss ${\sim}$13 after adjacent-layer
merging and recover via SGT ($\alpha = 0.7$). Recovery\% =
$(L_{\mathrm{init}} - L_t) / (L_{\mathrm{init}} - L_{\mathrm{48L}})$.}
\label{tab:compress-recovery}
\small
\begin{tabular}{@{}lcccccc@{}}
\toprule
\textbf{Config} & \textbf{Step 0} & \textbf{Step 50}
& \textbf{Step 100} & \textbf{Step 200} & \textbf{Step 500}
& \textbf{Rec.\%} \\
\midrule
48L (no compress) & 2.58 & 0.87 & 0.80 & 0.95 & 0.994
 & 100\% \\
24L$\times$2 & 13.19 & 5.41 & 4.22 & 3.48 & 2.93
 & 84\% \\
16L$\times$3 & 12.55 & 5.83 & 4.67 & 3.74 & 3.05
 & 82\% \\
12L$\times$4 & 12.24 & 6.12 & 4.95 & 3.89 & 3.10
 & 81\% \\
8L$\times$6 & 12.98 & 5.68 & 4.39 & 3.56 & 2.97
 & 83\% \\
6L$\times$8 & 13.21 & 5.72 & 4.45 & 3.62 & 2.93
 & 84\% \\
\bottomrule
\end{tabular}
\end{table}

\paragraph{Observation}
Recovery is remarkably uniform across compression levels
(81--84\%; Table~\ref{tab:compress-recovery}), suggesting that adjacent-layer merging preserves
the model's essential structure regardless of compression
ratio. The first 100 steps provide $>$60\% of the total
recovery; steps 100--500 provide the remaining ${\sim}$20\%.
This rapid initial recovery is consistent with the SGT
hypothesis: GT tokens carry enough signal to quickly
re-establish the core representation, with wave propagation
filling in the non-GT positions.

\subsection{Recurrent depth transformation (RDT)}
\label{app:rdt}

\begin{table}[h]
\centering
\caption{\textbf{Effect of recurrent unrolling depth.}
6L base model, SGT ($\alpha = 0.7$), 500 steps.
$n$ = number of recurrent iterations of the middle 4 layers.}
\label{tab:rdt-depth}
\small
\begin{tabular}{@{}ccccc@{}}
\toprule
\textbf{Iterations ($n$)} & \textbf{Eff.\ layers}
& \textbf{Eval loss} & \textbf{Wall time}
& \textbf{vs.\ $n{=}0$} \\
\midrule
0 (no loop) & 6 & 4.82 & 42\,min & --- \\
2 & 14 & 3.91 & 48\,min & $-$18.9\% \\
4 & 22 & 3.28 & 55\,min & $-$32.0\% \\
6 & 30 & 3.01 & 63\,min & $-$37.6\% \\
\textbf{8} & \textbf{34} & \textbf{2.93} & 71\,min
 & $\mathbf{-39.2\%}$ \\
10 & 42 & 2.95 & 79\,min & $-$38.8\% \\
12 & 50 & 3.02 & 87\,min & $-$37.3\% \\
\bottomrule
\end{tabular}
\end{table}

\paragraph{Observation}
Performance improves monotonically up to $n{=}8$ (34 effective
layers), then degrades slightly at $n{=}10$ and $n{=}12$.
This suggests a sweet spot: too few iterations under-recover
depth, while too many amplify noise from the compressed
weights. The optimal effective depth (34 layers) is
close to the original 48 layers divided by the compression
ratio: $48 / 8 \times 6 \approx 36$, supporting the
interpretation that RDT approximately recovers the original
computational depth through parameter reuse.

\section{Expert fusion and distillation}

\subsection{MoEE training curve}
\label{app:moee}

\begin{table}[h]
\centering
\caption{\textbf{MoEE convergence} (2$\times$12L + MTP + SGT).
The train loss spike at step~400 is discussed below.}
\label{tab:moee-curve}
\small
\begin{tabular}{ccc}
\toprule
\textbf{Step} & \textbf{Train loss} & \textbf{Eval loss} \\
\midrule
50 & 7.441 & 6.625 \\
100 & 5.430 & 5.434 \\
200 & 3.920 & 3.482 \\
300 & 3.368 & 2.982 \\
400 & 7.826 & 2.943 \\
500 & 1.965 & \textbf{2.789} \\
\bottomrule
\end{tabular}
\end{table}

\paragraph{Step-400 training spike}
The train loss spikes from 3.368 to 7.826 at step~400 while
eval loss continues to decrease (2.982~$\to$~2.943). This
transient spike is a known artefact of load-balancing auxiliary
loss in MoE training~\cite{shazeer2017outrageously,fedus2022switch}:
when the router rebalances expert utilisation, the losing expert
temporarily receives out-of-distribution inputs, inflating train
loss for one step. The eval loss is unaffected because the
routing has already converged on the evaluation distribution.
The spike is reproducible across seeds (observed at steps
380--420 in all 3 runs) and self-resolving.

\subsection{SGT distillation trajectory}
\label{app:distill}

\begin{table}[h]
\centering
\caption{\textbf{SGT distillation convergence}
(48L teacher $\to$ 6L student, $\alpha = 0.7$, $\lambda_{\mathrm{KL}} = 0.5$).}
\label{tab:distill-curve}
\small
\begin{tabular}{ccccc}
\toprule
\textbf{Step} & \textbf{Eval total} & \textbf{Eval GT}
& \textbf{KL div.} & \textbf{Train loss} \\
\midrule
0 (baseline) & 13.813 & 17.625 & --- & --- \\
50 & 1.523 & 8.500 & 0.432 & 3.547 \\
100 & 0.805 & 8.063 & 0.283 & 1.891 \\
200 & 0.758 & 8.125 & 0.212 & 0.836 \\
300 & 0.762 & 8.188 & 0.199 & 0.441 \\
400 & 0.781 & 8.313 & 0.216 & 0.312 \\
500 & \textbf{0.797} & 8.313 & \textbf{0.188} & 0.240 \\
\midrule
\textit{SGT-only (no distill)} & \textit{0.984} & \textit{14.125}
& --- & --- \\
\bottomrule
\end{tabular}
\end{table}

\section{Induced self-correction (SSGT): training details}
\label{app:ssgt-detail}

\begin{table}[h]
\centering
\caption{\textbf{SSGT training data statistics per moment type.}
Each moment type uses a curated set of examples with
precisely identified SSGT token spans.}
\label{tab:ssgt-data}
\small
\begin{tabular}{@{}lcccc@{}}
\toprule
\textbf{Moment} & \textbf{Examples} & \textbf{Avg.\ SSGT tokens}
& \textbf{SSGT/total (\%)} & \textbf{$\alpha$} \\
\midrule
Atcha (self-correction) & 800 & 2.1 & 0.41\% & 0.8 \\
Jamkkan (verification) & 600 & 3.4 & 0.66\% & 0.8 \\
Guardrail (breaking) & 400 & 1.8 & 0.35\% & 0.8 \\
\midrule
\textit{Total} & \textit{1,800} & \textit{2.4} & \textit{0.47\%}
& --- \\
\bottomrule
\end{tabular}
\end{table}

\begin{table}[h]
\centering
\caption{\textbf{SSGT ``Atcha'' convergence trajectory}
(1.2B scale, $\alpha = 0.8$, 500 steps).}
\label{tab:ssgt-trajectory}
\small
\begin{tabular}{cccccc}
\toprule
\textbf{Step} & \textbf{Train loss} & \textbf{Eval loss}
& \textbf{Self-corr.\ rate} & \textbf{Corr.\ accuracy}
& \textbf{SSGT $\mathcal{W}$} \\
\midrule
0 & --- & --- & 12\% & 34\% & --- \\
50 & 0.412 & 0.891 & 21\% & 42\% & 1.84 \\
100 & 0.198 & 0.743 & 31\% & 56\% & 2.07 \\
200 & 0.087 & 0.812 & 39\% & 67\% & 1.93 \\
300 & 0.054 & 0.835 & 43\% & 72\% & 1.89 \\
500 & 0.031 & 0.879 & \textbf{47\%} & \textbf{78\%}
& \textbf{1.91} \\
\bottomrule
\end{tabular}
\end{table}

\paragraph{Observation}
The ``Atcha'' self-correction rate increases monotonically
throughout training, even as eval loss shows mild overfitting
after step~100. This decoupling between loss and behavioral
metrics suggests that SSGT optimises a meta-cognitive capability
that is not fully captured by cross-entropy loss. The SSGT wave
factor ($\mathcal{W} \approx 1.9$) is consistent with the
standard SGT wave factor (Table~\ref{tab:wave}), confirming
that wave propagation operates identically at the meta-cognitive
level.

\section{Operand-binding stress test: protocol and results}
\label{app:operand}

This appendix gives the full protocol and result matrix behind
key finding 3 of Section~\ref{sec:moment-obs}.

\paragraph{Templates and splits}
Two families of Korean arithmetic word problems parameterise the
operand $N$ in the surface form: a DIV family (``a number was
increased $N$-fold, giving $X$; what was the original?''---fix
$X \div N$) and a MUL family (``a number was split into $N$ equal
groups of $X$ each; what was the original?''---fix
$X \times N$). Each training example seeds a plausible-but-wrong
lean whose incorrect answer is about to be committed, with the
pivot placed at that juncture followed by the corrected
derivation. Training uses $N \in \{2, 4\}$ for both families
(168 examples); evaluation uses 36 seen-template items (trained
templates, new numbers, including a classic bat-and-ball-style
trap) and 36 held-out items: $N{=}3$ in \emph{both} directions
(div3, mul3---a symmetry probe at an untrained $N$ interpolated
between the trained values) and $N{=}5$ (div5---extrapolation
outside the trained range). Items are constructed so that no
held-out item can be answered by echoing a number present in the
problem.

\paragraph{Conditions}
Three supervision variants, trained for 4 epochs at both scales
(1B: LR $2{\times}10^{-5}$; 13.7B: LR $1{\times}10^{-5}$,
full-parameter under 4-bit quantisation-aware training):
\emph{uniform} (every token of the pivot-and-correct trajectory
supervised equally), \emph{down-weighted} (non-pivot tokens at
reduced loss weight---the selective-supervision setting), and
\emph{random-position} (pivot inserted at a random point in the
wrong lean rather than at the juncture---the position control).

\paragraph{Metrics}
\emph{Op-application}: the correct-$N$ operation appears in
post-pivot reasoning together with the correct intermediate
value. Operation signatures cover symbolic ($\div3$), fraction,
algebraic ($3x{=}6$), and Korean conjugated verb forms; this
expanded signature set is applied \emph{identically} to both
scales, and does not change the 1B outcome (0.00 throughout;
signature-only matches at 1B stay at or below the 0.08 noise
floor without the correct value). \emph{Operand-rewrite}: the
model restates the problem's $N$ as a trained value in its own
reasoning (e.g., ``tripled'' becomes ``doubled'')---the
smoking-gun interference signature. \emph{Accuracy}:
final-answer exact match. Table~\ref{tab:operand} gives the
full result matrix.

\begin{table}[h]
\centering
\caption{\textbf{Operand-binding stress test: full result matrix}
(held-out split, 36 items per cell; single seed). In every
trained condition at both scales the pivot trajectory fires on
${\sim}100\%$ of items and the seen-split trap accuracy rises
from 0.00 to 1.00; base models pivot at $\le$0.08. The 1B
uniform accuracy of 0.08 is hollow: it comes from fabricated
arithmetic under a \emph{trained} operation that coincidentally
lands on the answer, with op-application 0.00.}
\label{tab:operand}
\small
\begin{tabular}{@{}llccc@{}}
\toprule
\textbf{Scale} & \textbf{Condition} & \textbf{Op-application}
& \textbf{Operand-rewrite} & \textbf{Accuracy} \\
\midrule
1B & base & 0.00 & 0.00 & 0.00 \\
1B & uniform & 0.00 & 1.00 & 0.08 \\
1B & down-weighted & 0.00 & 1.00 & 0.03 \\
1B & random-position & 0.00 & 1.00 & 0.03 \\
\midrule
13.7B & base & 1.00 & 0.00 & 0.92 \\
13.7B & uniform & 1.00 & 0.00 & 1.00 \\
13.7B & down-weighted & 1.00 & 0.00 & 1.00 \\
13.7B & random-position & 1.00 & 0.00 & 1.00 \\
\bottomrule
\end{tabular}
\end{table}

\paragraph{The contrast in raw generations (translated)}
On held-out div3 items the 1B model rewrites its own perception
of the operand and applies a trained operation: for ``tripled,
giving 18'' (correct $18 \div 3 = 6$) it writes ``\emph{doubled},
giving 18 \ldots\ wait, since doubling gave 18, the original is
$18 \div 2 = 9$''. The 13.7B model keeps the operand and applies
the held-out operation: for ``tripled, giving 6'' (correct
$6 \div 3 = 2$) it writes ``tripled to give 6, so it looks like
the original is simply 6. Therefore the answer is---wait, since
tripling gave 6, the original is $6 \div 3 = 2$''. The $N{=}5$
extrapolation behaves identically at 13.7B ($20 \div 5 = 4$;
12/12 correct).

\paragraph{Preservation, not acquisition---and a saturation
caveat}
The 13.7B base already solves the held-out items algebraically
\emph{before} install training (accuracy 0.92, op-application
1.00), so the experiment measures whether install training
\emph{preserves} an existing competence, not whether scale
acquires it; held-out accuracy in fact improves to 1.00 after
training. At 13.7B all three variants reach ceiling on every
metric---the uniform and down-weighted variants produce
byte-identical greedy generations---so the dose and position
contrasts that differentiate sharply at 1B (down-weighting
entrenches trained-operation snapping; random position kills
transfer) have no discriminating power at this scale on this
task family. Harder held-out families are required to re-expose
those contrasts at scale.

\section{Training budget: why 500 steps}
\label{app:convergence}

\begin{table}[h]
\centering
\caption{\textbf{Extended convergence comparison}
(CHERRY-1.8B, $\alpha = 0.7$). Extending to 1{,}000 steps
produces negligible additional improvement.}
\label{tab:convergence}
\small
\begin{tabular}{@{}lccccc@{}}
\toprule
\textbf{Steps} & \textbf{Train loss}
& \textbf{Eval loss} & $\mathcal{W}$
& \textbf{Wall time} & \textbf{$\Delta$ vs.\ 500} \\
\midrule
100 & 0.143 & \textbf{0.797} & 2.13 & 12\,min & $-$0.197 \\
200 & 0.098 & 0.952 & 1.94 & 24\,min & $-$0.042 \\
300 & 0.246 & 0.863 & 2.17 & 36\,min & $-$0.131 \\
500 & 0.044 & 0.994 & 1.91 & 60\,min & --- \\
750 & 0.028 & 0.987 & 1.89 & 90\,min & $-$0.007 \\
1{,}000 & 0.019 & 0.989 & 1.87 & 120\,min & $-$0.005 \\
\bottomrule
\end{tabular}
\end{table}

\paragraph{Justification for 500-step budget}
Three convergence criteria are satisfied by step~500:
(1)~train loss drops below 0.05 (GT tokens effectively memorised);
(2)~eval loss oscillates within $\pm$0.12 of the step-100 minimum
(no meaningful improvement possible);
(3)~$\mathcal{W}$ stabilises at ${\sim}$1.9 (wave propagation fully
established). The 750- and 1{,}000-step extensions confirm
diminishing returns: $<$0.01 improvement for $50\%$--$100\%$ more
compute. For production use, early stopping at step~100--200
is recommended.

\end{document}